\documentclass[11pt]{article}
\usepackage[T1]{fontenc}
\usepackage{kpfonts}
\DeclareMathAlphabet{\mathsf}{OT1}{cmss}{m}{n}
\SetMathAlphabet{\mathsf}{bold}{OT1}{cmss}{bx}{n}

\usepackage{ragged2e}
\usepackage{amsmath,amssymb}
\usepackage{natbib}
\usepackage[usenames]{color}
\usepackage{xcolor}
\usepackage{smile}
\usepackage{color}
\usepackage{float}
\usepackage{multirow}
\usepackage{url}

\usepackage{threeparttable}
\usepackage{graphicx,wrapfig}
\usepackage{longtable}
\usepackage{makecell}
\usepackage{subcaption}
\usepackage[font=small,labelfont=bf]{caption}

\usepackage{xr-hyper}
\PassOptionsToPackage{hyphens}{url}
\usepackage[colorlinks,
linkcolor=blue,
anchorcolor=blue,
citecolor=blue
]{hyperref}

\usepackage{bbm}
\usepackage{enumitem}
\newcommand{\Fr}{\text{F}}

\newcommand{\removed}[1]{}
\newcommand{\sgn}{{\rm sign}}
\newcommand{\HT}{{\c{T}}}
\newcommand{\gradvec}{\b{g}_i}
\newcommand{\gradmat}{\b{g}}
\newcommand{\supp}{{\rm supp}}

\newcolumntype{P}[1]{>{\centering\arraybackslash}p{#1}}
\newcolumntype{Y}{>{\centering\arraybackslash}m{0.9cm}}
\newcolumntype{G}{>{\bfseries\centering\arraybackslash}m{1.8cm}}

\usepackage[lined,commentsnumbered,ruled]{algorithm2e}

\defcitealias{Rambhatla2019NOODL}{\texttt{NOODL}}	
\SetKwInput{KwInput}{Input}
\SetKwInput{KwInit}{Initialize}
\SetKwInput{KwOutput}{Output}
\SetKwInput{KwReturn}{Return}
\SetKwInput{KwSet}{Set}

\def\[#1\]{$#1$}
\def\b#1{{\mathbf{#1}}}
\def\c#1{{\mathcal{#1}}}

\def\##1\#{\begin{align}#1\end{align}}
\def\$#1\${\begin{align*}#1\end{align*}}

\usepackage[colorinlistoftodos,color=gray!25]{todonotes}
\usepackage{cancel,soul}
\usepackage{xargs}
\usepackage{setspace}
\usepackage{fullpage}

\vspace{-0pt}

\begin{document}
\date{}
\title{\bf Provable Online CP/PARAFAC Decomposition of a Structured Tensor via Dictionary Learning}

\author{\normalsize Sirisha Rambhatla, Xingguo Li, and Jarvis Haupt \thanks{Sirisha Rambhatla is affiliated with the Computer Science Department, University of Southern California, Los Angeles, CA, USA; Email: {\tt{sirishar@usc.edu}}. Xingguo Li is affiliated with the Computer Science Department, Princeton University, Princeton, NJ, USA;  Email: {\tt{xingguol@cs.princeton.edu}}. Jarvis Haupt is affiliated with Department of Electrical and Computer Engineering, University of Minnesota, Minneapolis, MN; Email: \texttt{jdhaupt@umn.edu}.  This article is a preprint.} }
\maketitle

\begin{abstract}
\vspace{-5pt}
We consider the problem of factorizing a structured $3$-way tensor into its constituent Canonical Polyadic (CP) factors. This decomposition, which can be viewed as a generalization of singular value decomposition (SVD) for tensors, reveals how the tensor dimensions (features) interact with each other.  However, since the factors are \textit{a priori} unknown, the corresponding optimization problems are inherently non-convex. The existing guaranteed algorithms which handle this non-convexity incur an irreducible error (bias), and only apply to cases where all factors have the same structure. To this end, we develop a provable algorithm for online structured tensor factorization, wherein one of the factors obeys some incoherence conditions, and the others are sparse. Specifically we show that, under some relatively mild conditions on initialization, rank, and sparsity, our algorithm recovers the factors \emph{exactly} (up to scaling and permutation) at a linear rate. Complementary to our theoretical results, our synthetic and real-world data evaluations showcase superior performance compared to related techniques. Moreover, its scalability and ability to learn on-the-fly makes it suitable for real-world tasks.
\end{abstract}
\vspace{-14pt}
\section{Introduction} 
\vspace{-3pt}
Canonical Polyadic (CP) /PARAFAC decomposition aims to express a tensor as a sum of rank-$1$ tensors, each of which is formed by the outer-product (denoted by ``\[\circ\]'') of constituent factors columns. Specifically, the task is to factorize a given $3$-way tensor \[\underline{\b{Z}} \in \mathbb{R}^{n \times J \times K}\] as
	\vspace{-4pt}
\begin{align}\label{CPD}
\underline{\b{Z}} = \textstyle\sum_{i=1}^{m} \b{A}_i^* \circ \b{B}_i^* \circ \b{C}_i^* = [\![ \b{A}^*, \b{B}^*, \b{C}^*]\!], 
\end{align}
where  \[\b{A}_i^*\], \[\b{B}_i^*\] and \[\b{C}_i^*\] are columns of factors \[\b{A}^*\], \[\b{B}^*\], and \[\b{C}^*\], respectively, and are \emph{a priori} unknown. A popular choice for the factorization task shown in \eqref{CPD} is via the alternating least squares (ALS) algorithm; see \cite{Kolda09} and references therein. Here, one can add appropriate regularization terms (such as \[\ell_1\] loss for sparsity) to the least-square objective to steer the algorithm towards specific solutions \citep{Martinez2008, Allen2012, Papalexakis2013}. However, these approaches suffer from three major issues -- a) the non-convexity of associated formulations makes it challenging to establish recovery and convergence guarantees, b) one may need to solve an implicit model selection problem (e.g., choose the \emph{a priori} unknown \emph{tensor rank} \[m\]), and c) regularization may be computationally expensive, and may not scale well in practice.

Recent works for guaranteed tensor factorization  -- based on tensor power method \citep{Anandkumar15}, convex relaxations \citep{Tang15}, sum-of-squares formulations \citep{Barak2015, Ma2016, Schramm2017}, and variants of ALS algorithm \citep{Sharan2017}  --  have focused on recovery of tensor factors wherein all factors have a common structure, based on some notion of incoherence of individual factor matrices such as sparsity, incoherence, or both \citep{sun2017}. %As a result, these are not suitable for cases where the constituent factors do not have a shared structure. 
Furthermore, these algorithms a) incur \textit{bias} in estimation, b) are computationally expensive in practice,  and c) are not amenable for online (streaming) tensor factorization; See Table~\ref{tab:compare_tensor}. Consequently, there is a need to develop fast, scalable provable algorithms for exact (unbiased) factorization of structured tensors  arriving (or processed) in a streaming fashion (online), generated by heterogeneously  structured factors. To this end, we develop a provable algorithm to recover the unknown factors of tensor(s) \[\underline{\b{Z}}^{(t)} \] in Fig.\ref{fig:Tensor image}
 (arriving, or made available for sequential processing, at an instance \[t\]), assumed to be generated as \eqref{CPD}, wherein the factor \[\b{A}^*\] is \emph{incoherent} and fixed  (deterministic), and the factors \[\b{B}^{*(t)}\] and \[\b{C}^{*(t)}\] are sparse and vary with \[t\] (obey some randomness assumptions).

	\begin{figure}[!t]
		\centering
		\resizebox{0.43\textwidth}{!}{
			\centering
			\begin{tikzpicture}
			\node[anchor=south west,inner sep=0] (image) at (0.2,0) {\includegraphics[width=0.33\textwidth]{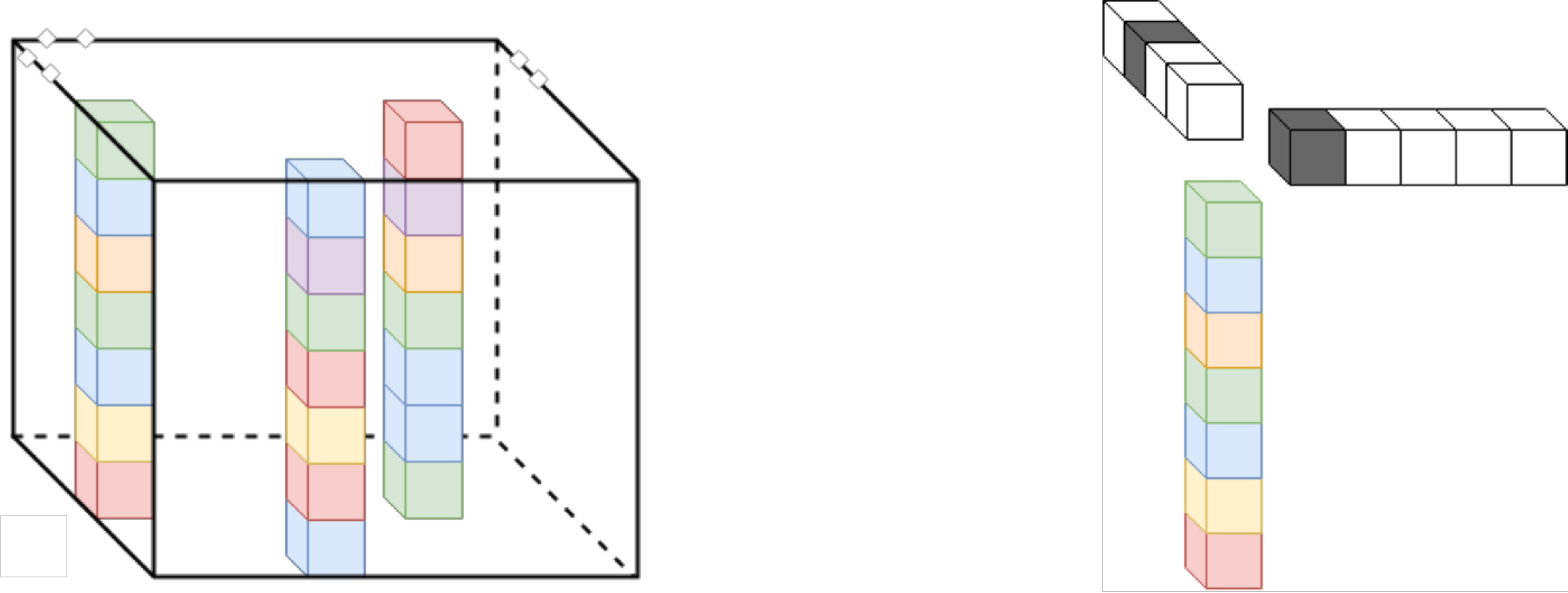}};
			\node[align=center]  at (0,1.2) {$n$};
			\node[align=center]  at (0.2,0.2) {$K$};
			\node[align=center]  at (1.6,-0.2) {$J$};
			\node[align=center] at (1.6,0.9) {\large$\underline{\b{Z}}^{(t)}$};
			\node[align=center] at (3.1,0.9) {\large$=$};
			\node[align=center] at (3.9,0.9) {\Large$\sum\limits_{i=1}^{m}$};
			\node[align=center] at (5.1,0) {$\b{A}_i^* $};
			\node[align=center] at (5.75,1.1) {$\b{B}_i^{*(t)} $};
			\node[align=center] at (4.9,2.2) {$\b{C}_i^{*(t)} $};
			\end{tikzpicture}}
		\vspace{-8pt}
		\caption{Tensor \protect{\[\underline{\b{Z}}^{(t)} \in \mathbb{R}^{n\times J \times K}\]} of interest, a few mode-1 fibers are dense.}\label{fig:Tensor image} 
			\vspace{-10pt}
	\end{figure}
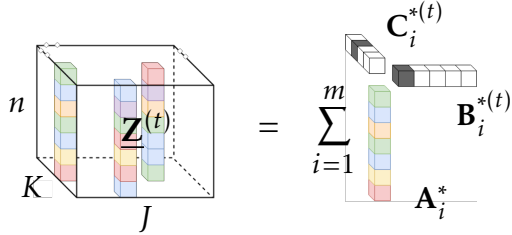

   \paragraph{Model Justification.} The tensor factorization task of interest arises in streaming applications where users interact only with a few items at each time \[t\], i.e. the user-item interactions are \emph{sparse}. Here, the fixed incoherent factor \[\b{A}^*\] columns model the underlying fixed interactions patterns (\emph{signatures}). At time \[t\], a fresh observation tensor \[\underline{\b{Z}}^{(t)} \] arrives, and the task is to estimate sparse factors (users and items), and the incoherent factor (patterns). This estimation procedure reveals users \[\b{B}^*_i\] and items \[\b{C}^*_i\] sharing the same pattern \[\b{A}^*_i\], i.e. the the underlying clustering, and finds applications in scrolling pattern analysis in web analytics \citep{Mueller2001}, sports analytics (section \ref{sec:simnba}),  patient response to probes \citep{Deburchgraeve2009, Becker15}, electro-dermal response to audio-visual stimuli \citep{Grundlehner2009,Silveira2013}, and organizational behavior via email activity \cite{Fu2015, Kolda09}.
\vspace{-5pt}
\subsection{Overview of the results}
\vspace{-3pt}
We take a matrix factorization view of the tensor factorization task to develop a provable tensor factorization algorithm for exact recovery of the constituent factors. Leveraging the structure of the tensor, we formulate the non-zero fibers as being generated by a dictionary learning model, where the data samples \[\b{y}_{(j)} \in \RR^{n}\] are assumed to be generated as follows from an \textit{a priori} unknown dictionary \[\b{A}^* \in \RR^{n \times m}\] and sparse coefficients \[\b{x}_{(j)}^* \in \RR^{m}\].
\addtolength{\abovedisplayskip}{-4pt}
\addtolength{\belowdisplayskip}{-4pt}
\begin{align}\label{eq:model_dl}
\b{y}_{(j)} = \b{A}^*\b{x}_{(j)}^*,~\|\b{x}_{(j)}^*\|_0 \leq s~~\text{for all}~~j = 1, 2, \dots \vspace{-20pt}%+ \b{w}_{(i)},
\end{align}
This modeling procedure includes a matricization or \emph{flattening} of the tensor, which leads to a Kronecker (Khatri-Rao) dependence structure among the elements of the resulting coefficient matrix; see section~\ref{sec:main_res}. As a result, the main challenges here are to a) analyze the Khatri Rao product (KRP) structure to identify and quantify data samples (non-zero fibers) available for learning, b) establish guarantees on the resulting sparsity structure, and c) develop a SVD-based guaranteed algorithm to successfully untangle the sparse factors using corresponding coefficient matrix estimate and the underlying KRP structure, to develop recovery guarantees. This matricization-based analysis can be of independent interest.

\vspace{-5pt}
\subsection{Contributions}
\vspace{-3pt}
We develop an algorithm to recover the CP factors of tensor(s) \[\underline{\b{Z}}^{(t)}\hspace{-2pt}\in\hspace{-2pt}\mathbb{R}^{n \times J \times K}\], arriving (or made available) at time \[t\], generated as per \eqref{CPD} from constituent factors \[\b{A}^*\in\mathbb{R}^{n \times m}\], \[\b{B}^{*(t)}\in\mathbb{R}^{J \times m}\], and \[\b{C}^{*(t)}\in\mathbb{R}^{K \times m}\], where the unit-norm columns of \[\b{A}^*\] obey some incoherence assumptions, and \[\b{B}^{*{(t)}} \] and \[\b{C}^{*(t)}\] are sparse. 
Our specific contributions are:
\vspace{-4pt}
\begin{itemize}[leftmargin=*]
\setlength{\itemsep}{-0.1pt}
\item \textbf{Exact recovery and linear convergence}: Our algorithm \texttt{TensorNOODL}, to the best of our knowledge, is the first to accomplish recovery of the true CP factors of this structured tensor(s) \[\underline{\b{Z}}^{(t)} \] \emph{exactly} (up to scaling and permutations) at a linear rate. Specifically, starting with an appropriate initialization \[\b{A}^{(0)}\] of \[\b{A}^*\] , we have \[\b{A}^{(t)}_i{\hspace{-2pt}\rightarrow}\b{A}^*_i\], \[\hat{\b{B}}^{(t)}_i{\hspace{-2pt}\rightarrow}\pi_{B_i}\b{B}^{*(t)}_i\], and  \[\hat{\b{C}}^{(t)}_i{\hspace{-2pt}\rightarrow}\pi_{C_i}\b{C}^{*(t)}_i\], as iterations \[\hspace{-1pt}t{\rightarrow}\hspace{-0pt}\infty\],  for constants \[\pi_{B_i}\] and \[\pi_{C_i}\]. 
\item \textbf{Provable algorithm for heterogeneously-structured tensor factorization}: We consider the \textit{exact} tensor factorization, an inherently non-convex task, when the factors do not obey same structural assumptions.  That is, our algorithmic procedure overcomes the non-convexity bottleneck suffered by related optimization-based ALS formulations. 
\item \textbf{Online, fast, and scalable}: The online nature of our algorithm, separability of updates, and specific guidelines on choosing the parameters, make it suitable for large-scale distributed implementations. Furthermore, our numerical simulations (both synthetic and real-world) demonstrate superior performance in terms of accuracy, number of iterations, and demonstrate its applicability to real-world factorization tasks.
\end{itemize}
\vspace{-5pt}
Furthermore, although estimating the rank of a given tensor is NP hard, the incoherence assumption on \[\b{A}^*\], and distributional assumptions on \[\b{B}^{*(t)}\] and \[\b{C}^{*(t)}\], ensure that our matrix factorization view is \emph{rank revealing} \citep{Sidiropoulos2017}. In other words, our assumptions ensure that the dictionary initialization algorithms (such as \cite{Arora15}) can recover the rank of the tensor. Following this, \texttt{TensorNOODL} recovers the true factors (up to scaling and permutation) whp.

\protect{\begin{table*}[!t]
		\vspace{-5pt}
		\caption{\protect{Comparison of provable algorithms for tensor factorization and dictionary learning.  As shown here, the existing provable tensor factorization techniques do not apply to the case where \[\b{A}\]: incoherent, \[(\b{B}, \b{C})\]: sparse. }}\label{tab:compare_tensor}
		\vspace{-0pt}
		\renewcommand{\arraystretch}{1.2}
		\centering
		\resizebox{0.98\textwidth}{!}{
			\begin{threeparttable}[h]
				\begin{tabular}{c|P{3.8cm}|c|c|P{4cm}|c}
					\Xhline{1pt}
					\multirow{2}{*}{\textbf{Method}}&\multicolumn{3}{c|}{\textbf{Conditions}}& \multicolumn{2}{c}{\textbf{Recovery Guarantees} }\\ \cline{2-6}
					& {\small\textbf{Model} }&  {\small\textbf{Rank} }& {\small\textbf{Initialization}}& \multirow{2}{*}{\small\textbf{Estimation Bias}} & \multirow{2}{*}{\small\textbf{Convergence}} \vspace{-5pt}\\ 
					& {\small \bf Considered} & &{\small\bf Constraints} & & \\ \hline
					TensorNOODL (this work) & \[\b{A}\]: incoherent, \[(\b{B}, \b{C})\]: sparse &
					\[m= \c{O}(n)\]& \[\mathcal{O}^*\rbr{\tfrac{1}{\log(n)}} \] & No Bias & Linear\\ \hline
					\citet{sun2017}\[^\ddagger\] &  \[(\b{A}, \b{B}, \b{C})\]: all incoherent and sparse  &
					\[m= o(n^{1.5})\]& \[o(1)\] & \[\|\b{A}_{ij} - \hat{\b{A}}_{ij}\|_{\infty} =\c{O}(\tfrac{1}{n^{0.25}})^\dagger\]  &  Not established\\ \hline
					\citet{Sharan2017}\[^\ddagger\] &  \[(\b{A}, \b{B}, \b{C})\]: all incoherent  &
					\[m= o(n^{0.25})\]& Random & \[\|\b{A}_i - \hat{\b{A}}_i\|_2 =\c{O}(\sqrt{\tfrac{m}{n}})^\dagger\]  &  Quadratic\\ \hline
					\multirow{2}{*}{\citet{Anandkumar15}\[^\ddagger\]} &\multirow{2}{*}{\[(\b{A}, \b{B}, \b{C})\]: all incoherent}& \[m= \c{O}(n)\]  &\[\mathcal{O}^*\rbr{\tfrac{1}{\sqrt{n}}} \]\[^\P\]&  \[\|\b{A}_i - \hat{\b{A}}_i\|_2 =\tilde{\c{O}}(\tfrac{1}{\sqrt{n}})^\dagger\] & Linear\[^\S\]\\\cline{3-6}
					&& \[m= o(n^{1.5})\]  &\[\c{O}(1)\]& \[\|\b{A}_i - \hat{\b{A}}_i\|_2 =\tilde{\c{O}}(\tfrac{\sqrt{m}}{n})^\dagger\] &Linear \\\hline
					\multirow{2}{*}{\citet{Arora15}} & \multirow{2}{*}{Dictionary Learning \eqref{eq:model_dl}}   &  \[m= \c{O}(n)\]  & \[\mathcal{O}^*\rbr{\tfrac{1}{\log(n)}} \]&  \[\mathcal{O}(\sqrt{{s}/{n}})\] & Linear\\\cline{3-6}
					& & \[m= \c{O}(n)\]   & \[\mathcal{O}^*\rbr{\tfrac{1}{\log(n)}} \]&  \textit{Negligible} bias \[^\S\]& Linear\\ \hline
					\citet{Mairal09}&  Dictionary Learning \eqref{eq:model_dl} & \multicolumn{4}{c}{Convergence to stationary point; similar guarantees by \citet{Huang2016}.}\\ \hline
					%\cite{Tang15}& \[(\b{A}, \b{B}, \b{C})\]: incoherent &\[N\leq m\]&\[\mathcal{O}^*\rbr{\tfrac{1}{\log(N)}} \] & Global, no bias & Linear\\ \hline
					%\cite{Agarwal14}{$^\ddagger$} & \[\mathcal{O}^*\rbr{{1}/{\rm poly(m)}}\] &  \[\mathcal{O}\rbr{{\sqrt[6]{n}}/{\mu}} \] & \[\Omega(m^2)\]&  No bias & N/A \\\hline
					%\cite{Spielman2012} \[(\text{for}~n\leq m)\]& N/A&  \[\c{O}(\sqrt{n})\] &\[\tilde{\Omega}(n^2)\] & No bias&N/A   \\
					\Xhline{1pt}
				\end{tabular}
				\begin{tablenotes}[wide=0pt]
					%\item[$\ast$] \xl{There should be no bias for coefficient, o.w., we need to write $\epsilon_0$ for dict.} At the iterate \[t\] of the online procedure and for \[0<\omega<1/2\]. For an initial estimate \[\b{A}^{(0)}\]. \[\|\b{A}_i^{(0)} -\b{A}_i^*\| \leq \epsilon_0 \].
					\item[\[\ddagger\]] \protect{This procedure is not \textit{online}.  $^\dagger$ Result applies for each \[ i\in [1,m]\]. \[^\P\] Polynomial number of initializations \[m^{\beta^2}\] are required, for \[\beta\geq m/n\]. \[^\S\] The procedure has an \textit{almost} Quadratic rate initially.} %\[^\natural\] Requires \[{\rm poly}(m)\] samples; not \emph{neurally plausible}.% and employs a quadratic convex program for the coefficients.  
				\end{tablenotes}
		\end{threeparttable}}%
	\vspace{-10	pt}
\end{table*}}

\vspace{-5pt}
\subsection{Related works} \label{sec:relatedworks}
\vspace{-3pt}
\paragraph{Tensor Factorization.} Canonical polyadic (CP)/PARAFAC decomposition % of a tensor can be viewed as a sum of rank-1 tensors, as 
\eqref{CPD} captures relationships between the latent factors, where the number of rank-1 tensors define the rank for a tensor. Unlike matrices decompositions, tensor factorizations can be unique under relatively mild conditions \citep{Kruskal1977,Sidiropoulos2000}. However, determining tensor rank is NP-hard  \citep{Hastad90}, and so are tasks like tensor decompositions \citep{Hillar13}. Nevertheless, regularized ALS-based approaches emerged as a popular choice to impose structure on the factors, however establishing convergence to even a stationary point is difficult \citep{Mohlenkamp2013}; see also \citep{Cohen17}.  %which considers a case where one of the factors admits a sparse-in-a-known-dictionary structure with local convergence guarantees. 
The variants of ALS with some convergence guarantees do so at the expense of complexity \citep{Li2015tens, Razaviyayn2013}, and convergence rate \citep{Uschmajew2012}; See also \citep{Kolda09} and \citep{Sidiropoulos2017}. 
On the other hand, guaranteed methods initially relied on a computationally expensive orthogonalizing step (\emph{whitening}), and therefore, did not extend to the overcomplete setting (\[m>n\]) \citep{Comon1994, Kolda2011,Zhang2001,Le2011,Huang2015, Anandkumar14a, Anandkumar2016}. As a result, works such as \citep{Tang15, Anandkumar15, Sun16}, relaxed orthogonality to an \emph{incoherence} condition to handle the overcomplete setting. To counter the complexity, \cite{Sharan2017} developed a orthogonalization-based provable ALS variant, however, this precludes its use in overcomplete settings. 

%based on convex relaxations and tensor power iterations, respectively, by relaxing the orthogonality requirement to an \emph{incoherence} condition to extend results to overcomplete settings. Extending \citep{Anandkumar15}, \citep{Sun16} considered a case wherein the factors are both \textit{sparse} and incoherent. 

%To overcome the complexity of these techniques and to leverage the simplicity of the ALS algorithm,  \cite{Sharan2017} developed a provable variant -- orth-ALS, by including an orthogonalization step. 

\paragraph{Dictionary Learning.} We now provide a brief overview of the dictionary learning literature. Popularized by the rich sparse inference literature, \emph{overcomplete} \[(m\geq n)\] representations lead to sparse(r) representations which are robust to noise; see \cite{Mallat1993, Chen1998, Donoho2006}. Learning such sparsifying overcomplete representations is known as \emph{dictionary learning} \citep{Olshausen97, Lewicki2000,Mairal09, Remi2010}. Analogous to the ALS algorithm, the alternating minimization-based techniques became widely popular in practice, however theoretical guarantees were still limited.
Provable algorithms for under- and over-complete settings were developed, however their computational complexity and initialization requirements limited their use \cite{Spielman2012, Agarwal14, Arora14, Barak2015}. Tensor factorization algorithms have also been used to learn orthogonal (\cite{Barak2015} and \cite{Ma2016}), and convolutional \citep{Huang2015} dictionaries. More recently, \citep{Rambhatla2019NOODL} proposed \texttt{NOODL}: a simple, scalable gradient descent-based algorithm for joint estimation of the dictionary and the coefficients,  for \textit{exact} recovery of both factors at a linear rate. Although this serves as a great starting point, tensor factorization task cannot be handled by a mere ``lifting'' due to the induced dependence structure.

Overall, the existing provable techniques (Table~\ref{tab:compare_tensor}) in addition to being computationally expensive, incur an irreducible error (bias) in estimation and apply to cases where all factors obey the same conditions. Consequently, there is a need for fast and scalable provable tensor factorization techniques which can recover structured factors with no estimation bias.

\paragraph{Notation.} Bold, lower-case (\[\b{v}\]) and upper-case (\[\b{M}\]) letters, denote vectors and matrices, respectively. We use \[\b{M}_i\], \[\b{M}_{(i,:)}\], \[\b{M}_{ij}\] (also \[\b{M}(i,j)\]), and \[\b{v}_i\] (also \[\b{v} (i)\]) to denote the \[i\]-th column, \[i\]-th row, \[(i,j)\] element, respectively. We use \[``\odot"\] and \[``\otimes"\] to denote the Khatri-Rao (column-wise Kronecker product) and Kronecker product, respectively. Next, we use \[(\cdot)^{(n)}\] to denote the \[n\]-th iterate, and \[(\cdot)_{(n)}\] for the \[n\]-th data sample.  We also use standard Landau notations \[\cO(\cdot), \Omega(\cdot)\] (\[\tilde{\cO}(\cdot), \tilde{\Omega}(\cdot)\]) to denote the asymptotic behavior (ignoring log factors). Also, for a constant \[L\] (independent of \[n\]), we use \[g(n) = \mathcal{O}^*(f(n))\] to indicate that \[g(n) \leq L f(n)\]. We use \[c(\cdot)\] for constants determined by the quantities in \[(\cdot)\]. Also, we define \[\HT_{\tau}(z) := z\cdot\mathbbm{1}_{|z|\geq \tau}\] as the hard-thresholding operator, where ``\[\mathbbm{1}\]'' is the indicator function, and \[\supp(\cdot)\] for the support (set of non-zero elements) and \[\sgn(\cdot)\] for element-wise sign. Also, \[(.)^{(r)}\] denotes potential iteration dependent parameters. See Appendix~\ref{app:summary_notation}.
        
\vspace{-6pt}
\section{Problem Formulation} \label{sec:probform}
\vspace{-4pt}
%\begin{minipage}{0.4\textwidth}
Our formulation is shown in Fig.~\ref{fig:Tensor_fac}. Here, our aim is to recover the CP factors of tensors \[\{\underline{\b{Z}}^{(t)}\}_{t=0}^{T-1}\] assumed to be generated at each iteration as per \eqref{CPD}. Without loss of generality, let the factor \[\b{A}^*\] follow some incoherence assumptions, while the factors \[\b{B}^{*(t)}\] and \[\b{C}^{*(t)}\] be sparse. Now, the \emph{mode-$1$ unfolding} or matricization \[\b{Z}_1^{(t)} \in  \mathbb{R}^{JK \times n} \] of \[\underline{\b{Z}}^{(t)}\] is given by
%\end{minipage}\hspace{4pt}
%\begin{minipage}{0.57\textwidth}
%	\vspace{-45pt}

%\end{minipage}

\vspace{-7pt}

\vspace{-0pt}
\begin{align}
\label{mode-1-T}
\b{Z}_1^{(t)\top}= \b{A}^*(\b{C}^{*(t)}\odot \b{B}^{*(t)})^\top = \b{A}^*\b{S}^{*(t)},
\end{align}
where \[\b{S}^{*(t)} \in \mathbb{R}^{m \times JK}\] is \[\b{S}^{*(t)} := (\b{C}^{*(t)}\odot \b{B^{*(t)}})^\top\]. As a result, matrix \[\b{S}^{*(t)}\] has a \emph{transposed Khatri-Rao} structure, i.e. the \[i\]-th row of \[\b{S}^{*(t)}\] is given by \[(\b{C}^{*(t)}_i\otimes \b{B}_i^{*(t)})^\top\]. Further, since \[\b{B}^{*(t)}\] and \[\b{C}^{*(t)}\] are sparse, only a few  \[\b{S}^{*(t)}\] columns (say \[p\]) have non-zero elements. Now, let \[\b{Y}^{(t)} \in \mathbb{R}^{n \times p}\] be a matrix formed by collecting the non-zero \[\b{Z}_1^{(t)\top}\] columns, we have
\begin{align}
\label{eq:dl}
\b{Y}^{(t)} = \b{A}^*\b{X}^{*(t)}, 
\end{align}
where \[\b{X}^{*(t)} \in \mathbb{R}^{m \times p}\] denotes the sparse matrix corresponding to the non-zero columns of  \[\b{S}^{*(t)}\]. Since recovering \[\b{A}^*\] and \[\b{X}^{*(t)}\] given \[\b{Y}^{(t)}\] is a dictionary learning task~\eqref{eq:model_dl}, we can now employ a dictionary learning algorithm (such as \citetalias{Rambhatla2019NOODL}) which \emph{exactly} recovers  \[\b{A}^*\] (the \emph{dictionary}) and \[\b{X}^{*(t)}\] (the \emph{sparse coefficients}) at each time step \[t\] of the (online) algorithm. The exact recovery of \[\b{X}^{*(t)}\] enables recovery of \[\b{B}^{*(t)}\] and \[\b{C}^{*(t)}\] using our untangling procedure.

	\begin{figure}[t]
	\vspace{-5pt}
		\centering
		\resizebox{0.72\textwidth}{!}{\protect{
				\begin{tikzpicture}
				\node[anchor=south west,inner sep=0] (image) at (0.2,0.8) {\includegraphics[width=0.845\textwidth]{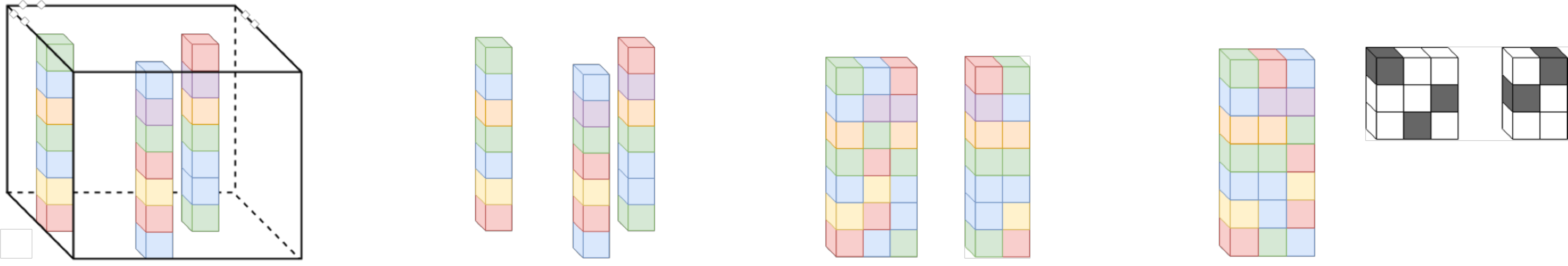}};
				\node[align=center]  at (0,2.2) {\large $n$};
				\node[align=center]  at (0.2,1) {\large $K$};
				\node[align=center]  at (1.8,0.6) {\large $J$};
				\node[align=center] at (1.8,1.8) {\Large$\underline{\b{Z}}^{(t)} $};
				\node[align=center] at (3.7,1.8) {\huge$\rightarrow$};
				\node[align=center] at (6.8,1.8) {\huge$\rightarrow$};
				\node[align=center] at (11.6,1.8) {\Large $\b{A}^* $};
				\node[align=center] at (11.5,3.2) {\large Dictionary};
				\node[align=center] at (13.4,1.6) {\Large$\b{X}^{*(t)} $};
				\node[align=center] at (13.4,1.1) {\large Sparse};
				\node[align=center] at (13.4,0.7) {\large Coefficients};
				\node[align=center] at (5.1,3.3) {\large Dense Columns};
				\node[align=center] at (8.4,3.3) {\large Collected to};
				\node[align=center] at (8.4,2.9) {\large form a Matrix};
				\node[align=center] at (8.6,2.4) {\large \[...\]};
				\node[align=center] at (13.4,2.6) {\large \[...\]};
				\node[align=center] at (10.3,1.8) {\huge$=$};
				\node[align=center] at (8.5,1.8) {\Large$\b{Y}^{(t)} $};
				\end{tikzpicture}}}
		\vspace{-5pt}
		\caption{\protect{Problem Formulation: The dense columns of \[{\b{Z}}^{(t)} \in \mathbb{R}^{n\times J \times K}\] are collected in a matrix \[\b{Y}^{(t)}\]. Then \[\b{Y}^{(t)}\] is viewed as arising from a dictionary learning model. }}\label{fig:Tensor_fac}
	%	\vspace{-14pt}
	\end{figure}
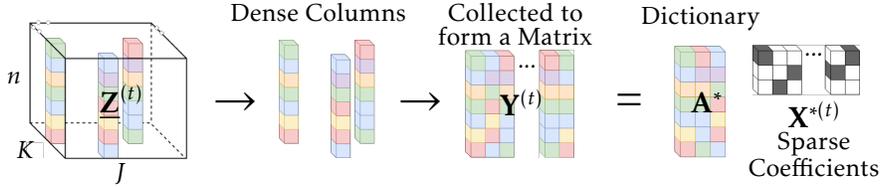

\vspace{-6pt}
\section{Algorithm} \label{sec:algorithm}
\vspace{-4pt}
\setlength{\textfloatsep}{3pt}

\vspace{-3pt}
We begin by presenting the algorithmic details referring to relevant assumptions, we then analyze the model assumptions and the main result in section~\ref{sec:main_res}. \texttt{TensorNOODL} (Alg.~\ref{alg:main_alg_tens}) operates by casting the tensor decomposition problem as a dictionary learning task. Initially, Alg.~\ref{alg:main_alg_tens} is given a \[(\epsilon_0, 2)\]-close (defined below) estimate \[\b{A}^{(0)}\] of \[\b{A}^*\] for \[\epsilon_0 = \mathcal{O}^*(1/\log(n))\]. This initialization, which can be achieved by algorithms such as \citet{Arora15}, ensures that the estimate \[\b{A}^{(0)}\] is both, column-wise and in spectral norm sense, close to \[\b{A}^*\].
\vspace{-5pt} 
\begin{definition}[\[(\epsilon, \kappa)\]-closeness]\label{def:del_kappa}
	\textit{Matrix \[\b{A}\] is \[(\epsilon, \kappa)\]-close to \[\b{A}^*\] if \[\|\b{A} - \b{A}^*\| \leq \kappa\|\b{A}^*\|\], and if there is a permutation \[\pi: [m] \rightarrow [m]\] and collection of signs \[\sigma: [m] \rightarrow \{\pm 1\}\] s.t. \[\|\sigma(i)\b{A}_{\pi(i)} - \b{A}^*_i\|\leq \epsilon, ~\forall~ i\in [m]\].}\vspace{-5pt}
\end{definition}
%\vspace{-10pt}

Next, we sequentially provide the tensors to be factorized, \[\{\underline{\b{Z}}^{(t)}\}_{t=0}^{T-1}\] (generated independently as per \eqref{CPD}) at each iteration \[t\]. The algorithm proceeds in the following stages. 

\vspace{3pt}
\noindent\textbf{I. Estimate Sparse Matrix \[\b{X}^{*(t)}\]}: We use  \[R\] iterative hard thresholding (IHT) steps \eqref{alg:coeff_iht} -- with step-size \[\eta_x^{(r)}\] and threshold \[\tau^{(r)}\]  chosen according to \ref{assumption:step coeff} -- to arrive at an estimate 	\[\hat{\b{X}}^{(t)}\] (or \[\b{X}^{(R)(t)}\]). Iterations \[R\] are determined by the target tolerance (\[\delta_R\]) of the desired coefficient estimate, i.e. we choose \[R = \Omega(\log(1/\delta_R))\], where \[(1 - \eta_x^{(r)})^R \leq \delta_R\].

\vspace{3pt}

\addtolength{\floatsep}{-16pt}
\addtolength{\abovecaptionskip}{-10pt}
\addtolength{\belowcaptionskip}{-10pt}
%\vspace{-30pt}

\begin{algorithm}[t]
\setstretch{1}
\caption{\textsc{TensorNOODL}: Neurally plausible alternating Optimization-based Online Dictionary Learning for Tensor decompositions.}\label{alg:main_alg_tens}
\SetAlgoLined
\KwInput{Structured tensor \[\underline{\b{Z}}^{(t)} \in \mathbb{R}^{n\times J \times K} \] at each \[t\] generated as per \eqref {CPD}. %; See Note~\ref{note:fresh}. 
Parameters \[\eta_A\], \[\eta_x\], \[\tau\], \[T\], \[C\], and \[R\] as per \ref{assumption:dist}, \ref{assumption:step dict}, and \ref{assumption:step coeff}.}
\KwOutput{Dictionary \[\b{A}^{(t)}\] and the factor estimates \[\b{B}^{(t)}\] and \[\b{C}^{(t)}\] (corresponding to \[\underline{\b{Z}}^{(t)}\]) at \[t\]. }
\KwInit{Estimate \[\b{A}^{(0)}\], which is \[(\epsilon_0, 2)\]-near to \[\b{A}^*\] for \[\epsilon_0 = \mathcal{O}^*(1/\log(n))\]; see Def.~\ref{def:del_kappa}.}
\For{\[t=0\] \textbf{to} \[T-1\]}{
\vspace*{2pt}
\hspace{-3pt}\textbf{I. Estimate Sparse Matrix \[\b{X}^{*(t)}\]:} \vspace{-4pt}
\begin{flalign}\label{alg:coeff_init}
&\textbf{Initialize:}\hspace{2pt}\b{X}^{(0)(t)} =\HT_{C/2}(\b{A}^{(t)^\top}\b{Y}^{(t)})~ &\hspace{-21pt}~\text{See Def.}\ref{def:dist_bc}  &
\end{flalign}\\[-4pt]
	\For{\[r = 0\]  \textbf{to} \[R-1\]}{\vspace{-6pt}
	{\begin{flalign}\label{alg:coeff_iht}
	&\b{X}^{(r+1)(t)} = \HT_{\tau^{(r)}}(\b{X}^{(r)(t)} - \eta_{x}^{(r)}\b{A}^{(t)^\top}(\b{A}^{(t)}\b{X}^{(r)} -\b{Y}^{(t)}))&
	\end{flalign}\vspace{-12pt}}}

%	Update: \[ \b{x}^{(r+1)}_{(j)} = \HT_{\tau^{(r)}}(\b{x}^{(r)}_{(j)} - \eta_{x}^{(r)}~\b{A}^{(t)^\top}(\b{A}^{(t)}\b{x}^{(r)}_{(j)} - \b{y}_{(j)}))\]
\[\hat{\b{X}}^{(t)} :=\b{X}^{(R)(t)} \]. \\
%(We drop \[(.)^{(t)}\] in our discussion for simplicity.)
\vspace*{3pt}

\hspace{-3pt}\textbf{II. Recover Sparse Factors  \[\b{B}^*\] and \[\b{C}^*\]}: \\
\vspace{6pt}
Form \[\hat{\b{S}}^{(t)}\] by putting back columns of \[\hat{\b{X}}^{(t)}\] at the non-zero column locations of \[\b{Z}_1^{(t)\top}\].\\
\vspace{6pt}
\hspace* {4pt}\[[\hat{\b{B}}^{(t)}, ~\hat{\b{C}}^{(t)}] = \texttt{UNTANGLE-KRP}(\hat{\b{S}}^{(t)})\] \\
%Normalize: \[\b{A}^{(t+1)}_i = \b{A}^{(t+1)}_i/\|\b{A}^{(t+1)}_i\|\] for \[i \in [m]\]. % took out "optional" this step is not optional
\vspace*{3pt}

\hspace{-3pt}\textbf{III. Update Dictionary Factor \[\b{A}^{(t)}\]:}\vspace{-14pt}

\begin{flalign}\label{eq:emp_grad_est}
&\hspace{0.14in}\hat{\b{g}}^{(t)} = \tfrac{1}{p}(\b{A}^{(t)}\hat{\b{X}}^{(t)}_{\rm{indep}}  - \b{Y}^{*(t)})\sgn(\hat{\b{X}}^{(t)}_{\rm{indep}}  )^\top\hspace{-0.05in}&
\end{flalign}\vspace{-22pt}
%Form empirical gradient estimate : \[\hat{\b{g}}^{(t)} = \tfrac{1}{p}\sum_{j = 1}^{p}(\b{A}^{(t)}\hat{\b{x}}_{(j)}  - \b{y}_{(j)})\sgn(\hat{\b{x}}_{(j)} )^\top\] \\ %gr = (Y_m - M.*(A*XS))*(XS)';
\begin{flalign}\label{eq:apx_grad}
&\hspace{0.14in}\b{A}^{(t+1)} = \b{A}^{(t)} - \eta_A~\hat{\b{g}}^{(t)}&
\end{flalign}\vspace{-22pt}
%Take a gradient descent step: \[\b{A}^{(t+1)} = \b{A}^{(t)} - \eta_A~\hat{\b{g}}^{(t)}\]\\
\begin{flalign*}
&\hspace{0.14in}\b{A}^{(t+1)}_i = \b{A}^{(t+1)}_i/\|\b{A}^{(t+1)}_i\| ~\forall~i \in [m]&
\end{flalign*}\vspace{-22pt}\\

%\vspace*{-8pt}
} 
\end{algorithm}
%\addtolength{\abovecaptionskip}{-10pt}
%\addtolength{\belowcaptionskip}{-10pt}
%\addtolength{\abovecaptionskip}{-10pt}
%\addtolength{\belowcaptionskip}{-10pt}
%\vspace{10pt}
%\addtolength{\floatsep}{-16pt}
%\addtolength{\abovecaptionskip}{-10pt}
%\addtolength{\belowcaptionskip}{-10pt}
\begin{algorithm}[t]
	\SetAlgoSkip{-20pt}
	\KwInput{Estimate \[\hat{\b{S}}^{(t)}\] of the KRP \[\b{S}^{*(t)}\]}
	\KwOutput{Estimates \[\hat{\b{B}}^{(t)}\] and \[\hat{\b{C}}^{(t)}\] of \[\b{B}^{*(t)}\] and \[\b{C}^{*(t)}\].}
	\For{\[i = 1 \dots m\]}{
		\textbf{Reshape:} $i$-th row of \[\hat{\b{S}}^{(t)}\]  into \[\b{M}^{(i)} \in \mathbb{R}^{J \times K}\].\\
		\hspace{-2pt}	\textbf{Set:} \[\hat{\b{B}}_i^{(t)} \leftarrow \sqrt{\sigma_1} \b{u}_1\], and \[\hat{\b{C}}_i^{(t)} \leftarrow  \sqrt{\sigma_1}\b{v}_1\], where \[\sigma_1\], \[\b{u}_1\], and \[\b{v}_1\] are the principal left and right singular vectors of \[\b{M}^{(i)} \], respectively.
	}
	\caption{\textsc{Untangle Khatri-Rao Product (KRP)}: Recovering the Sparse factors}
	\label{algo:KRP_algo}
\end{algorithm}
%\vspace*{-5pt}
%\end{minipage}}
%\vspace{-2pt}

%Therefore to recreate the Khatri-Rao structure, we form the estimate $\hat{\b{S}}^{(t)}$ of $\b{S}^{*(t)}$ by placing columns of \[\hat{\b{X}}^{(t)}\] at their corresponding locations of \[\b{Z}_1^{(t)\top}\]. Since our main result guarantees that \[\hat{\b{S}}^{(t)}\] has the same sign and  and support as \[\hat{\b{S}}^{*(t)}\], we provably recover the original Khatri-Rao product structure.  
%Now, we can use the SVD-based algorithm (Alg.~\ref{algo:KRP_algo}) to recover the sparse factors using the element-wise \[\zeta\]-close estimate of \[\b{S}^{*(t)}\], i.e., {$|\hat{\b{S}}_{ij}^{(t)} - \b{S}^{*(t)}_{ij}| \leq \zeta$} by Alg.~\ref{alg:main_alg_tens}. 

%\vspace{-10pt}
%\setlength{\intextsep}{1\baselineskip}

%\hspace{-18pt}\begin{minipage}{0.5\textwidth}

 \noindent\textbf{II. Estimate \[\b{B}^*\] and \[\b{C}^*\]}: As discussed in section~\ref{sec:probform}, the tensor matricization leads to a Khatri-Rao dependence structure between the factors \[\b{B}^{*(t)}\] and \[\b{C}^{*(t)}\]. To recover these, we develop a SVD-based algorithm (Alg.~\ref{algo:KRP_algo}) to estimate sparse factors (\[\b{B}^{*(t)}\] and \[\b{C}^{*(t)}\]) using an element-wise \[\zeta\]-close estimate of \[\b{S}^{*(t)}\], i.e., {$|\hat{\b{S}}_{ij}^{(t)} - \b{S}^{*(t)}_{ij}| \leq \zeta$}.  Here, we form the estimate $\hat{\b{S}}^{(t)}$ of $\b{S}^{*(t)}$ by placing columns of \[\hat{\b{X}}^{(t)}\] at their corresponding locations of \[\b{Z}_1^{(t)\top}\] to the Khatri-Rao structure (\texttt{TensorNOODL} is agnostic to the tensor structure of the data since it only operates on the non-zero fibers \[\b{Y}^{(t)}\] of \[\b{Z}_1^{(t)\top}\] see \eqref{eq:dl} and Fig.~\ref{fig:Tensor_fac}). Our recovery result for \[\hat{\b{X}}^{(t)}\] guarantees that \[\hat{\b{S}}^{(t)}\] has the same sign and  and support as \[\hat{\b{S}}^{*(t)}\], we therefore provably recover the original Khatri-Rao product structure.    
	
	\vspace{5pt}
	
	\noindent\textbf{III. Update \[\b{A}^*\] estimate }: We use \[\b{X}^{*(t)}\] estimate to update \[\b{A}^{(t)}\] by an approximate gradient descent strategy \eqref{eq:apx_grad} with step size \[\eta_A\]  (\ref{assumption:step dict}). The algorithm requires \[T= \max( \Omega(\log(1/\epsilon_T)), \Omega(\log(\sqrt{s}/\delta_T)))\] for {\small$\|\b{A}^{(T)}_i \hspace{-1pt}-\hspace{-1pt} \b{A}^*_i\| \hspace{-2pt}\leq\hspace{-2pt} \epsilon_T, \forall i\in\hspace{-2pt} [m]$} and {\small$|\hat{\b{X}}_{ij}^{(T)} - \b{X}_{ij}^{*(t)}| \leq \delta_T$}.
	
	\vspace{6pt}
	\noindent\textbf{Runtime}: The runtime of \texttt{TensorNOODL} is \[\c{O}(mnp\log(\tfrac{1}{\delta_R}) \max( \log(\tfrac{1}{ \epsilon_T}),\log(\tfrac{\sqrt{s}}{\delta_T}))\] for \[p \hspace{-2pt}= \hspace{-2pt}\Omega(ms^2)\].  Furthermore, since \[\b{X}^*\] columns can be estimated independently in parallel, \texttt{TensorNOODL} is scalable and can be implemented in highly distributed settings. 
	
		\vspace{-6pt}
		\section{Main Result}\label{sec:main_res}
		\vspace{-4pt}
		We now formalize our model assumptions and state our main result; details in Appendix~\ref{sec:theory}. 
		
	\setlength{\textfloatsep}{20pt}

%\end{minipage}\scalebox{0.81}{
%\begin{minipage}{.62\textwidth}

\vspace{3pt}
\noindent\textbf{Model Assumptions}: First, we require that \[\b{A}^*\] is \[\mu\]-incoherent
		(defined below), which defines the notion of incoherence for \[\b{A}^*\] columns \hspace{-3pt}(refered to as \emph{dictionary}).

	\vspace{-0pt}
	\begin{definition} \label{def:mu}
		\textit{ A matrix \[\b{A} \hspace{-3pt}\in \hspace{-3pt}\mathbb{R}^{n \times m}\] with unit-norm columns is \[\mu\]-incoherent if for all \[{i \neq j}\] the inner-product between the columns of the matrix follow \[ |\langle \b{A}_i, \b{A}_j\rangle| \leq \mu/\sqrt{n}\].}
	\end{definition}%
	\vspace{-0pt} 
This ensures that dictionary columns are distinguishable, akin to relaxing the orthogonality constraint. Next, we assume that sparse factors \[\b{B}^{*(t)}\] and \[\b{C}^{*(t)}\] are drawn from distribution classes $\Gamma_{\alpha, C}^{\rm sG}$ and \[\Gamma_{\beta}^{\rm Rad}\], respectively, here \[\Gamma_{\gamma, C}^{\rm sG}\] and  $\Gamma_{\gamma}^{\rm Rad}$ are defined as follows.

%\vspace{-7pt}
\vspace{-0pt}
\begin{definition} [Distribution Class \[\Gamma_{\gamma, C}^{\rm sG}\] and  \[\Gamma_{\gamma}^{\rm Rad}\]]
	\label{def:dist_bc} A matrix \[\b{M}\] belongs to class %a Distribution class 
	\vspace{-7pt}
		\begin{itemize}[leftmargin={25pt}]
		\item \[\Gamma_{\gamma}^{\rm Rad}\]: if each entry of \[\b{M}\] is independently non-zero with probability \[\gamma\], and the values at the non-zero locations are drawn from the Rademacher distribution.
		\item \[\Gamma_{\gamma, C}^{\rm sG}\]:  if each entry of \[\b{M}\] is independently non-zero with probability \[\gamma\], and the values at the non-zero locations are sub-Gaussian, zero-mean with unit variance and bounded away from \[C\] for some positive constant \[C\leq1\], i.e., \[|\b{M}_{ij}| \geq C\] for \[(i,j) \in \supp(\b{M})\].
			\vspace{-6pt}
		
		\end{itemize}
\end{definition} 
\vspace{-0pt}
In essence, we assume that elements of \[\b{B}^{*(t)}\] (\[\b{C}^{*(t)}\]) are non-zero with probability \[\alpha\] (\[\beta\]), and that for \[\b{B}^{*(t)}\] the values at the non-zero locations are drawn from a zero-mean unit-variance sub-Gaussian distribution, bounded away from zero, and the non-zero values of \[\b{C}^{*(t)}\] are drawn from the Rademacher distribution
\footnote{The non-zero entries of \[\b{C}^{*(t)}\] can also be assumed to be drawn from a sub-Gaussian distribution (like \[\b{B}^{*(t)}\]) at the expense of sparsity, incoherence, dimension(s), and sample complexity. Specifically when non-zero entries of \[\b{B}^{*(t)}\] and \[\b{C}^{*(t)}\] are drawn from sub-Gaussian distribution (as per \[\Gamma_{\gamma, C}^{\rm sG}\]), we will need the dictionary learning algorithm to work with the coefficient matrix \[\b{X}^{*(t)}\] (formed by product of entries of \[\b{B}^{*(t)}\] and \[\b{C}^{*(t)}\]) which now has sub-Exponential non-zero entries.}.% Since sub-Exponential tails decay slower than sub-Gaussians, we will need additional restrictions on other parameters.}.

%\vspace{5pt}
%\begin{minipage}{0.4\textwidth}
\paragraph{Analyzing the Khatri-Rao Dependence:} We now turn our attention to  the KR dependence structure of \[\b{S}^{*(t)}\]. Fig.~\ref{fig:krp_row} shows a row of the matrix \[\b{S}^{*(t)}\], each entry of which is
 formed by multiplication of an element of \[\b{C}_i^{*(t)}\] with each element of columns of \[\b{B}_i^{*(t)}\]. Consequently, each row of the resulting matrix \[\b{S}^{*(t)}\] has \[K\] \emph{blocks} (of size \[J\]), where the \[k\]-th block is controlled by \[\b{C}_{k,i}^{*(t)}\], and therefore the \[(i,j)\]-th entry of \[\b{S}^{*(t)}\] can be written as 
\vspace{-2pt}
\begin{align}\label{eq:dep_str}
\b{S}_{ij}^{*(t)} = \b{C}^{*(t)}(\left\lfloor\tfrac{j}{J}\right\rfloor+1, i)~\b{B}^{*(t)}(j - J \left\lfloor\tfrac{j}{J}\right\rfloor,i).\vspace{-2pt}
\end{align}

	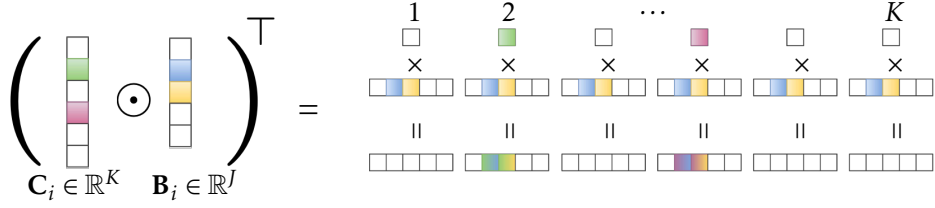
\begin{figure}[t]
 	%	\vspace{-5pt}
 		\centering
 		\resizebox{0.75\textwidth}{!}{
 			\begin{tikzpicture}
 			\node[anchor=south west,inner sep=0] (image) at (1.3,0) {\includegraphics[width=0.84\textwidth]{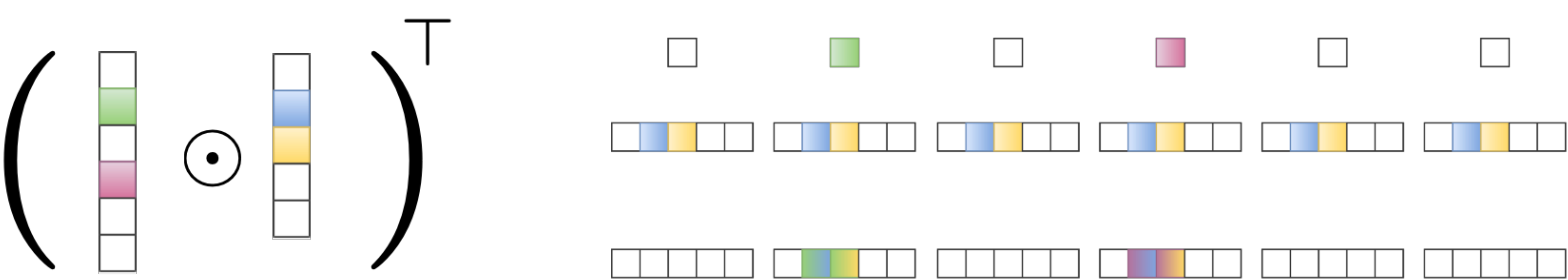}};
 			\node[align=center]  at (2.3,-0.2) {\large$\b{C}_i \in \mathbb{R}^K$};
 			\node[align=center]  at (4.1,-0.2) {\large$\b{B}_i \in \mathbb{R}^J$};
 			\node[align=center] at (5.8,0.9) {\Large$=$};
 			\node[align=center] at (7.4,1.6) {\large $\times$};
 			\node[align=center] at (7.4,2.4) {\large$1$};
 			\node[align=center] at (7.45,0.6) {\large\rotatebox{90}{$\,=$}};
 			\node[align=center] at (8.8,1.6) {\large $\times$};
 			\node[align=center] at (8.8,2.4) {\large$2$};
 			\node[align=center] at (8.85,0.6) {\large\rotatebox{90}{$\,=$}};
 			\node[align=center] at (11.7,1.6) {\large $\times$};
 			\node[align=center] at (11.75,0.6) {\large\rotatebox{90}{$\,=$}};
 			\node[align=center] at (10.3,1.6) {\large $\times$};
 			\node[align=center] at (10.35,0.6) {\large\rotatebox{90}{$\,=$}};
 			\node[align=center] at (13.2,1.6) {\large $\times$};
 			\node[align=center] at (13.2,0.6) {\large\rotatebox{90}{$\,=$}};
 			\node[align=center] at (14.6,1.6) {\large $\times$};
 			\node[align=center] at (14.6,2.4) {\large$K$};
 			\node[align=center] at (14.65,0.6) {\large\rotatebox{90}{$\,=$}};
 			\node[align=center] at (11,2.4) {\large$\dots$};
 			\end{tikzpicture}}
 	%	\vspace{-2pt}
 		\caption{Transposed Khatri-Rao dependence. }\label{fig:krp_row}
 	%	\vspace{-18pt}
 	\end{figure}

As a result, depending upon \[\alpha(\beta)\], \[\b{S}^{*(t)}\] (consequently \[\b{Z}_1^{(t)\top}\]) may have all-zero (degenerate) columns. Therefore, we only use the non-zero columns \[\b{Y}^{(t)}\] of \[\b{Z}_1^{(t)\top}\]. Next, although elements in a column of \[\b{S}^{*(t)}\] are independent, the KR structure induces a dependence structure across elements in a row when the elements depend on the same \[\b{B}^{*(t)}\] or \[\b{C}^{*(t)}\] element; see \eqref{eq:dep_str}. In practice, we can use all non-zero columns of \[\b{Z}_1^{(t)\top}\], however for our probabilistic analysis, we require an independent set of samples. We form one such set by selecting the first column from the first block, second column from the second block and so on; see Fig.~\ref{fig:krp_row}. This results in a \[L=\min(J,K)\] independent samples set for a given \[\b{Z}_1^{(t)\top}\]. With this, and our assumptions on sparse factors ensure that the \[L\] independent columns of \[\b{X}^{*(t)}\] (\[\b{X}_{\rm{indep}}^{*(t)}\]) belong to the distribution class \[\c{D}\] defined as follows.

\vspace{-2pt}
\begin{definition}[Distribution class \[\c{D}\]]\label{dist_x} %Made imp change j to i in expectation
\textit{The coefficient vector \[\b{x}^*\] belongs to an unknown distribution \[\c{D}\], where the support \[S = \supp(\b{x}^*)\] is at most of size \[s\], \[\b{Pr}[i \in S] = \Theta(s/m)\] and \[\b{Pr}[i, j \in S] = \Theta(s^2/m^2)\]. Moreover, the distribution is normalized such that \[\b{E}[\b{x}_i^*|i \in S] = 0\] and \[\b{E}[\b{x}_i^{*^2}|i \in S] = 1\], and when \[i \in S\],  \[|\b{x}^*_i| \geq C\] for some constant \[C \leq 1\]. In addition, the non-zero entries are sub-Gaussian and pairwise independent conditioned on the support.}
\end{definition}
\vspace{-2pt}  
Further, the  \[(\epsilon_0,2)\]-closeness (Def.~\ref{def:del_kappa})  ensures that the signed-support (defined below) of the coefficients are recovered correctly (with high probability).
\vspace{-2pt}
\begin{definition}\label{def:signed-support}
\textit{The signed-support of a vector \[\b{x}\] is defined as \[\sgn(\b{x})\cdot \supp(\b{x})\].}
\end{definition}
%\vspace{-2pt}
\textbf{Scaling and Permutation Indeterminacy}: The unit-norm constraint on \[\b{A}^*\] implies that the scaling (including the sign) ambiguity only exists in the recovery of \[\b{B}^{*(t)}\] and \[\b{C}^{*(t)}\]. To this end, we will regard our algorithm to be successful in the following sense.
\vspace{-2pt}
\begin{definition}[Equivalence]\label{assumption:scale}\textit{Factorizations \[[\![ \b{A},\b{B},\b{C}]\!]\] are considered equivalent up to scaling, i.e, \[	[\![ \b{A}, \b{B},\] \[\b{C}]\!]= [\![ \b{A}^*, \b{B}^*\b{D}_{\sigma_b}, \b{C}^*\b{D}_{\sigma_c}]\!]\]
 where \[\sigma_b\](\[\sigma_c\]) is a vector of scalings (including signs) corresponding to columns of the factors \[\b{B}\] and \[\b{C}\], respectively.}
\end{definition}
\vspace{-2pt}
\textbf{Dictionary Factor Update Strategy}: We use an \emph{approximate} (we use an estimate of \[\b{X}^{*(t)}\]) gradient descent-based strategy \eqref{eq:emp_grad_est} to update \[\b{A}^{(t)}\] by finding a direction \[\b{g}_i^{(t)}\] to ensure descent. Here, the \[(\Omega(s/m), \Omega(m/s),0 )\]-correlatedness (defined below) of the expected gradient vector is sufficient to make progress (\[``0"\] indicates no bias); see \cite{Candes2015, Chen2015, Arora15, Rambhatla2019NOODL}.

\vspace{-2pt}
\begin{definition}\label{def:grad_alpha_beta}
	\textit{A vector \[\b{g}^{(t)}_i\] is  \[(\rho_{-}, \rho_{_+}, \zeta_t)\]-correlated with a vector \[\b{z}^*\] if for any vector \[\b{z}^{(t)}\]
	\begin{align*}
	\langle \b{g}^{(t)}_i, \b{z}^{(t)} - \b{z}^{*}\rangle \geq \rho_{-}\|\b{z}^{(t)} - \b{z}^{*}\|^2 + \rho_{+} \|\b{g}^{(t)}_i\|^2 - \zeta_{t}.
	\end{align*}}
\end{definition}
\vspace{-2pt}

Our model assumptions can be formalized as follows, with which we state our main result.% formalizing the specifics of the model, shown in \eqref{eq:model}. 
\vspace{-5pt}
\begin{enumerate}[label=\textbf{A.\arabic*}, ref=\textbf{A.\arabic*}, leftmargin={*}]
\setlength{\itemsep}{-0.05cm}
\item \label{assumption:mu} \[\b{A}^*\] is \[\mu\]-incoherent (Def.~\ref{def:mu}), where \[\mu = \mathcal{O}(\log(n))\], \[\|\b{A}^*\| = \mathcal{O}(\sqrt{m/n})\] and \[m = \mathcal{O}(n)\];
\item \label{assumption:close} \[\b{A}^{(0)}\] is \[(\epsilon_0, 2)\]-close to \[\b{A}^*\]  as per Def.~\ref{def:del_kappa}, and \[\epsilon_0 = \mathcal{O}^*(1/\log(n))\];
\item \label{assumption:dist} Factors \hspace{0pt}\[\b{B}^{*(t)}\]\hspace{-1pt} and \[\b{C}^{*(t)}\] \hspace{-2pt}are respectively drawn from distributions \[\Gamma_{\alpha, C}^{\rm sG}\]  and  \[\Gamma_{\beta}^{\rm Rad}\] (Def.\ref{def:dist_bc});
\item \label{assumption:k}Sparsity controlling parameters \[\alpha\] and \[\beta\] obey \[\alpha\beta = \mathcal{O}({\sqrt{n}}/{m\mu~\log(n)})\] for \[m = \Omega(\log({\min(J,K)})/\] \[\alpha \beta)\], resulting column sparsity \[s\] of \[\b{S}^{*(t)}\] is \[s = \c{O}(\alpha \beta m)\];
\item \label{assumption:step dict}The dictionary update step-size satisfies \[\eta_A = \Theta(m/s)\];
\item \label{assumption:step coeff}The coefficient update step-size and threshold satisfy \[\eta_x^{(r)} \hspace{-2pt}< \hspace{-2pt}c_1(\epsilon_t, \mu, n, s) \hspace{-2pt}= \hspace{-2pt}\tilde{\Omega}({s}/{\sqrt{n}})\hspace{-2pt}<\hspace{-2pt}1\] and \[\tau^{(r)} = c_2(\epsilon_t, \mu, s, n) = \tilde{\Omega}({s^2}/{n})\] for small constants \[c_1\] and \[c_2\].
\end{enumerate}
%\clearpage
\vspace{-8pt}
\begin{theorem}[Main Result]\label{main_result}\textit{	Suppose a tensor \[\underline{\b{Z}}^{(t)} \in \mathbb{R}^{n \times J \times K}\] provided to Alg.~\ref{alg:main_alg_tens} at each iteration \[t\] admits a decomposition of the form \eqref{CPD}  with factors \[\b{A}^* \in \mathbb{R}^{n \times m}\], \[\b{B}^{*(t)}\in \mathbb{R}^{J \times m}\] and \[\b{C}^{*(t)}\in \mathbb{R}^{K \times m}\] and \[\min(J,K) = \Omega(ms^2)\]. Further, suppose that the assumptions \ref{assumption:mu}-\ref{assumption:step coeff}  hold. %then for \[\eta_A = \Theta(m/k)\] 
Then, given \[R = \Omega({\rm log}(n))\], with probability at least \[(1 - \delta_{\text{alg}})\] for some small constant \[\delta_{\text{alg}}\], the estimate \[\hat{\b{X}}^{(t)}\] at \[t\]-th iteration has the correct signed-support and satisfies
\addtolength{\abovedisplayskip}{2pt}
\addtolength{\belowdisplayskip}{2pt}
\begin{align*}
\vspace{10pt}
(\hat{\b{X}}_{i,j}^{(t)} - \b{X}_{i,j}^{*(t)})^2 \leq \zeta^2 %C_{i_1}^{(R)} 
&:= \mathcal{O}(s(1 - \omega)^{t/2}\|\b{A}_i^{(0)} - \b{A}_i^*\|), \forall (i,j) \in \supp({\b{X}^{*(t)}}).
\vspace{10pt}
\end{align*} 
Furthermore, for some \[0 < \omega < 1/2\], the estimate \[\b{A}^{(t)}\] at \[t\]-th iteration satisfies 
\begin{align*}
\vspace{10pt}
\|\b{A}_i^{(t)} - \b{A}_i^*\|^2 \leq (1 - \omega)^t\|\b{A}_i^{(0)} - \b{A}_i^*\|^2,~\forall~t = 1,2,\dots 
\vspace{10pt}
\end{align*}
Consequently,  Alg.~\ref{algo:KRP_algo} recovers the supports of the sparse factors \[\b{B}^{*(t)}\] and \[\b{C}^{*(t)}\] correctly, and \[\|\hat{\b{B}}_i^{(t)} - \b{B}_i^{*(t)}\|_2 \leq \epsilon_B\] and \[\|\hat{\b{C}}_i^{(t)}- \b{C}_i^{*(t)}\|_2 \leq \epsilon_C\], where \[\epsilon_B = \epsilon_C = \mathcal{O}(\tfrac{\zeta^2}{\alpha\beta})\].
}
\vspace{-3pt}
\end{theorem}
	\begin{figure}[t]
			\centering
			\resizebox{0.95\textwidth}{!}{
				\begin{tabular}{ccc}
					\hspace*{25pt}{\Huge$(J, K) = 100$} 
					&\hspace{-5pt}{\Huge $(J, K) = 300$} 
					&\hspace{-5pt}{\Huge $(J, K) = 500$} \\
					{\rotatebox{90}{\hspace{55pt}\parbox{9cm}{\centering \Huge \textbf{Dictionary Recovery Across Techniques}}}}\hspace{4pt}
					\includegraphics[width=1\textwidth]{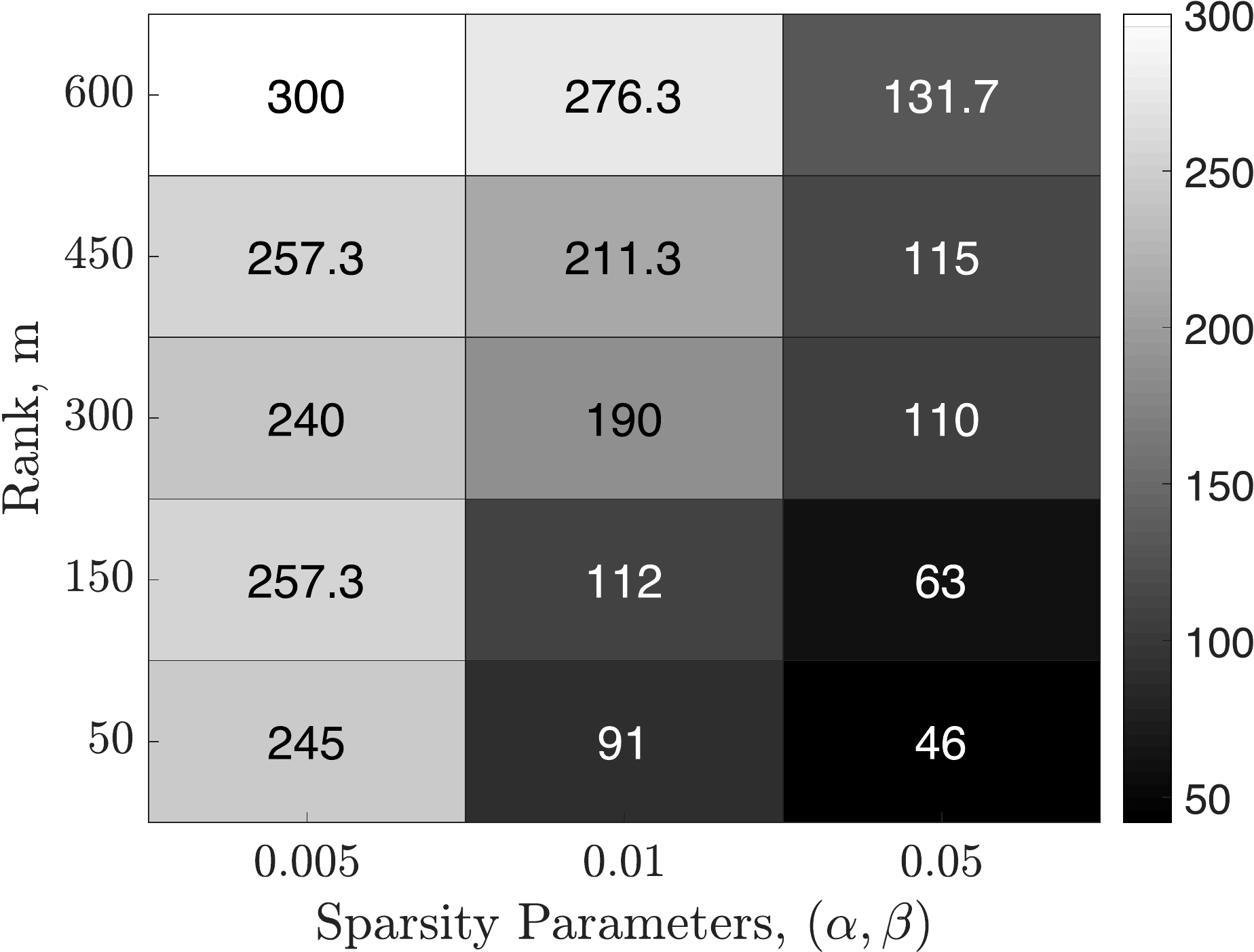}&\hspace*{-6pt}
					\includegraphics[width=1\textwidth]{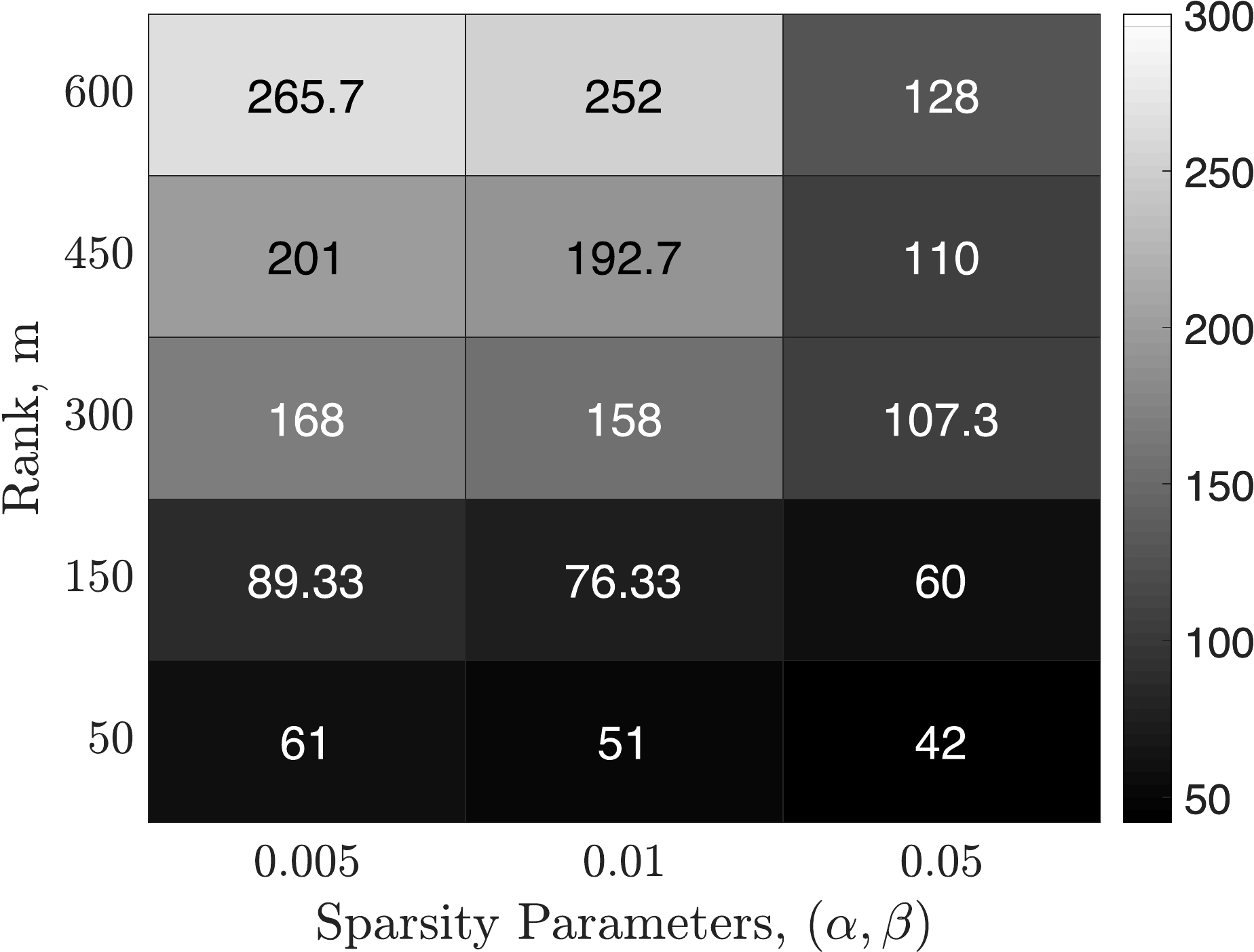}&\hspace*{-6pt}
					\includegraphics[width=1\textwidth]{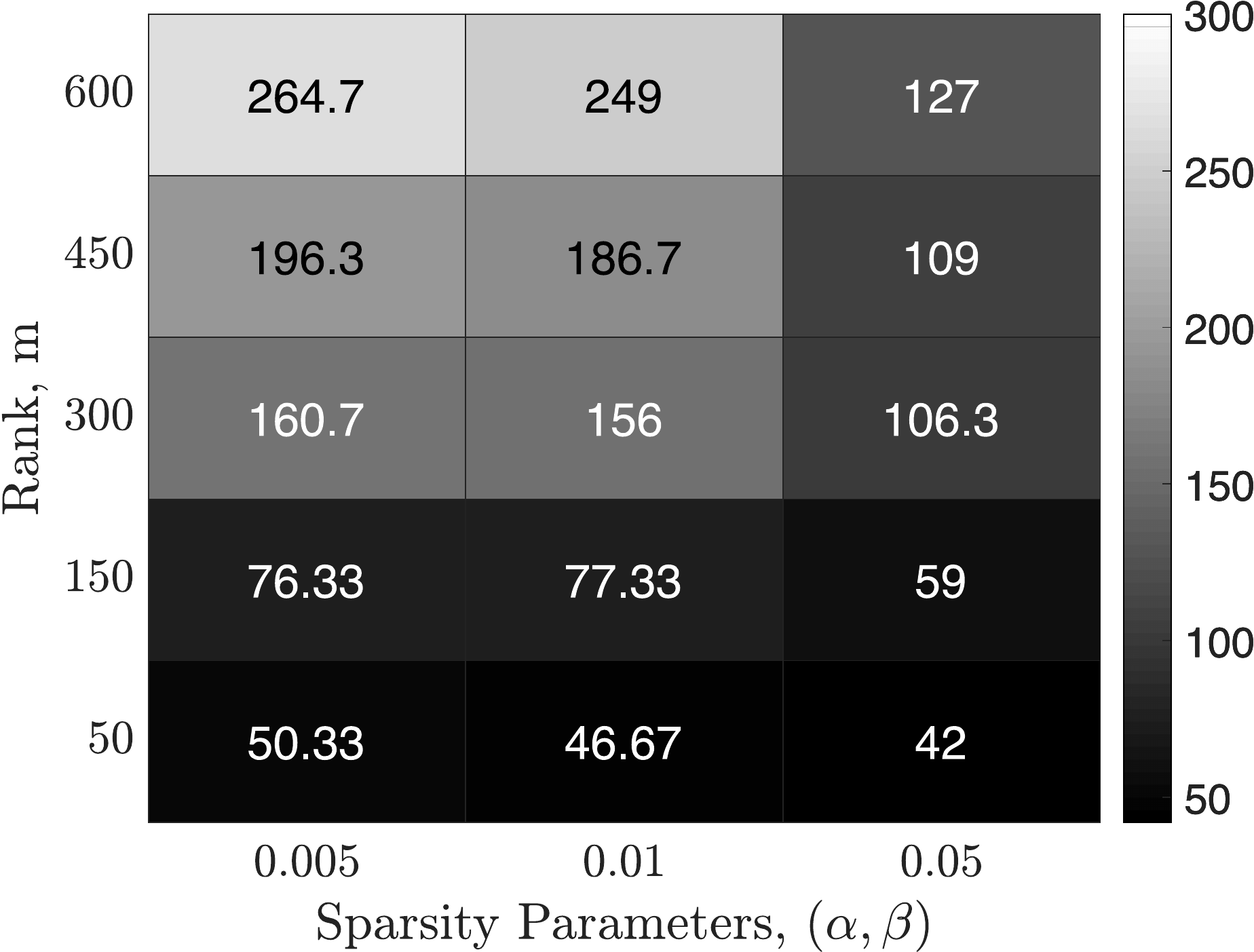} \vspace{10pt}\\
					\vspace{10pt}
					~~~~~~~~{\Huge (a)} 
					&{\Huge(b)} 
					&{\Huge(c)} %\vspace{-15pt}\\
				%	\vspace{-14pt}\\
			\end{tabular}}
			\caption{ Number of iterations for convergence as a surrogate for data samples requirement~\protect\footnotemark. Panels (a), (b), and (c) show the iterations taken by \texttt{TensorNOODL} to achieve a tolerance of \[10^{-10}\] for \[\b{A}\] for \[J\hspace{-2pt}=\hspace{-2pt}K=\hspace{-2pt}\{100, 300,500\}\], respectively across ranks \[m\hspace{-2pt}=\hspace{-2pt}\{50, 150,300, 450, 600\}\] and \[\alpha\hspace{-2pt}=\hspace{-2pt}\beta=\hspace{-2pt}\{0.005, 0.01, 0.05\}\], avg. across $3$ Monte Carlo runs.} % The corresponding recovery results for \[\b{A}\] and \[\b{X}\] (across all techniques) are shown in  Appendix~\ref{app:exp}. }
			\label{fig:num_T}
		%	\vspace{-16pt}
	\end{figure} 
\textbf{Discussion}: Theorem~\ref{main_result} states the sufficient conditions under which, for an appropriate dictionary factor initialization (\ref{assumption:close}), if the incoherent factor \[\b{A}^*\] columns are sufficiently spread out ensuring identifiability (\ref{assumption:mu}), the sparse factors \[\b{B}^{*(t)}\] and \[\b{C}^{*(t)}\] are appropriately sparse (\ref{assumption:dist} and \ref{assumption:k}), and for appropriately chosen learning parameters (step sizes and threshold \ref{assumption:step dict}\[\sim\]\ref{assumption:step coeff}), then Alg.~\ref{alg:main_alg_tens} succeeds whp. Such initializations can be achieved by existing algorithms and can also be used for model selection, i.e., determining \[m\] i.e. \emph{revealing rank}; see \cite{Arora15}. Also, from \ref{assumption:k}, we observe that the sparsity \[s\] (number of non-zeros) in a column of \[\b{S}^{*(t)})\] are critical for the success of the algorithm. Specifically, the upper-bound on \[s\] keeps \[s\] small for the success of dictionary learning, while the lower-bound on \[m\] for given sparsity controlling probabilities\[(\alpha, \beta)\] ensures that there are enough \emph{independent} non-zero columns in \[\b{S}^{*(t)})\] for learning. In other words, this condition ensures that sparsity is neither too low (to avoid degeneracy) nor too high (for dictionary learning), requiring that the independent samples \[L = \min(J,K) = \Omega(ms^2)\], wherein \[s = \c{O}(\alpha\beta m)\] whp.

%\begin{minipage}{0.4\textwidth}
	\vspace{-5pt}
\section{Numerical Simulations}
\label{sec:sims}
\vspace{-4pt}

We evaluate \texttt{TensorNOODL} on synthetic and real-world data; more results in Appendix~\ref{app:exp}.  
%\vspace{-6pt}
\subsection{Synthetic data evaluation}
\vspace{-2pt}
\noindent\textbf{Experimental set-up}: We compare \texttt{TensorNOODL} with online dictionary learning algorithms presented in \cite{Arora15} (\texttt{Arora(b)} (incurs bias) and {\texttt{Arora(u)} (claim no bias)}), and \cite{Mairal09}, which can be viewed as a variant of ALS (matricized) \footnotemark.
%\end{minipage}\hspace{2pt}
\footnotetext{As discussed, the provable tensor factorization algorithms shown in Table~\ref{tab:compare_tensor}, are suitable only for cases wherein all the factors obey same structural assumptions, and also are not online.}
%\begin{minipage}{0.59\textwidth}
%	\vspace{-25pt}
	%{\footnotesize
	%}
%\end{minipage}
\footnotetext{\label{foot:sam}  Our algorithm takes a fresh tensor at each \[t\], we use \[T\] as a surrogate for sample requirement.}
\vspace{-2pt}
  We analyze the recovery performance of the algorithms across different choices of tensor dimensions \[J=K =\] \[\{100,~300, ~500\}\]  for a fixed \[n = 300\], rank \[m = \{50, 150, 300, 450, 600\}\], and the sparsity parameters \[\alpha = \beta = \{0.005,\] \[0.01, 0.05\}\] of factors \[\b{B}^{*(t)}\] and \[\b{C}^{*(t)}\], across $3$ Monte-Carlo runs~\footnote{We fix \[(J, K)\] \& \[(\alpha,\beta)\], but \texttt{TensorNOODL} can also be used with iteration-dependent parameters.}. %The simulation results corresponding to  \[\alpha = \beta = \{0.005, 0.01, 0.05\}\] are shown in Table~\ref{alpha_0_005}, \ref{alpha_0_01}, and \ref{alpha_0_05}, respectively.
We draw entries of \[\b{A}^* \in \mathbb{R}^{n \times m}\] from \[\c{N}(0,1)\], and normalize its columns to be unit-norm. To form \[\b{A}^{(0)}\], we perturb \[\b{A}^*\] with random Gaussian noise and normalized its columns, such that it is column-wise \[2/\log(n)\] away from \[\b{A}^*\] (\ref{assumption:close}). To form \[\b{B}^{*(t)}\] (and \[\b{C}^{*(t)}\]), we independently pick the non-zero locations with probability \[\alpha\] (and \[\beta\]), and draw the values on the support from the Rademacher distribution\footnote{The corresponding code is available at \url{https://github.com/srambhatla/TensorNOODL}.}; see Appendix~\ref{app:syn_res} for details.
	{\footnotesize
		\begin{figure}[!t]
			\centering
			\resizebox{0.85\textwidth}{!}{
				\begin{tabular}{cccc}
					&\hspace*{25pt}{ \large $(\alpha, \beta) = 0.005$} 
					&\hspace{25pt}{\large $(\alpha, \beta) = 0.01$} 
					&\hspace{25pt}{\large $(\alpha, \beta) = 0.05$} \\
					{\rotatebox{90}{\parbox{4cm}{\hspace{9pt}\large\centering\textbf{Recovery of \[\b{A}^*\] across techniques}}}}\hspace{4pt}&
					\includegraphics[width=0.3\textwidth]{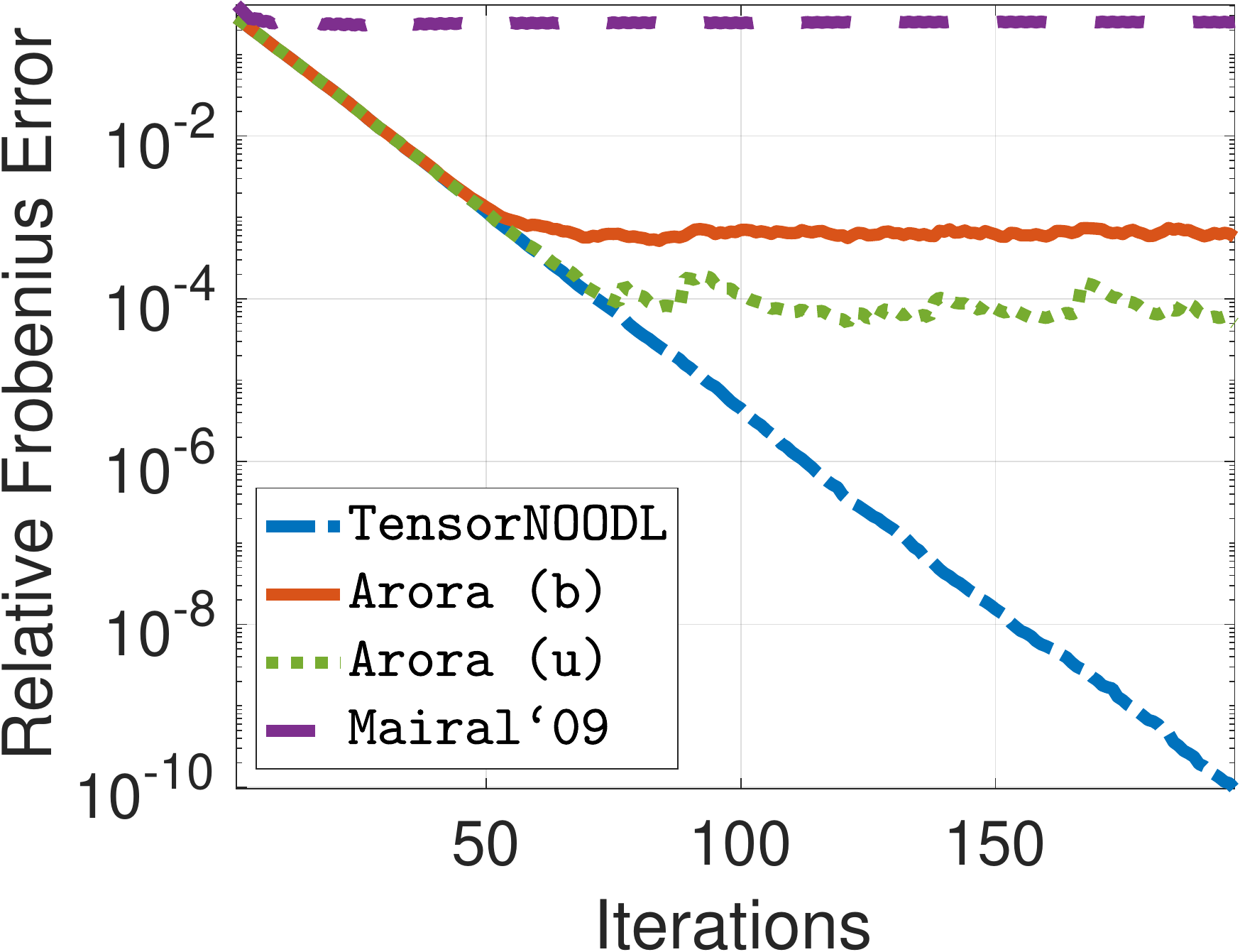}&\hspace*{-6pt}
					\includegraphics[width=0.3\textwidth]{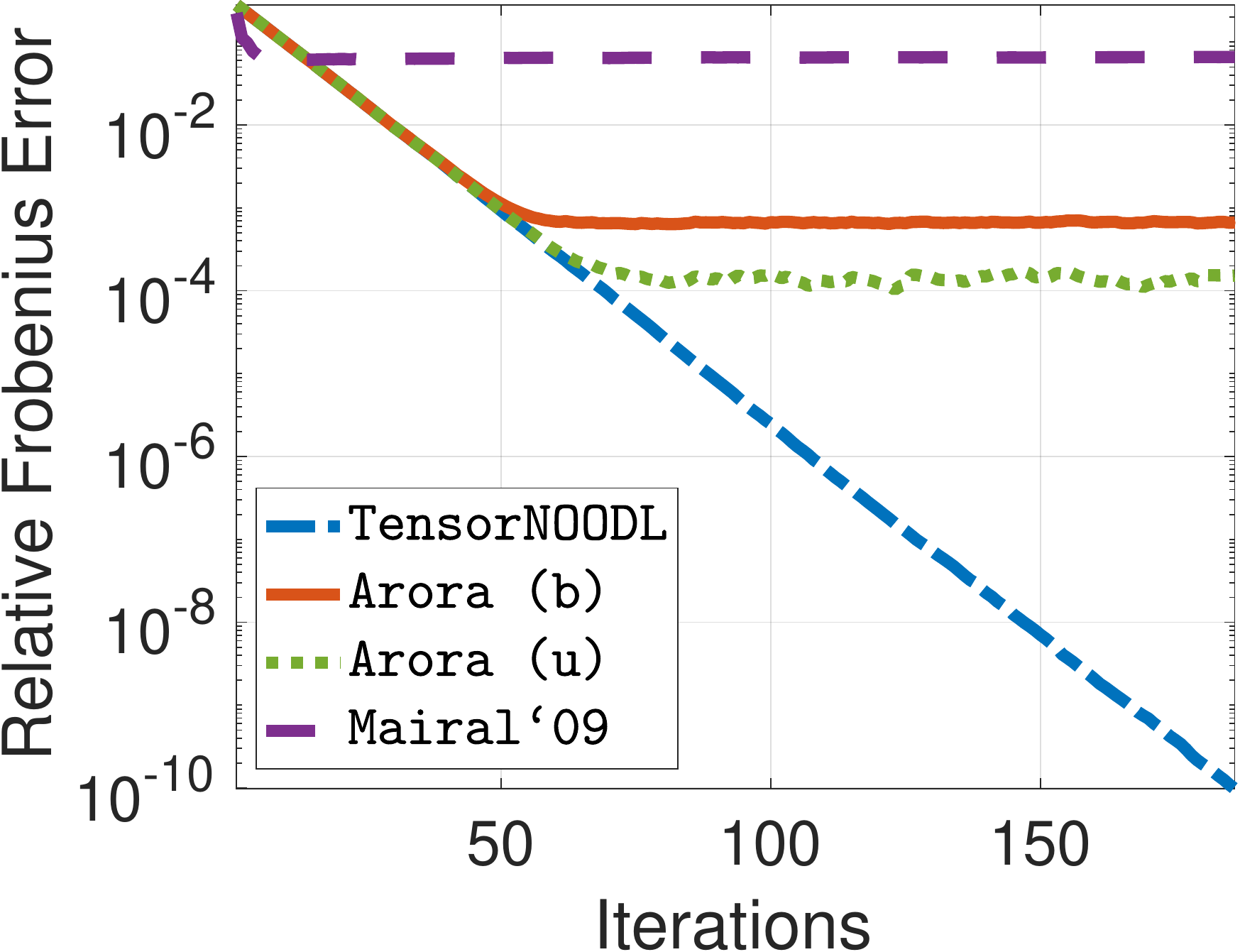}&\hspace*{-6pt}
					\includegraphics[width=0.3\textwidth]{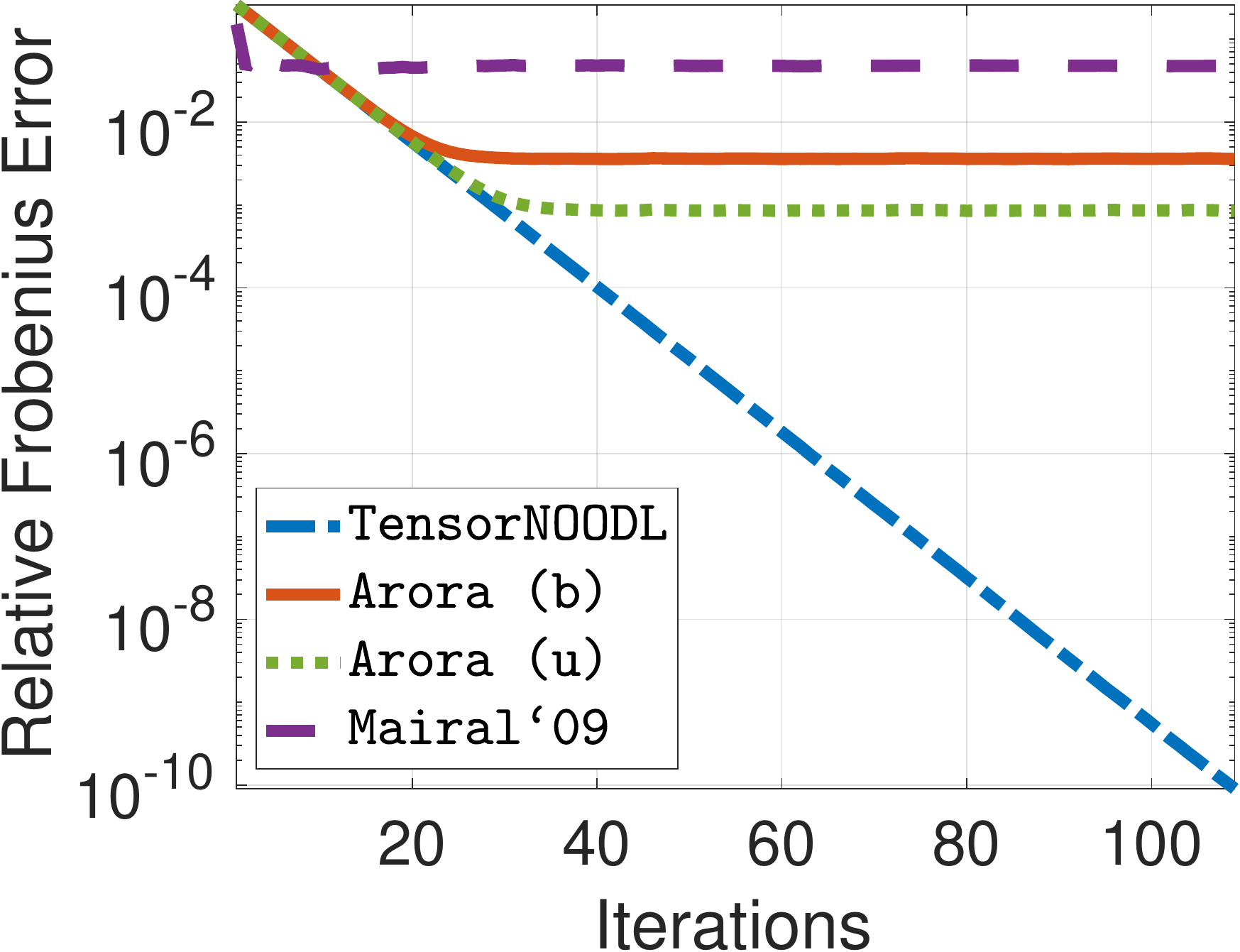}\\
					\vspace{0.05in}
					&{\hspace{25pt}(a)} 
					&{\hspace{25pt}(b)} 
					&{\hspace{25pt}(c)} \vspace{-6pt}\\
					{\rotatebox{90}{\parbox{4cm}{\hspace{0pt}\large\centering\textbf{Recovery of \[\b{A}^*\] and \[\b{X}^{*(t)}\] by \texttt{TensorNOODL} }}}}\hspace{4pt}&
					\includegraphics[width=0.3\textwidth]{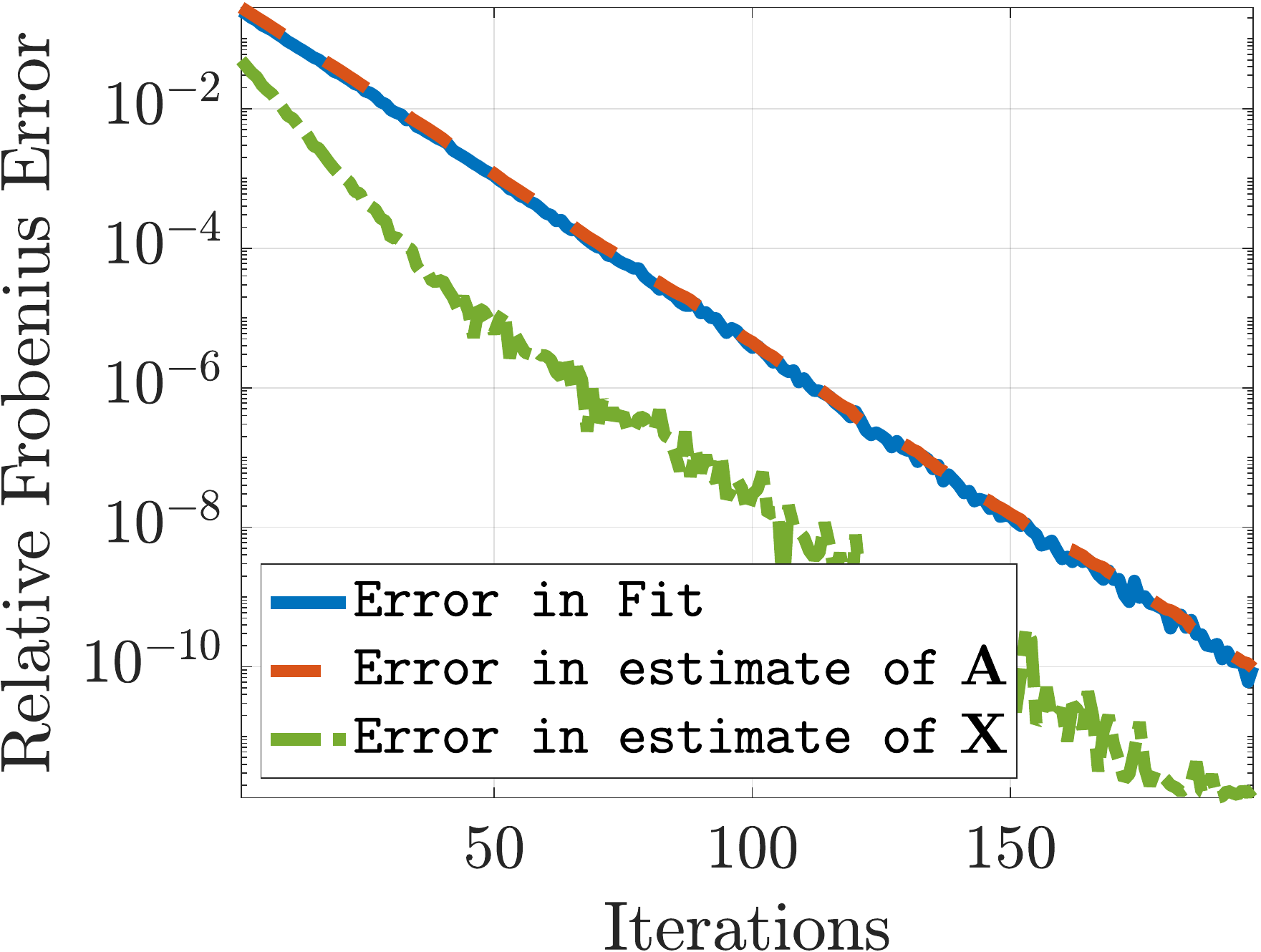}&\hspace*{-4pt}
					\includegraphics[width=0.3\textwidth]{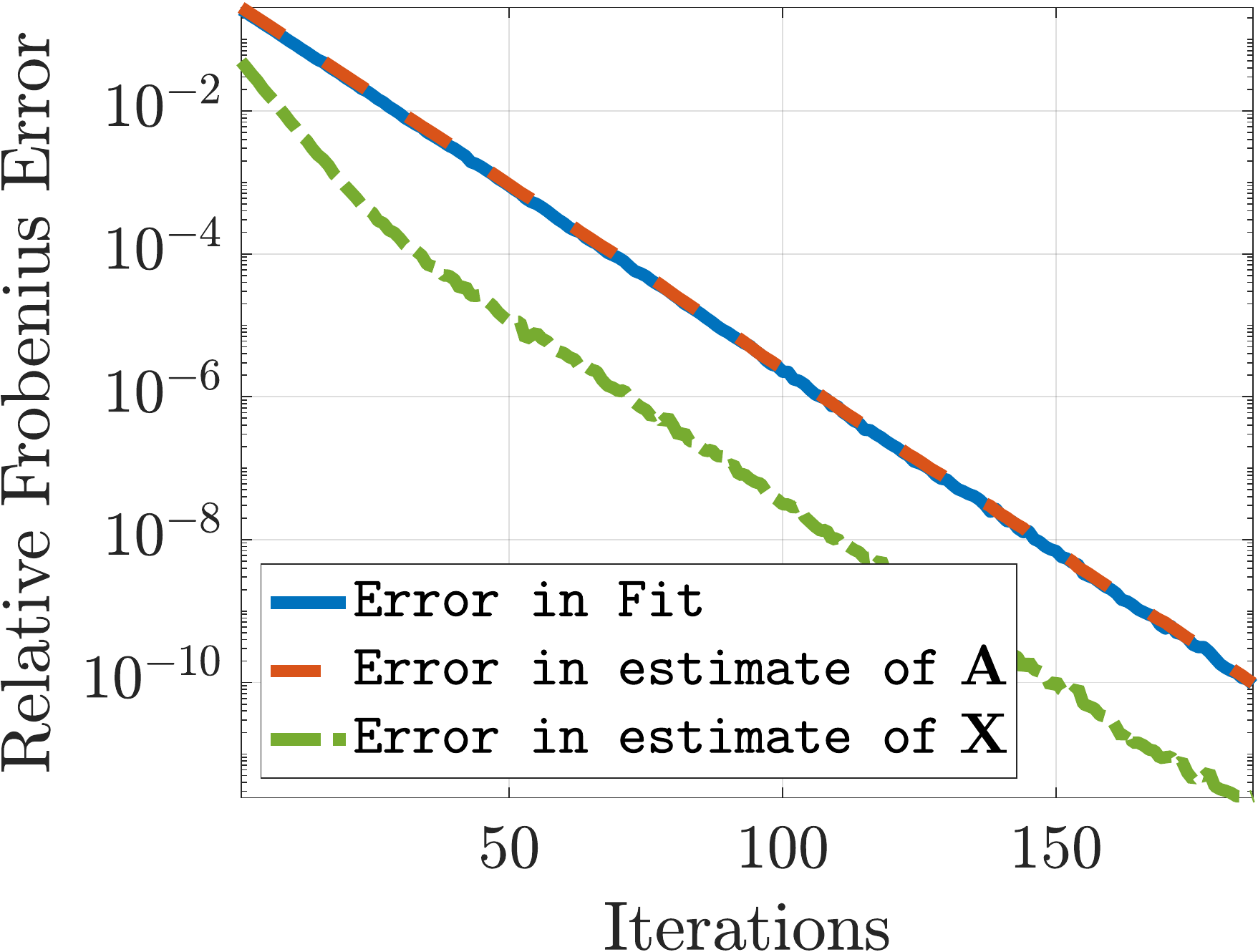}&\hspace*{-6pt}
					\includegraphics[width=0.3\textwidth]{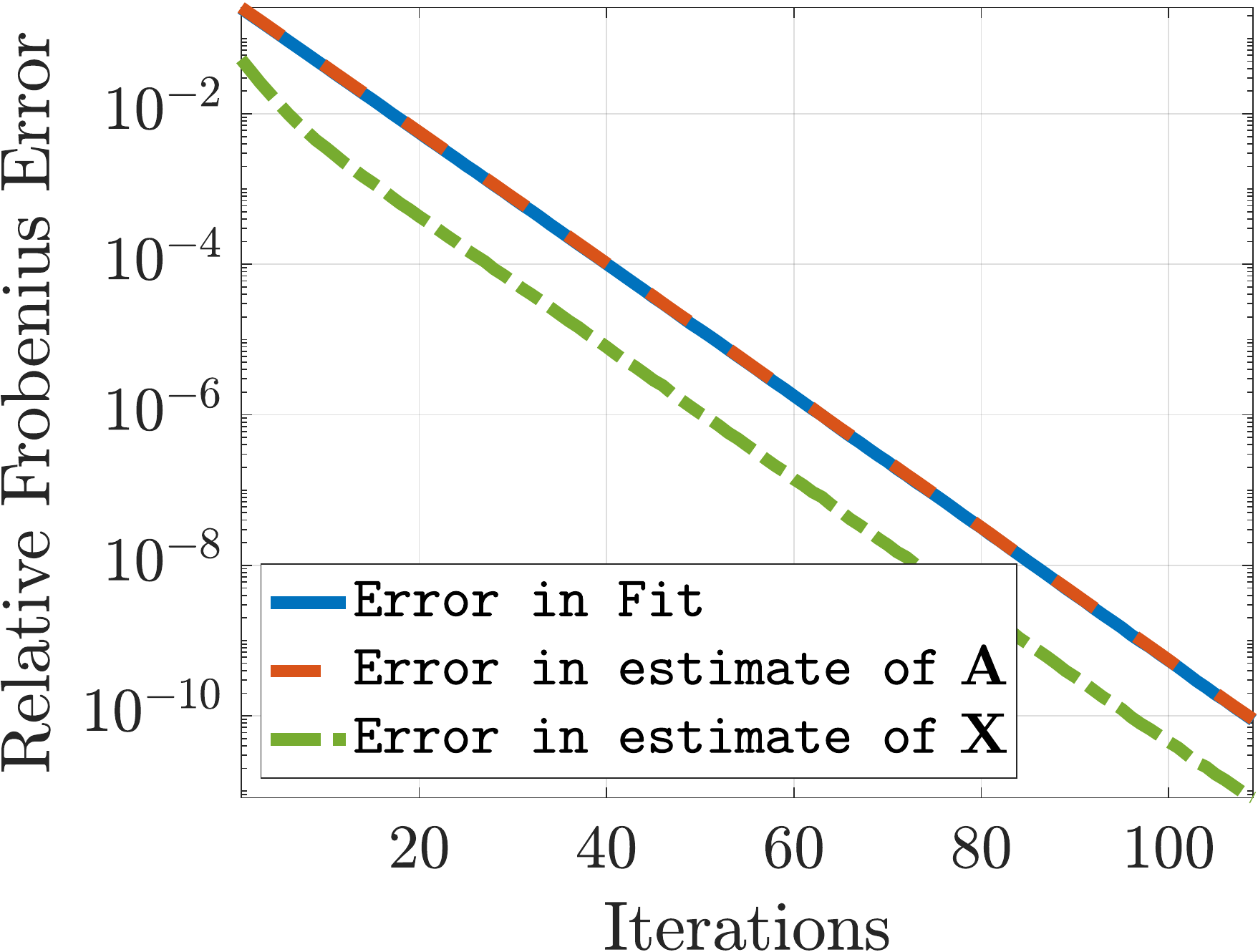}\\
					\vspace{0.05in}
					&{\hspace{25pt} (d)} 
					&{\hspace{25pt}(e)} 
					&{\hspace{25pt}(f)} \vspace{0pt}
				%	\vspace{-6pt}\\
			\end{tabular}}
			\caption{Linear convergence of \texttt{TensorNOODL}. Panels (a), (b), and (c) show the convergence properties of \texttt{TensorNOODL}, \texttt{Arora (b)},~\texttt{Arora (u)} and \texttt{Mairal`09} for the incoherent factor \[\b{A}\] recovery for \[(\alpha, \beta) = 0.005, 0.01\] and \[0.05\] respectively for \[m=450\], \[(J, K) = 500\] and seed\[=26\]. Panels (cd, (e) and (f), show the recovery of \[\b{X}^{*(t)}\] (i.e. \[\b{B}^{*(t)}\] and \[\b{C}^{*(t)}\]) \[\b{A}^{*}\], and the data fit (i.e., \[\|\b{Y}^{(t)} - \b{A}^{(t)}\hat{\b{X}}^{(t)}\|_{\rm F}/\|\b{Y}^{(t)}\|_{\rm F}\]) for \texttt{TensorNOODL} corresponding to (a), (b), and (c), respectively.  }
			\label{fig:lin_T}
		%	\vspace{-16pt}
	\end{figure}}

\vspace{3pt}
\noindent\textbf{Discussion}: We focus on the recovery of \[\b{X}^{*(t)}\] (including support recovery) since the performance of Alg.~\ref{algo:KRP_algo} solely depends on exact recovery of \[\b{X}^{*(t)}\]. In Fig.~\ref{fig:num_T}, we analyze the samples requirement across different choices of the dimension \[(J,K)\], rank \[(m)\] and sparsity parameters \[(\alpha, \beta)\] averaged across Monte Carlo runs using the total iterations \[T\]$^\text{\ref{foot:sam}}$. In line with theory, we observe a) in each panel the total iterations (to achieve tolerance \[\epsilon_T\]) decreases with increasing \[(\alpha, \beta)\], and b) for a fixed rank and sparsity parameters the \[T\] decreases with increasing \[(J,K)\], these are both due to the increase in available data samples; also sample requirement increases with rank \[m\].  Furthermore, only \texttt{TensorNOODL} recovers the correct support of \[\b{X}^{*(t)}\], crucial for sparse factor recovery. Corroborating our theoretical results, \texttt{TensorNOODL} achieves orders of magnitude superior recovery at linear rate (Fig.~\ref{fig:lin_T}) as compared to competing techniques both for the recovery of \[\b{A}^*\], and \[\b{X}^{*(t)}\]. Moreover,  since \[\b{X}^*\] columns can be estimated independently, \texttt{TensorNOODL} is scalable and can be implemented in highly distributed settings.
\vspace{0pt}

\subsection{Real-world data evaluation}
\vspace{0pt}
We consider a real data application in sports analytics. Additional real-data experiments for an email activity-based organizational behavior application are presented in Appendix~\ref{app:enron}.

 \begin{figure}[t]
    		\begin{center}
    		\begin{tabular}{cc}
    		
    		\includegraphics[width=0.6\textwidth]{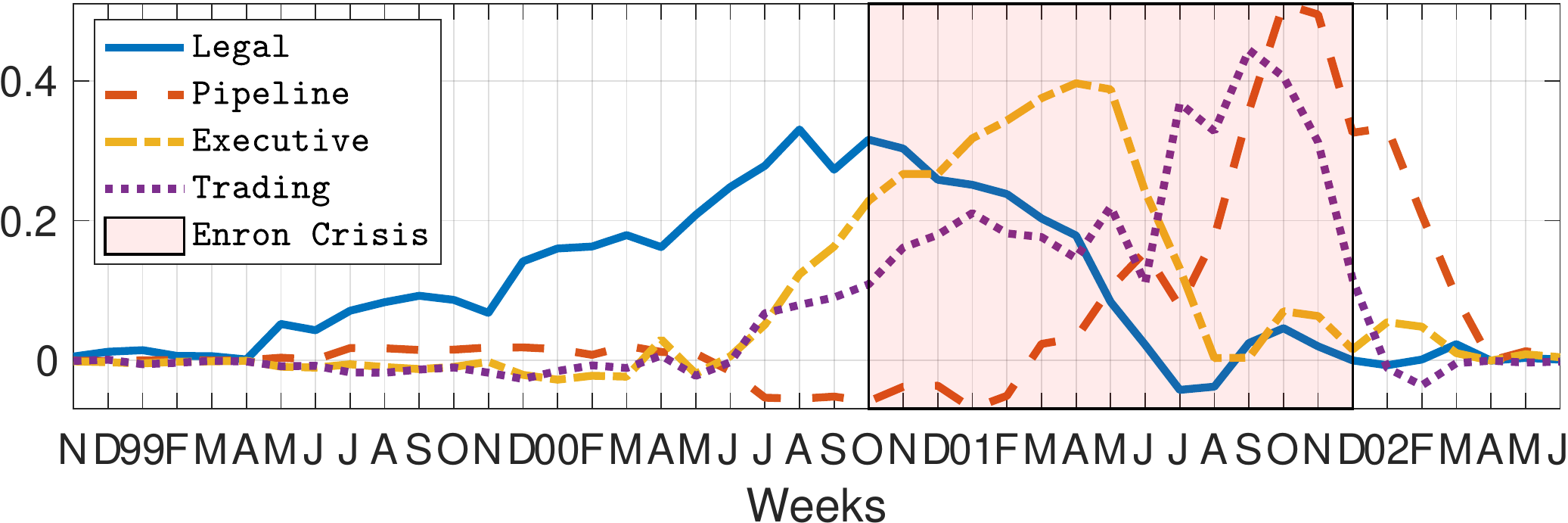}
    			&
    				\resizebox{0.22\textwidth}{!}{ 
    						\centering
    						\begin{tikzpicture} 
    						\node[anchor=south west,inner sep=0] (image) at (0.2,0) {\includegraphics[width=0.18\textwidth]{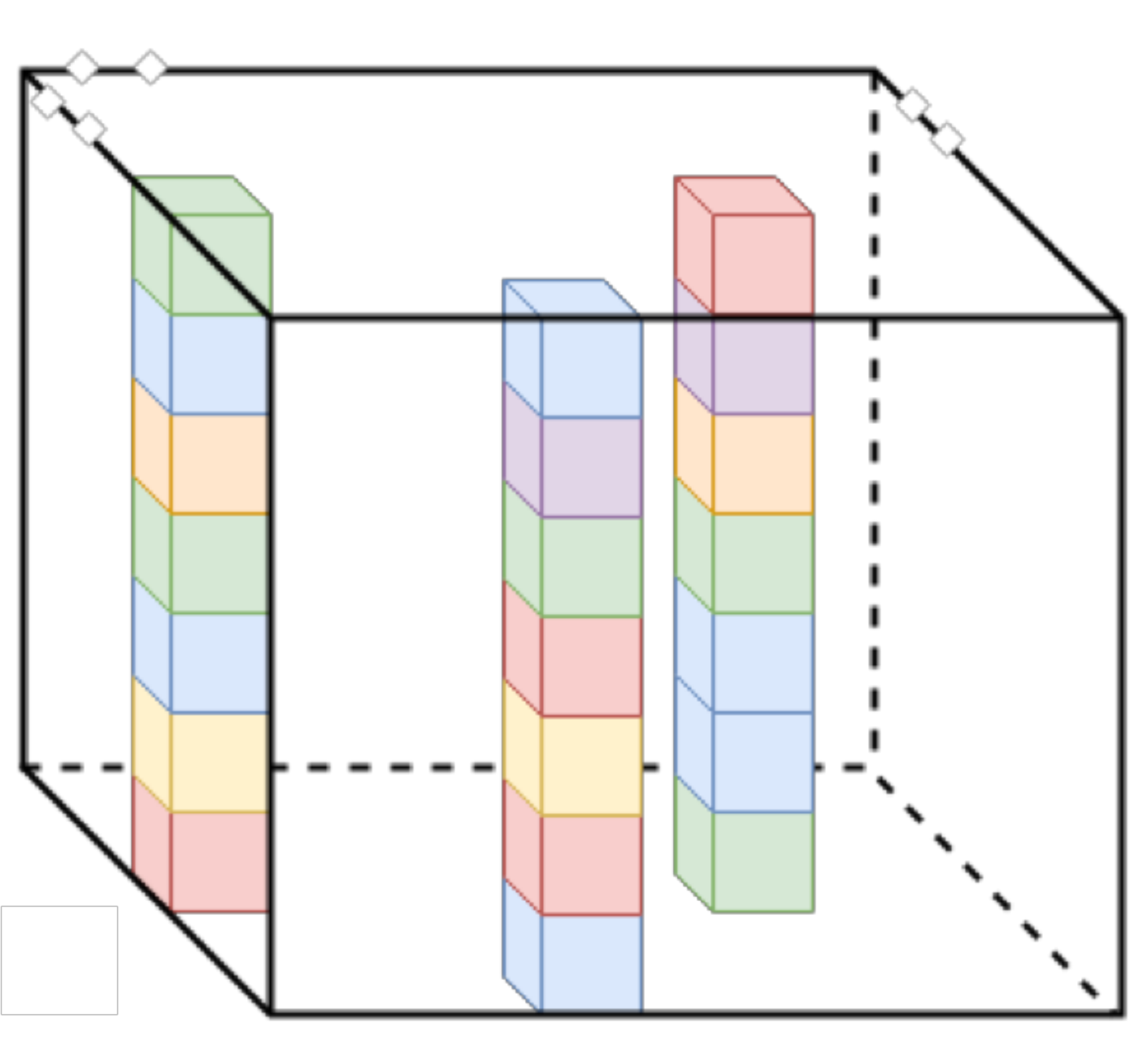}};
    						\node[align=center]  at (-0.25,1.4) {\rotatebox[]{90}{\scriptsize$n = 44$}};
    						\node[align=center]  at (0,1.4) {\rotatebox[]{90}{\scriptsize Weeks}};
    						\node[align=center]  at (0.2,0.2) {\scriptsize\rotatebox[]{-45}{$K = 184$}};
    						\node[align=center]  at (-0.0,-0.0) {\rotatebox[]{-45}{\scriptsize Employees}};
    						\node[align=center]  at (1.7,-0.2) {\scriptsize$J= 184$};
    						\node[align=center]  at (1.7,-0.5) {\scriptsize Employees};
    						\node[align=center] at (1.8,1) {\large$\underline{\b{Z}}$};
    						\end{tikzpicture}}\vspace{-5pt}\\
    						
    						{\small(a)}&{\small(b)}
    						\end{tabular}
    			\resizebox{0.34\linewidth}{!}{\hspace{-120pt}\begin{tabular}{c|c|c|c|c}
    					\multicolumn{5}{c}{(c) Cluster Quality: False Positives/ Cluster Size }\\
    					\hline
    					\textbf{Method}&\textbf{Legal} & \textbf{Pipeline} & \textbf{Executive} & \textbf{Trading} \\ \hline
    					\texttt{TensorNOODL}& 2/13 & 4/11& 1/14 & 10/24 \\
    					\hline
    					\texttt{Mairal `09}& 1/10 & Not Found& 8/17 & 3/7 \\ \hline
    					\cite{Fu2015} & 4/16 & 3/15& 3/30$^\dagger$ & 5/12\\ \hline
    				\end{tabular}
    	
    		}\end{center}
    		
    		\vspace*{-5pt}
    		\setstretch{0}
    		
    		\vspace{3pt}
    		\caption{Enron Email Analysis. The plot and the table show the recovered group email activity patterns over time, and the cluster quality analysis, respectively. Note the increased legal team activity before the crisis broke out internally (Oct. `00), to public (Oct '01), till lay-offs. $^\dagger$The authors set the number of cluster to $5$, here we combine the two clusters corresponding to ``Executive''.}
    		\label{fig:enron}
    	%	\vspace{-4pt}
    	\end{figure}

\subsubsection{Enron Email Dataset}
\vspace{-3pt}
Sparsity-regularized ALS-based tensor factorization techniques, albeit possessing limited convergence guarantees, have been a popular choice to analyze the Enron Email Dataset ($184\times184\times44$) \cite{Fu2015, Bader2006}. We now use \texttt{TensorNOODL} to analyze the email activity of $184$ Enron employees over $44$ weeks (Nov. `98 --Jan. '02) during the period before and after the financial irregularities were uncovered.

%We employ \texttt{TensorNOODL} to analyze the Enron Email Dataset ($184\times184\times44$) motivated from \cite{Fu2015, Bader2006} to showcase the performance of \texttt{TensorNOODL} on real dataset for a batch setting. The Enron dataset corresponds to email activity of $184$ Enron employees over $44$ weeks (Nov. `98 --Jan. '02) during the period before and after the financial irregularities were uncovered.

 \vspace{5pt}
\noindent\textbf{Methodology:} For \texttt{TensorNOODL} and \texttt{Mairal `09}, we use the initialization algorithm of \cite{Arora15}, which yielded $4$ dictionary elements. Following this, we use these techniques in batch setting to simultaneously identify email activity patterns and cluster employees. We also compare our results to \citet{Fu2015}, which just aims to cluster the employees by imposing sparsity constraint on one of the factors, and does not learn the patterns. As opposed to \cite{Fu2015}, \texttt{TensorNOODL}  did not require us to guess the number of dictionary elements to be used. We use Alg.~\ref{algo:KRP_algo} to identify the employees corresponding to email activity patterns from the recovered sparse factors. 

 \vspace{5pt}
\noindent\textbf{Discussion} -- Fig.~\ref{fig:enron} shows the $4$ main groups of employees recovered, and their activity over time.  In line with \citet{Diesner2005}, we observe that during the crisis the employees of different divisions indeed exhibited cliquish behavior. Furthermore, \texttt{TensorNOODL} is also superior in terms of cluster purity as inferred from the False Positives to Cluster-size ratio (Fig.~\ref{fig:enron}); see Appendix~\ref{app:enron} for details.

%\vspace{-15pt}
\subsubsection{NBA Shot Pattern Dataset}\label{sec:simnba}
\vspace{1pt}
 We analyze weekly shot patterns of the $100$ high scoring players ($80^{\rm th}$ percentile) against $30$ teams in the $2018-19$ regular season ($27$ weeks) of the National Basketball Association (NBA) league.  The task is to identify specific shot patterns attempted by players against teams and cluster them from the weekly $100\times30\times120$ shot pattern tensor.

 \vspace{5pt}
 \noindent\textbf{Methodology}: We divide half-court into $10\times12$ blocks and sum-up all shots attempted by a player in a game from a particular block, and vectorize to form a shot pattern vector ($\mathbb{R}^{120}$) of a player against a particular opponent team. We use  $2017-18$'s  regular season data to initialize incoherent factor using \cite{Arora15}, recovering $7$ elements. 
  \vspace{5pt}
  
\noindent\textbf{Discussion}: In  Fig.~\ref{fig:nba} we show $3$ recovered shot patterns and corresponding weights (week-$10$). \texttt{TensorNOODL} reveals the similarity in shot selection of James Harden and Devin Booker, in line with the sports reports at the time \citep{Rafferty18,Uggetti18}. The shared elements show their shot preference above the $3$-point line (Fig.~\ref{fig:nba}(a-b)) and at the rim (Fig.~\ref{fig:nba}(c)); See Appendix~\ref{app:nba} for detailed results, and Appendix~\ref{app:enron} for evaluations on Enron data.

\vspace{-6pt}
\section{Discussion}
\label{sec:conclusions}
\vspace{-5pt}
\noindent\textbf{Summary}: Leveraging a matrix view of the tensor factorization task, we propose \texttt{TensorNOODL}, to the best of our knowledge, the first provable algorithm to achieve exact (up to scaling and permutations) \emph{online} structured $3$-way tensor factorization at a linear rate. Our analysis to untangle the Kronecker product dependence structure (induced by the matricized view) can be leveraged by other tensor factorization tasks. %Since \[\b{X}^*\] columns can be estimated independently, \texttt{TensorNOODL} can be implemented in highly distributed settings. 

\begin{figure}[t]
\begin{center}
 \resizebox{0.9\textwidth}{!}{\begin{tabular}{ccccc}
 \includegraphics[width=0.25\textwidth]{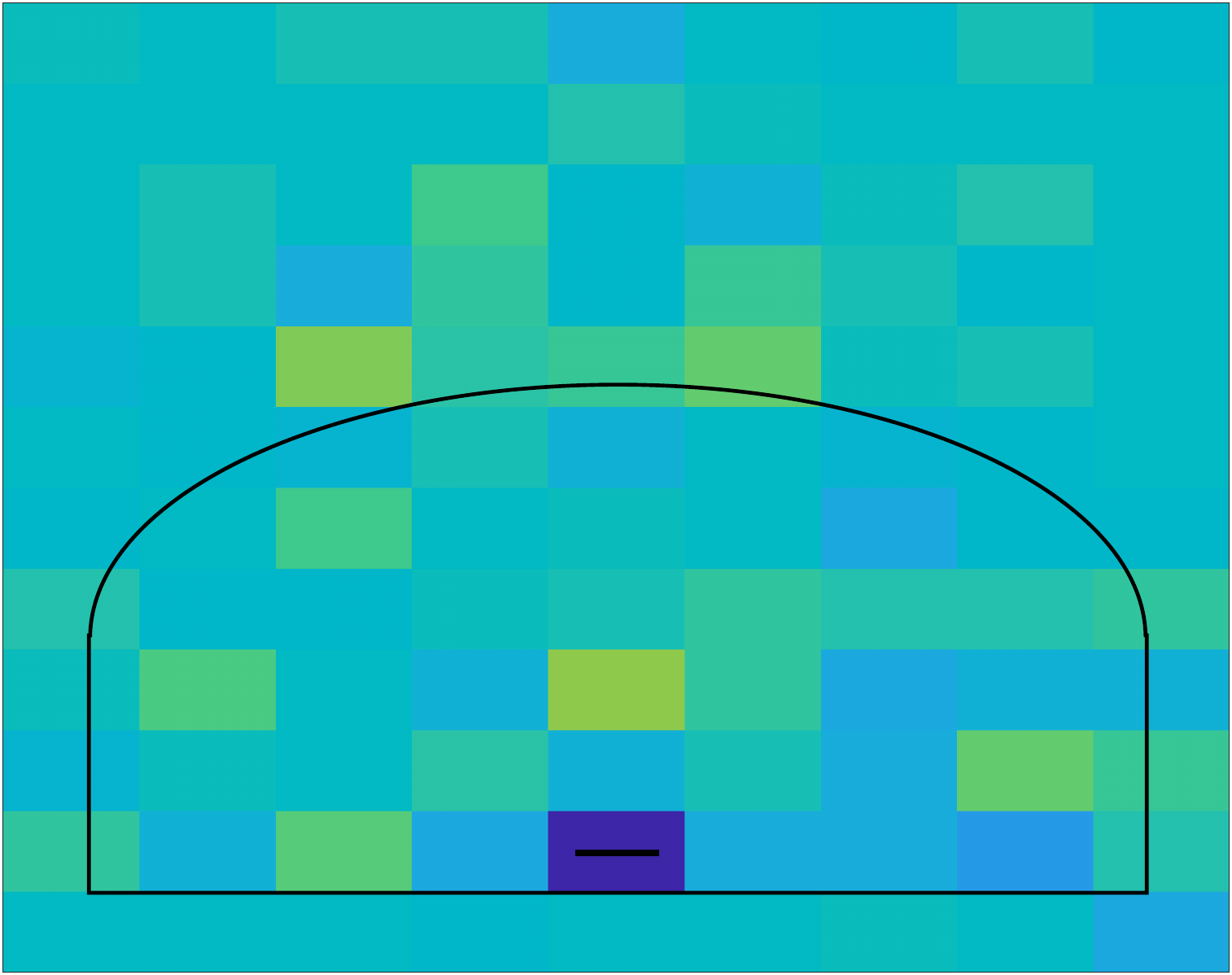}&
  \includegraphics[width=0.25\textwidth]{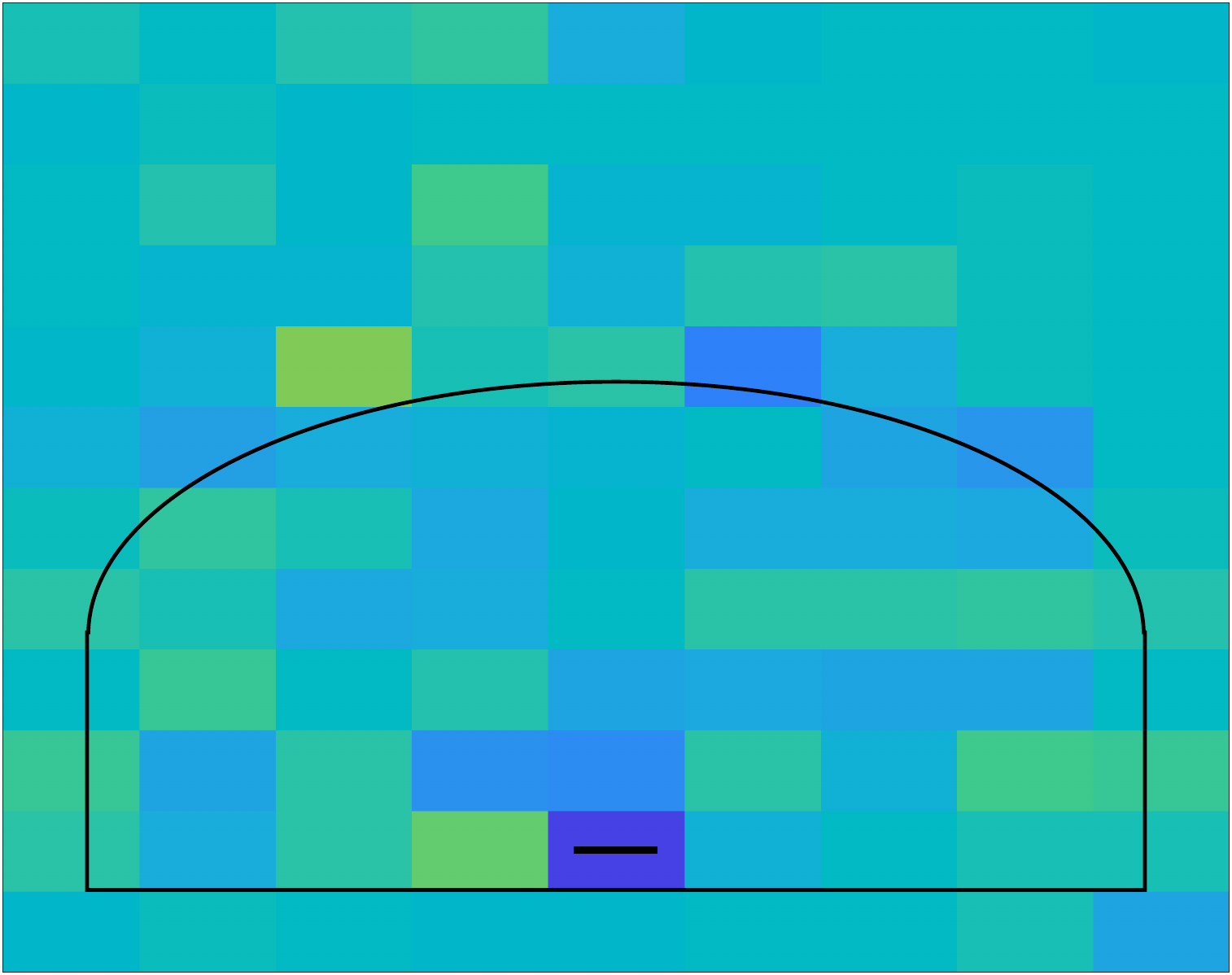}&
   \includegraphics[width=0.25\textwidth]{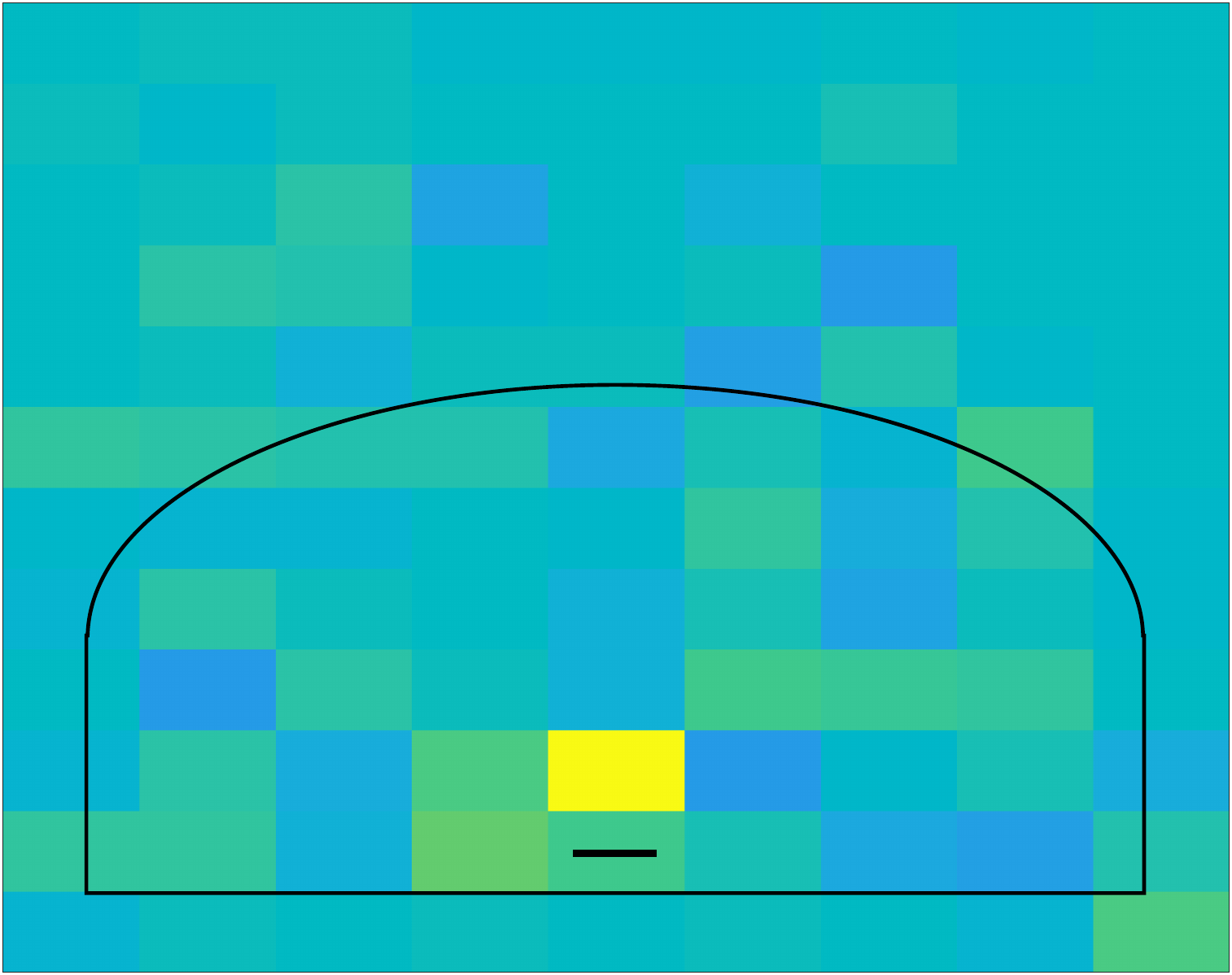}&
     \includegraphics[width=0.035\textwidth]{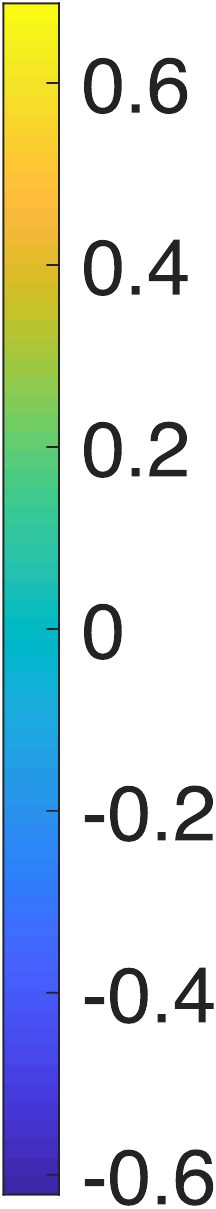}&
     \resizebox{0.23\textwidth}{!}{ 
                   			\begin{tikzpicture} 
                   			\node[anchor=south west,inner sep=0] (image) at (0.2,0) {\includegraphics[width=0.17\textwidth]{Tensor_only.pdf}};
                   			\node[align=center]  at (-0.25,1.4) {\rotatebox[]{90}{\scriptsize$n = 120$}};
                   			\node[align=center]  at (0,1.4) {\rotatebox[]{90}{\scriptsize shot patterns}};
                   			\node[align=center]  at (0.2,0.2) {\scriptsize\rotatebox[]{-45}{$K = 30$}};
                   			\node[align=center]  at (-0.0,-0.0) {\rotatebox[]{-45}{\scriptsize Teams}};
                   			\node[align=center]  at (1.7,-0.2) {\scriptsize$J= 100$};
                   			\node[align=center]  at (1.7,-0.5) {\scriptsize Players};
                   			\node[align=center] at (1.8,1) {\large$\underline{\b{Z}}^{(t)}$};
                   			\end{tikzpicture}}\\
     {\large(a) Element $4$}& {\large(b) Element $5$} & {\large(c) Element $6$}&& (d)\\
 \vspace{-30pt}
   \end{tabular}
 }\end{center}

 \vspace{-5pt}

 \begin{center} \hspace{-100pt}
  \resizebox{0.55\linewidth}{!}{\begin{tabular}{c|c|c|c}
  \multicolumn{4}{c}{\small Corresponding Sparse factor (Players) Coefficients}\\
  \hline
  \textbf{Player} & \textbf{Element $\b{4}$}& \textbf{Element $\b{5}$} & \textbf{Element $\b{6}$} \\ \hline
  James Harden&0.1992 & 0.0678& 0.2834 \\
   \hline
   Devin Booker&0.0114 & 0.0104& 0.4668 \\ \hline
    \end{tabular}
  }\end{center}
   \vspace{-5pt}
  \caption{NBA Regular Season Shot Pattern data analysis. \texttt{TensorNOODL} clusters the players and the teams. We show the three recovered dictionary factor elements (panels (a)-(c)) shared by James Harden and Devin Booker (believed to have similar styles) during week $10$ of the regular season ($2018-19$). Panel (d) shows the weekly shot pattern tensor, the input for \texttt{TensorNOODL} at each iteration \[t\].}
 \label{fig:nba}
 % \vspace{-16pt}
 \end{figure}

  \vspace{5pt}
\noindent\textbf{Limitations and Future Work}: We use probabilistic model assumptions which requires us to carefully identify independent samples. Although not an issue in practice, this leads to somewhat conservative results.  Future work includes improving this sample efficiency. 

  \vspace{5pt}
\noindent\textbf{Conclusions}: We analyze an exciting modality where the tensor decomposition task can be reduced to that of matrix factorization. Such correspondences offer a way to establish strong convergence and recovery guarantees for structured tensor factorization tasks.

\section*{Acknowledgement}
The authors graciously acknowledge the support from the DARPA YFA, Grant N66001-14-1-4047. The authors would also like to thank Prof. Nikos Sidiropoulos and Di Xiao for helpful discussions. 

\bibliographystyle{ims}
\bibliography{referDL_jrnl}

\appendix
\section*{Navigating Supplementary Material}
We summarize the notation used in our work in Appendix~\ref{app:summary_notation}, including with a list of frequently used symbols and their corresponding definitions. Next, in Appendix~\ref{sec:theory}, we present the proof of our main result, and organize the the proofs of intermediate results in Appendix~\ref{app:tensor}; additional results used are listed in Appendix~\ref{app:add} for completeness. Furthermore, we show the detailed synthetic and real-world experimental results, along with how to reproduce them, in Appendix~\ref{app:exp}. Corresponding code with specific recommendation on the parameter setting is available at \url{https://github.com/srambhatla/TensorNOODL}. 
\vspace{-5pt}
\appendix

\section{Summary of Notation}\label{app:summary_notation}
\vspace{-5pt}
In addition to the notation described in the manuscript, we use \[\|\b{M}\|\] and \[\|\b{M}\|_{\Fr}\] for the spectral and Frobenius norm, respectively, and \[\|\b{v}\|\], \[\|\b{v}\|_0\], and \[\|\b{v}\|_1\] to denote the \[\ell_2\], \[\ell_0\] (number of non-zero entries), and \[\ell_1\] norm, respectively.  In addition, we use \[\b{D}_{(\vb)}\] as a diagonal matrix with elements of a vector \[\b{v}\] on the diagonal. Given a matrix $\Mb$, we use \[\b{M}_{-i}\] to denote a resulting matrix without \[i\]-th column. Also note that, since we show that \[\|\b{A}^{(t)}_i - \b{A}^*_i\| \leq \epsilon_t\] contracts in every step, therefore we fix  \[\epsilon_t, \epsilon_0 = \mathcal{O}^*(1/\log(n))\] in our analysis. We summarize the definitions of some frequently used symbols in our analysis in Table~\ref{tab:symbols_1} and \ref{tab:symbols}.

\begin{table}[!t]
  \centering
  	\caption{Frequently used symbols: Probabilities}
  	\vspace{-10pt}
  	\label{tab:symbols_1}
  		 \begin{minipage}{\textwidth}
  		 \centering
  		  \resizebox{\columnwidth}{!}{
  		  \begin{tabular}{c|p{7.5cm}|c|p{7.5cm}}
  		  \multicolumn{4}{l}{}\\ 
 		\multicolumn{4}{l}{\textbf{Probabilities}}\\ \hline
  		 \textbf{Symbol} &\textbf{Definition}&\textbf{Symbol} &\textbf{Definition} \\ \hline
  		  \[\gamma\] & \[\gamma := \alpha\beta\], where \[\alpha(\beta)\] is the probability that an element \[\b{B}_{ij}^{*(t)}\] ( \[\b{C}_{ij}^{*(t)}\]) of \[\b{B}^{*(t)}\] (\[\b{C}^{*(t)}\]) is non-zero. &\[ \delta_{\b{B}_i}^{(t)}\] &  \[\delta_{\b{B}_i}^{(t)} = \exp ({-\frac{\epsilon^2 J\alpha}{2(1 + \epsilon/3)}}) \] for any \[\epsilon>0\].\\ \hline
  		 \[\delta_{\HT}^{(t)}\] &\[\delta_{\HT}^{(t)} = 2m~{\exp}({-\tfrac{C^2}{\mathcal{O}^*(\epsilon_t^2)}})\]&  \[\delta_{\beta}^{(t)}\]& \[2s~{\exp}(-\tfrac{1}{\mathcal{O}(\epsilon_t)})\]\\ \hline
  		 \[\delta_{\rm s}^{(t)} \] & \[\delta_s^{(t)} = \min(J,K)\exp( {-{\epsilon^2 \alpha \beta m}/{2(1 + \epsilon/3)}})\] for any \[\epsilon > 0\]&  \[\delta_{p}^{(t)} \]&\[\delta_{p}^{(t)} = \exp ({-\frac{\epsilon^2}{2}L(1 - (1 - \gamma)^m)})\]\\ \hline
  		 \[\delta_{\rm IHT}^{(t)} \]&  \[\delta_{\rm IHT}^{(t)}   = \delta_{\HT}^{(t)}\] + \[\delta_{\beta}^{(t)}\]& \[\delta_{\text{NOODL}}^{(t)} \] & \[\delta_{\text{NOODL}}^{(t)} =  \delta_{\HT}^{(t)}  + \delta_{\beta}^{(t)} + \delta_{\rm HW} +\delta_{\gradvec}^{(t)} + \delta_{\gradmat}^{(t)}\] \\ \hline
  		\[q_i \] & \[q_i = \b{Pr}[i \in S] = \Theta(\tfrac{s}{m})\] &
  		 \[q_{i,j} \] & \[q_{i,j} = \b{Pr}[i,j \in S] = \Theta(\tfrac{s^2}{m^2})\] \\ \hline
  		 \[p_i\] & \[p_i = \b{E}[\b{X}_{ij}^*\sgn(\b{X}^*_{ij})| \b{X}_{ij}^* \neq 0]\] &
  		 \[	\delta_{\rm HW}^{(t)} \] &  \[\delta_{\rm HW}^{(t)} = \exp(-{1}/{\mathcal{O}(\epsilon_t)})\]\\ \hline
 		 \[	\delta_{\gradvec}^{(t)}\] & \[\delta_{\gradvec}^{(t)} =  \exp(-\Omega(s))\] &
  		 \[\delta_{\gradmat}^{(t)} \] & \[\delta_{\gradmat}^{(t)} = (n+m)\exp(-\Omega(m\sqrt{\log(n)})\]\\ \hline
  		  	\end{tabular}	\vspace*{10pt}}
  		  	\end{minipage}
  \end{table}
  \begin{table}[!t]

  \centering
  	\caption{Frequently used symbols: Notation and Parameters}
  	\vspace{10pt}
  	\label{tab:symbols}
   \centering
  		\resizebox{0.8\textwidth}{!}{
  		\small
  		\begin{tabular}{P{1.2cm}|p{5cm}||P{1.2cm}|p{5cm}}
  		%\multicolumn{2}{|l|}{\textbf{Claims}}\\ \hline
  		\textbf{Symbol} &\textbf{Definition} &\textbf{Symbol} &\textbf{Definition}\\
  		\hline
  		\[(\cdot)^* \] & Used to represent the ground-truth matrices.&
 		\[(\cdot)^{(t)}\], \[\hat{(\cdot)}^{(t)}\], and \[\hat{(\cdot)}\] & Used to represent the estimates formed by the algorithm.\\ \hline	 
  		\[(\cdot)^{(t)}\]& The subscript ``\[t\]'' is used to represent the estimates at \[t\]-iteration of the online algorithm.&
 		\[\b{X}^{(r)(t)}\]& The \[r\]-th IHT iterate at \[t\]-th iterate of the online algorithm. \\ \hline	
 		\[(\cdot)^{(r)}\]&The subscript ``\[r\]'' is used to represent the \[r\]-th IHT iterate. &\[\hat{\b{X}}^{(t)}\]& The final IHT estimate at (\[r=R\]), i.e., \[\b{X}^{(R)(t)}\] at the \[t\]-th iterate of the online algorithm. \\\hline
  		\[\b{A}^{(t)}_i \] & \[i\]-th column of  \[{\b{A}^(t)}\] (estimate of \[\b{A}^*\] at the \[t\]-th iteration of the online algorithm). &
   		\[\hat{\b{B}}^{(t)}\] (\[\hat{\b{C}}^{(t)}\]) & Estimate of \[\b{B}^{*(t)}\] (\[\b{C}^{*(t)}\]) at the \[t\]-th iteration of the online algorithm.\\ \hline	
  	   \[\b{S}^{*(t)}\] &  Transposed Khatri-Rao structured (sparse) matrix,  \[\b{S}^{*(t)} = (\b{C}^{*(t)} \odot \b{B}^{*(t)})^\top\], its \[i\]-th row is given by \[\b{C}^{*(t)}_i \otimes \b{B}^{*(t)}_i\].&
      \[\b{X}^{*(t)}\] &  Sparse matrix formed by collecting non-zero columns of \[\b{S}^{*(t)}\].\\ \hline
    	\[p\] &  Number of columns in \[\b{X}^{*(t)}\], also the number of non-zero columns in  \[\b{S}^{*(t)}\]. & 
      \[\b{Z}_1^{(t)\top}\] &  Mode-$1$ unfolding of \[\underline{\b{Z}}^{(t)}\], \[\b{Z}_1^{(t)\top} = \b{A}^*(\b{C}^{*(t)} \odot \b{B}^{*(t)})^\top\] at the \[t\]-th iteration of the online algorithm.\\ \hline
    	\[\epsilon_t\] & Upper-bound on column-wise error at the \[t\]-th iterate,\[\|\b{A}^{(t)}_i - \b{A}^*_i\| \leq \epsilon_t = \mathcal{O}^*(\tfrac{1}{\log(n)})\].&
    		\[\mu\] & The incoherence between the columns of the factor \[\b{A}^*\]; see Def. \ref{def:mu}.\\ \hline
    	    \[\mu_t\] & Incoherence between the columns of  \[\b{A}^{(t)}\], \[\tfrac{\mu_t}{\sqrt{n}} = \tfrac{\mu}{\sqrt{n}} + 2\epsilon_t\].&
    	 \[\xi\] & The element-wise upper bound on the error between \[\hat{\b{S}}_{ij}^{(t)}\] and \[\b{S}_{ij}^{*(t)}\], i.e., \[|\b{S}_{ij}^{*(t)} - \hat{\b{S}}_{ij}^{(t)}| \leq \xi\]. \\ \hline
    	 \[s\] & The number of non-zeros in a column of \[\b{S}^{*(t)}\], also refered to as the \emph{sparsity}.&
    	 \[\alpha(\beta)\] & The probability that an element \[\b{B}_{ij}^{*(t)}\] ( \[\b{C}_{ij}^{*(t)}\]) of \[\b{B}^{*(t)}\] (\[\b{C}^{*(t)}\]) is non-zero.\\\hline
    	\[\epsilon_B\] & Upper-bound on column-wise \[\ell_2\]-error  in the estimate \[\hat{\b{B}}^{(t)}\]  at \[t\]-th iteration, i.e.,, \[\|\hat{\b{B}}_i^{(t)} - \b{B}^{*(t)}_i\| \leq \epsilon_B = \mathcal{O}(\tfrac{\xi^2}{\alpha\beta})\]. &
	 \[\epsilon_C\] & Upper-bound on column-wise \[\ell_2\]-error in the estimate \[\hat{\b{C}}^{(t)}\] at \[t\]-th iteration, i.e., \[\|\hat{\b{C}}_i^{(t)} - \b{C}^{*(t)}_i\| \leq \epsilon_C = \mathcal{O}(\tfrac{\xi^2}{\alpha\beta})\]. \\ \hline
 
 		\[R\] & The total number of IHT steps at the \[t\]-th iteration of the online algorithm.&
 	 		\[T\] & Total number of online iterations.\\ \hline
 	 	\[\delta_{R}\] & Decay parameter for final IHT step at every \[t\], \[ {\rm ceil} (\dfrac{\log(\tfrac{1}{\delta_R})}{\log(1 - \eta_{x})})\leq R \], where \[\eta_{x}\] is the step-size parameter for the IHT step. &
 	 	\[\delta_T\] &  Element-wise target error tolerance for final estimate (at \[t=T\]) of \[\b{X}^{*(T)}\], \[|\hat{\b{X}}_{ij}^{(T)} - \b{X}_{ij}^{*(T)}| \leq \delta_T \forall i\in \supp(\b{X}^{*(T)})\]. \\ \hline
 	 \[C\] & Lower-bound on \[\b{X}^*_{ij}\], \[|\b{X}^{*(t)}_{ij}|\geq C\] for \[(i,j) \in \supp(\b{X}^{*(t)})\] and \[C \leq 1\]&\[L\] & \[L:= \min(J,K)\] \\ \hline
  		 \end{tabular}}
  \end{table}

\clearpage

\section{Proof of Theorem~1}
\label{sec:theory}
In this section, we present the details of the analysis pertaining to our main result.

 \paragraph{Theorem~\ref{main_result} [Main Result]}
\emph{Suppose a tensor \[\underline{\b{Z}}^{(t)} \in \mathbb{R}^{n \times J \times K}\] provided to Algorithm~\ref{alg:main_alg_tens} at each iteration \[t\] admits a decomposition of the form \eqref{CPD}  with factors \[\b{A}^* \in \mathbb{R}^{n \times m}\], \[\b{B}^{*(t)} \in \mathbb{R}^{J \times m}\] and \[\b{C}^{*(t)}\in \mathbb{R}^{K \times m}\] and \[\min(J,K) = \Omega(ms^2)\]. Further, suppose that the assumptions \ref{assumption:mu}-\ref{assumption:step coeff}  hold. %then for \[\eta_A = \Theta(m/k)\] 
Then, given \[R = \Omega({\rm log}(n))\], with probability at least \[(1 - \delta_{\text{alg}})\] for some small constant \[\delta_{\text{alg}}\], the coefficient estimate \[\hat{\b{X}}^{(t)}\] at \[t\]-th iteration has the correct signed-support and satisfies
\begin{align*}
(\hat{\b{X}}_{i,j}^{(t)} - \b{X}_{i,j}^{*(t)})^2 \leq \zeta^2 %C_{i_1}^{(R)} 
&:= \mathcal{O}(s(1 - \omega)^{t/2}\|\b{A}_i^{(0)} - \b{A}_i^*\|), ~\text{for all}~(i,j) \in \supp({\b{X}^{*(t)}}).
\end{align*} 
Furthermore, for some \[0 < \omega < 1/2\], the estimate \[\b{A}^{(t)}\] at \[t\]-th iteration satisfies 
\begin{align*}
\|\b{A}_i^{(t)} - \b{A}_i^*\|^2 \leq (1 - \omega)^t\|\b{A}_i^{(0)} - \b{A}_i^*\|^2,~\text{for all}~t = 1,2,\ldots. 
\end{align*}
Consequently,  Algorithm~\ref{algo:KRP_algo} recovers the supports of the sparse factors \[\b{B}^{*(t)}\] and \[\b{C}^{*(t)}\] correctly, and \[\|\hat{\b{B}}_i^{(t)} - \b{B}_i^{*(t)}\|_2 \leq \epsilon_B\] and \[\|\hat{\b{C}}_i^{(t)}- \b{C}_i^{*(t)}\|_2 \leq \epsilon_C\], where \[\epsilon_B = \epsilon_C = \mathcal{O}(\tfrac{\zeta^2}{\alpha\beta})\].}

\textit{Here, \[\delta_{alg} = \delta_{\rm s} + \delta_{p}^{(t)} + \delta_{\b{B}_i}^{(t)} + \delta_{\rm NOODL} \]. Further, \[\delta_{\text{NOODL}}^{(t)} =  \delta_{\HT}^{(t)}  + \delta_{\beta}^{(t)} + \delta_{\rm HW} +\delta_{\gradvec}^{(t)} + \delta_{\gradmat}^{(t)}\], where \[\delta_{\HT}^{(t)} = 2m~{\exp}({-{C^2}/{\mathcal{O}^*(\epsilon_t^2)}})\], \[\delta_\beta^{(t)} = 2s~{\exp}(-{1}/{\mathcal{O}(\epsilon_t)})\], \[\delta_{\rm HW}^{(t)} = \exp(-{1}/{\mathcal{O}(\epsilon_t)})\], \[\delta_{\gradvec}^{(t)} =  \exp(-\Omega(s))\], \[\delta_{\gradmat}^{(t)} = (n+m)\exp(-\Omega(m\sqrt{\log(n)})\]. Furthermore, \[\delta_s^{(t)} = \min(J,K)\exp( {-{\epsilon^2 \alpha \beta m}/{2(1 + \epsilon/3)}})\] for any \[\epsilon > 0\], \[\delta_{p}^{(t)} = \exp ({-\frac{\epsilon^2}{2}L(1 - (1 - \gamma)^m)})\], and \[\delta_{\b{B}_i}^{(t)} = \exp ({-\frac{\epsilon^2 J\alpha}{2(1 + \epsilon/3)}}) \] for any \[\epsilon>0\]. Also, \[\|\b{A}_i^{(t)} -\b{A}_i^*\| \leq \epsilon_{t}\].}
\vspace{10pt}

\begin{proof}\textbf{of Theorem~\ref{main_result}} The proof procedure relies on analyzing three main steps of Alg.~\ref{alg:main_alg_tens} -- 1) estimating the \[\b{X}^{*(t)}\] reliably corresponding to \[\underline{\b{Z}}^{(t)}\], 2) using \[\b{X}^{(t)}\] to estimate the factors \[\b{B}^*\] and \[\b{C}^*\], and 3) making progress on the estimate of \[\b{A}^*\] at every iteration \[t\] of the online algorithm.

\paragraph{Estimating the \[\b{X}^{*(t)}\] reliably:} The sparse matrix \[\b{X}^{*(t)}\] is formed by collecting the non-zero columns of \[\b{S}^{*(t)} := (\b{C}^{*(t)}\odot \b{B^{*(t)}})^\top\] corresponding to \[\underline{\b{Z}}^{(t)}\]. The sparsity pattern of \[\b{X}^{*(t)}\] columns encodes the sparsity patterns of columns of \[\b{B}^{*(t)}\] and \[\b{C}^{*(t)}\]. As a result, recovering the support of \[\b{X}^{*(t)}\] \emph{exactly} is crucial to recover \[\b{B}^*\] and \[\b{C}^*\]. Furthermore,  recovering the signed-support is also essential for making progress on the dictionary factor. We begin by characterizing the number of non-zeros (\[s\]) in a column of \[\b{S}^{*(t)}\] (\[\b{X}^{*(t)}\]). The number of non-zeros in a column of \[\b{S}^{*(t)}\] are dependent on the non-zero elements of \[\b{B}^*\] and \[\b{C}^*\]. Since each element of \[\b{B}^*\] (\[\b{C}^*\]) is non-zero with probability \[\alpha (\beta)\], the upper-bound on the sparsity (\[s\]) of \[\b{S}^{*(t)}\] column is given by the following lemma. 

%In order to leverage the results of \cite{Rambhatla2019NOODL}, we need to get a handle on the sparsity (number of non-zeros in a column of \[\b{S}^{*(t)}\]), and characterize the number of usable (independent) data samples available to the algorithm. To this end, the following lemma characterizes the upper bound on the sparsity, \[k\], the number of non-zeros in a column of \[\b{S}^{*(t)}\]. 
\vspace{-2pt}
\begin{lemma}\label{our sparsity}
If \[m = \Omega(\log({\min(J,K)})/{\alpha \beta})\] then with probability at least \[(1-\delta_s^{(t)})\] the number of non-zeros \[s\], in a column of \[\b{S}^{*(t)}\] are upper-bounded as \[s = \c{O}(\alpha \beta m)\], where \[\delta_s^{(t)} = \min(J,K)\exp( {-{\epsilon^2 \alpha \beta m}/{2(1 + \epsilon/3)}})\] for any \[\epsilon > 0\].
\end{lemma}
\vspace{-2pt}
In line with our intuition, the sparsity scales with the parameters \[\alpha\], \[\beta\] and \[m\]. Next, we focus on the Iterative Hard Thresholding (IHT) phase of the algorithm;  Similar results were established in \cite[Lemma~1--4]{Rambhatla2019NOODL}. Here, the first step includes recovering the correct signed-support (Def.~\ref{def:signed-support}) of \[\b{X}^{*(t)}\] given an estimate \[\b{A}^{(0)}\], which is \[(\epsilon_0, 2)\]-near to \[\b{A}^*\] for \[\epsilon_0 = \mathcal{O}^*(1/\log(n))\]; see Def.~\ref{def:del_kappa}. To this end, we leverage the following lemma, to guarantee that the initialization step correctly recovers the signed-support with probability at least \[ (1-\delta_{\HT}^{(t)})\], for \[\delta_{\HT}^{(t)} = 2m~{\exp}({-\tfrac{C^2}{\mathcal{O}^*(\epsilon_t^2)}})\].
\vspace{-2pt}
\begin{lemma}\label{lem:recover_sign}\textnormal{\textbf{(Signed-support recovery)}}
	Suppose \[\b{A}^{(t)}\] is \[\epsilon_t\]-close to \[\b{A}^*\]. Then, if \[\mu = \c{O}(\log(n))\], \[s = \mathcal{O}^*({\sqrt{n}/\mu\log(n)})\], and \[\epsilon_t = \mathcal{O}^*(1/\sqrt{\log(m)})\], with probability at least \[(1 - \delta_{\HT}^{(t)})\] for each random sample \[\b{y} = \b{A}^*\b{x}^*\]:
	\begin{align*}
	\sgn(\HT_{C/2}((\b{A}^{(t)})^\top \b{y}) = \sgn(\b{x}^*),
	\end{align*}
	where \[\delta_{\HT}^{(t)} = 2m~{\exp}({-\tfrac{C^2}{\mathcal{O}^*(\epsilon_t^2)}})\].
\end{lemma} 
\vspace{-2pt}
Using Lemma~\ref{our sparsity} and \ref{lem:recover_sign} we also arrive at the condition that  \[s = \mathcal{O}(\alpha\beta m) = \mathcal{O}^*{\sqrt{n}/\mu\log(n)}\], formalized as \ref{assumption:k}. We now use the following result to ensure that each step of the IHT stage preserves the correct signed-support. Lemma~\ref{our:signed_supp}, states the conditions on the step size parameter \[\eta_x^{(r)}\], and the threshold \[\tau^{(r)}\], such that that the IHT-step preserves the correct signed-support with probability \[\delta_{\rm IHT}^{(t)} \], for \[\delta_{\rm IHT}^{(t)}  =  2m~{\exp}({-\tfrac{C^2}{\mathcal{O}^*(\epsilon_t^2)}}) + 2s~{\exp}(-\tfrac{1}{\mathcal{O}(\epsilon_t)})\].
\vspace{-2pt}
\begin{lemma}\textnormal{\textbf{(IHT update step preserves the correct signed-support}})\label{our:signed_supp}
	Suppose \[\b{A}^{(t)}\] is \[\epsilon_t\]-close to \[\b{A}^*\],  \[\mu = \c{O}(\log(n))\], \[s = \mathcal{O}^*(\sqrt{n}/\mu\log(n))\], and \[\epsilon_t = \mathcal{O}^*(1/\log(m))\] Then, with probability at least \[(1 - \delta_{\beta}^{(t)} - \delta_{\HT}^{(t)}  )\], each iterate of the IHT-based coefficient update step shown in \eqref{alg:coeff_iht} has the correct signed-support, if for a constant \[c^{(r)}_1(\epsilon_t, \mu, s, n) = \tilde{\Omega}({k^2}/{n})\], the step size is chosen as \[\eta_x^{(r)}\leq c^{(r)}_1\] ,
	and the threshold \[\tau^{(r)}\] is chosen as
	\begin{align*}
	\tau^{(r)} = \eta_x^{(r)}(t_\beta + \tfrac{\mu_t}{\sqrt{n}} \|\b{x}^{(r-1)} - \b{x}^*\|_1) :=c_2^{(r)}(\epsilon_t, \mu, s, n) = \tilde{\Omega}({s^2}/{n}),
	\end{align*}
	for some constants \[c_1\] and \[c_2\]. Here, \[t_\beta = \mathcal{O}(\sqrt{s\epsilon_t})\], \[\delta_{\HT}^{(t)} = 2m~{\exp}({-\tfrac{C^2}{\mathcal{O}^*(\epsilon_t^2)}})\] ,and \[\delta_{\beta}^{(t)} = 2s~{\exp}(-\tfrac{1}{\mathcal{O}(\epsilon_t)})\].
\end{lemma}
\vspace{-2pt}
Lemma~\ref{our:signed_supp} establishes condition on correct signed-support recovery by the IHT stage.  We now leverage the following result, Lemma~\ref{iht:x_R_error} to quantify the error incurred by \[\hat{\b{X}}^{(t)}\] at the end of the \[R\] IHT steps, i.e., \[|\b{X}_{ij}^{*(t)} - \hat{\b{X}}_{ij}^{(t)}| = |\b{S}_{ij}^{*(t)} - \hat{\b{S}}_{ij}^{(t)}| \leq \xi\].
\vspace{-2pt}
\begin{lemma}\textnormal{\textbf{(Upper-bound on the error in coefficient estimation)}}\label{iht:x_R_error}
 	With probability at least \[(1 - \delta_{\beta}^{(t)} - \delta_{\HT}^{(t)})\] the error incurred by each element \[(i_1, j_1) \in \supp(\b{X}^{*(t)})\] of the coefficient estimate is upper-bounded as
 	\begin{align*}
 	|\hat{\b{X}}_{i_1j_1}^{(t)} - \b{X}_{i_1j_1}^{*(t)}|
 	&\leq    \mathcal{O}(t_\beta) + \left({(R + 1)}s \eta_x\tfrac{\mu_t}{\sqrt{n}}~\underset{(i, j)}{\max}|\b{X}_{ij}^{(0)(t)} - \b{X}_{ij}^{*(t)}| + |\b{X}_{i_1j_1}^{(0)(t)} - \b{X}_{i_1j_1}^{*(t)}| \right)\delta_{R}
 	=  \mathcal{O}(t_\beta)
 	\end{align*}
 	where \[t_\beta = \mathcal{O}(\sqrt{s\epsilon_t})\], \[\delta_{R} := (1 - \eta_{x} + \eta_x\tfrac{\mu_t}{\sqrt{n}})^{R}\], \[\delta_{\HT}^{(t)} = 2m~{\exp}({-\tfrac{C^2}{\mathcal{O}^*(\epsilon_t^2)}})\], \[\delta_{\beta}^{(t)} = 2s~{\exp}(-\tfrac{1}{\mathcal{O}(\epsilon_t)})\], and \[\mu_t\] is the incoherence between the columns of \[\b{A}^{(t)}\].
 \end{lemma}
 \vspace{-2pt}
Also, the corresponding expression for \[\hat{\b{X}}^{(t)}\], which facilitates the analysis of the dictionary updates, is given by Lemma~\ref{iht:R_th_term}.
\vspace{-2pt}
 \begin{lemma}\textnormal{\textbf{(Expression for the coefficient estimate at the end of \[R\]-th IHT iteration)}}]\label{iht:R_th_term}
 	With probability at least \[(1 - \delta_\HT^{(t)} - \delta_{\beta}^{(t)})\] the \[i\]-th element of the coefficient estimate, for each \[i \in \supp(\b{x}^*)\], is given by
 	\begin{align*}
 	\hat{\b{x}}_{i} := \b{x}_{i}^{(R)} =  \b{x}_{i}^* (1 - \lambda^{(t)}_{i}) + \vartheta^{(R)}_{i}.
 	\end{align*}
 	Here, \[|\vartheta^{(R)}_{i}| = \mathcal{O}(t_{\beta})\], 
 	where \[t_\beta = \mathcal{O}(\sqrt{s\epsilon_t})\]. Further, \[\lambda^{(t)}_{i} = |\langle \b{A}^{(t)}_{i}- \b{A}^{*}_{i}, \b{A}^*_{i}\rangle| \leq \tfrac{\epsilon_t^2}{2}\], \[\delta_{\HT}^{(t)} = 2m~{\exp}({-\tfrac{C^2}{\mathcal{O}^*(\epsilon_t^2)}})\] and \[\delta_{\beta}^{(t)} = 2s~{\exp}(-\tfrac{1}{\mathcal{O}(\epsilon_t)})\].
 \end{lemma}
\vspace{-2pt}

Interestingly, Lemma~\ref{iht:x_R_error} shows that  the error in the non-zero elements of \[\hat{\b{X}}\] only depends on the error in the incoherent factor (dictionary) \[\b{A}^{(t)}\], which results in the following expression for \[ \xi^2\].
\begin{align}
  \xi^2 : = \mathcal{O}(s(1 - \omega)^{t/2}\|\b{A}_i^{(0)} - \b{A}_i^*\|),  ~\text{for all}~(i,j) \in \supp(\b{X}^*).
\end{align}
Therefore, if the the column-wise error in the dictionary decreases at each iteration \[t\], then the IHT-based sparse matrix estimates also improve progressively. 

\paragraph{{Recover Sparse Factors  \[\b{B}^*\] and \[\b{C}^*\] via Alg.\ref{algo:KRP_algo}:}} The results for the IHT-stage are foundational for the recovery of the sparse tensor factors \[\b{B}^{*(t)}\] and \[\b{C}^{*(t)}\] since they a) ensure correct signed-support recovery, guaranteed by  Lemma~\ref{our:signed_supp} and b) establish an upper-bound on the estimation error in \[\hat{\b{S}}^{(t)}\]. With these results, we now establish the correctness of Alg.~\ref{algo:KRP_algo} given an entry-wise \[\zeta\]-close estimate of \[\b{S}^{*(t)}\], \[|\hat{\b{S}}_{ij}^{(t)}-\b{S}^{*(t)}_{ij}|\leq \zeta\] given by the IHT stage. This procedure recovers the sparse factors \[\b{B}^{*(t)}\] and \[\b{C}^{*(t)}\], given element-wise \[\xi\]-close estimate \[\hat{\b{S}}\] of \[\b{S}^{*(t)}\]. The following lemma establishes recovery guarantees on the sparse factors using the SVD-based Alg.~\ref{algo:KRP_algo}, up to sign and scaling ambiguity. 
\vspace{-2pt}
\begin{lemma}\label{recover krp}
	Suppose the input \[\hat{\b{S}}^{(t)}\] to Alg.~\ref{algo:KRP_algo} is entry-wise \[\zeta\] close to \[\b{S}^{*(t)}\], i.e., \[|\hat{\b{S}}_{ij}^{(t)}-\b{S}^{*(t)}_{ij}|\leq \zeta\] and has the correct signed-support as  \[\b{S}^{*(t)}\]. Then with probability atleast \[(1-\delta_{\rm IHT}^{(t)}  - \delta_{\b{B}_i}^{(t)})\], both \[\hat{\b{B}}_i ^{(t)}\] and \[\hat{\b{C}}_i^{(t)} \] have the correct support, and  \[	\left\|\tfrac{\b{B}_i^{*(t)}}{\|\b{B}_i^{*(t)}\|} - \pi_i \tfrac{\hat{\b{B}}_i^{(t)}}{\|\hat{\b{B}}_i^{(t)}\|} \right\| = \c{O}(\zeta^2)\] and \[	\left\|\tfrac{\b{C}_i^{*(t)}*}{\|\b{C}_i^{*(t)}\|} - \pi_i \tfrac{\hat{\b{C}}_i^{(t)}}{\|\hat{\b{C}}_i^{(t)}\|} \right\| = \c{O}(\zeta^2)\], where \[\delta_{\rm IHT}^{(t)}  =  2m~{\exp}({-\tfrac{C^2}{\mathcal{O}^*(\epsilon_t^2)}}) + 2s~{\exp}(-\tfrac{1}{\mathcal{O}(\epsilon_t)})\] for \[\|\b{A}_i^{(t)} - \b{A}_i^*\|\leq \epsilon_t\], and \[\delta_{\b{B}_i}^{(t)} = \exp ({-\frac{\epsilon^2 J\alpha}{2(1 + \epsilon/3)}}) \] for any \[\epsilon>0\].
\end{lemma} 
\vspace{-2pt}
Here, we have used  \[\delta_{\rm IHT} ^{(t)} = \delta_{\beta}^{(t)} + \delta_{\HT}^{(t)} \] for simplicity. 

\paragraph{{Update Dictionary Factor \[\b{A}^{(t)}\]:}} The update of the dictionary factor involves concentration results which rely on an independent set of data samples. For this, notice that the \[i\]-th row of \[\b{S}^{*(t)}\] can be written as \[(\b{C}_i^{*(t)}\otimes \b{B}_i^{*(t)})^\top\].  Now, since \[\b{B}^{*(t)}\] and \[\b{C}^{*(t)}\] are sparse, there are a number of columns in \[\b{S}^{*(t)}\] which are degenerate (all-zeros). As a result, the corresponding data samples (columns of \[\b{Z}_1^{(t)\top}\]) are also degenerate, and cannot be used for learning. Furthermore, due to the dependence structure in \[\b{S}^{*(t)}\] (discussed in section~\ref{sec:main_res}) some of the data samples are dependent on each other, and at least from the theoretical perspective, are not eligible for the learning process. Therefore, we characterize the expected number of viable data samples in the following lemma.
\vspace{-2pt}
\begin{lemma}\label{our samples}
For \[L = \min(J,K)\], \[\gamma = \alpha\beta\], and any \[\epsilon>0\] and suppose we have 
	\begin{align*}
		L \geq \tfrac{2}{(1 - (1 - \gamma)^m)\epsilon^2}\log(\tfrac{1}{\delta_{p}^{(t)}}),
	\end{align*}
	 then with probability at least \[(1 - \delta_p)\], 
	\begin{align*}
	p=  L(1 - (1 - \gamma)^m),
	\end{align*}
	where \[\delta_{p}^{(t)} = \exp ({-\frac{\epsilon^2}{2}L(1 - (1 - \gamma)^m)})\].
\end{lemma}
\vspace{-2pt}
Here, we observe that the number of viable samples increase with number of independent samples \[L = \min(J,K)\], sparsity parameter \[\gamma = \alpha\beta\], and rank of the decomposition \[m\]. 
To recover the incoherent (dictionary) factor \[\b{A}^*\], we follow analysis similar to \cite[Lemma~5-9]{Rambhatla2019NOODL}. Here, we first develop an expression for the expected gradient vector in Lemma~\ref{grad_exp}. 
\vspace{-2pt}
\begin{lemma}\textnormal{\textbf{(Expression for the expected gradient vector)}}\label{grad_exp}
	Suppose that \[\b{A}^{(t)}\] is \[(\epsilon_t, 2)\]-near to \[\b{A}^*\]. Then, the dictionary update step in Alg.~\ref{alg:main_alg_tens} amounts to the following for the \[j\]-th dictionary element
	\vspace{-5pt}
		\begin{align*}
		\b{E}[\b{A}^{(t+1)}_j] = \b{A}^{(t)}_j + \eta_A \b{g}^{(t)}_j,
		\vspace{-5pt}
		\end{align*}	
		where for a small \[\tilde{\gamma}\], \[\b{g}^{(t)}_j\] is given by
			\vspace{-5pt}
		\begin{align*}
		\b{g}^{(t)}_j = q_j p_j \big((1 - \lambda^{(t)}_j)\b{A}^{(t)}_j- \b{A}^*_j + \tfrac{1}{q_j p_j}\Delta^{(t)}_j \pm \tilde{\gamma}\big),
		\end{align*}
	\[\lambda^{(t)}_{j}=|\langle \b{A}^{(t)}_j - \b{A}^{*}_j, \b{A}^*_j\rangle|\], and \[\Delta^{(t)}_j:=\b{E}[\b{A}^{(t)}_S\vartheta^{(R)}_S\sgn(\b{x}^*_j)]\], where \[\|\Delta^{(t)}_{j}\|= \mathcal{O}(\sqrt{m}q_{i,j}p_{j}\epsilon_t\|\b{A}^{(t)}\|)\].
\end{lemma}
\vspace{-2pt}
Since we use empirical gradient estimate, the following lemma establishes that the empirical gradient vector concentrates around its mean, and that it make progress at each step.
\vspace{-2pt}
\begin{lemma}\textnormal{\textbf{(Concentration of the empirical gradient vector)}}\label{lem:grad_vec_concentrates}
Given \[p = \tilde{\Omega}(mk^2)\] samples, the empirical gradient vector estimate corresponding to the \[i\]-th dictionary element, \[\hat{\b{g}}_i^{(t)}\] concentrates around its expectation, i.e.,
		\begin{align*}
		\|\hat{\b{g}}_i^{(t)} - \b{g}_i^{(t)}\| \leq o(\tfrac{s}{m}\epsilon_t).
		\end{align*}
		with probability at least \[(1 -\delta_{\gradvec}^{(t)}- \delta_{\beta}^{(t)} - \delta_{\HT}^{(t)} - \delta_{\rm HW}^{(t)})\], where \[\delta_{\gradvec}^{(t)} =  \exp(-\Omega(s))\].
\end{lemma}
\vspace{-2pt}
We then leverage Lemma~\ref{grad_corr} to show that  the empirical gradient vector \[\hat{\b{g}}^{(t)}_j \] is correlated with the descent direction (see Def.~\ref{def:grad_alpha_beta}), which ensures that the dictionary estimate makes progress at each iteration of the online algorithm.
\vspace{-2pt}
\begin{lemma}\textnormal{\textbf{(Empirical gradient vector is correlated with the descent direction)}}\label{grad_corr}
Suppose \[\b{A}^{(t)}\] is \[(\epsilon_t, 2)\]-near to \[\b{A}^*\], \[s = \mathcal{O}(\sqrt{n})\] and \[\eta_A = \mathcal{O}(m/s)\]. Then, with probability at least \[(1 - \delta_{\HT}^{(t)} - \delta_{\beta}^{(t)} - \delta_{\rm HW}^{(t)} -\delta_{\gradvec}^{(t)})\] the empirical gradient vector \[\hat{\b{g}}^{(t)}_j\] %given by
	 is \[(\Omega(k/m), \Omega(m/k),0 )\]-correlated with \[(\b{A}^{(t)}_j - \b{A}^*_j)\], and for any \[t \in [T]\],
	\begin{align*}
	\|\b{A}^{(t+1)}_j- \b{A}^*_j\|^2 \leq (1 - \rho_{\_}\eta_A)\|\b{A}^{(t)}_j- \b{A}^*_j\|^2.
	\end{align*}
	%	\textcolor{red}{Here, \[\delta_{\HT}^{(t)} =?\], \[\delta_{\beta} =?\], \[\delta_{\gradvec}^{(t)} =?\].}
\end{lemma}
\vspace{-2pt}
This step also requires closeness that the estimate \[\b{A}^{(t)}\] and \[\b{A}^*\] are close, both column-wise and in the spectral norm-sense, as per Def~\ref{def:del_kappa}. To this end, we show that the updated dictionary matrix maintain the closeness property. For this, we first show that the gradient matrix concentrates around its mean in Lemma~\ref{lem:grad_mat}. 
\vspace{-2pt}
\begin{lemma}\textnormal{\textbf{(Concentration of the empirical gradient matrix)}}\label{lem:grad_mat}
With probability at least \[(1 - \delta_{\beta}^{(t)} - \delta_{\HT}^{(t)} - \delta_{\rm HW}^{(t)} - \delta_\gradmat^{(t)})\], 
	\[\|\hat{\b{g}}^{(t)} -  \b{g}^{(t)}\|\] is upper-bounded by \[ \mathcal{O}^*(\tfrac{s}{ m} \|\b{A}^*\| )\], 
		where \[\delta_{\gradmat}^{(t)} = (n+m)\exp(-\Omega(m\sqrt{\log(n)})\].
\end{lemma} 
\vspace{-2pt}
 Further, the closeness property is maintained, as shown below. 
 \vspace{-2pt}
\begin{lemma}\textnormal{\textbf{(\[\b{A}^{(t+1)}\] maintains closeness)}}\label{lem:closeness}
	Suppose \[\b{A}^{(t)}\] is \[(\epsilon_t, 2)\] near to \[\b{A}^*\] with \[\epsilon_t = \mathcal{O}^*(1/\log(n))\], and number of samples used in step \[t\] is \[p = \tilde{\Omega}(ms^2)\], then with probability at least \[(1 - \delta_{\HT}^{(t)} - \delta_{\beta}^{(t)} 
	- \delta_{\rm HW}^{(t)}- \delta_{\gradmat}^{(t)})\], \[\b{A}^{(t+1)}\] satisfies \[\|\b{A}^{(t+1)} - \b{A}^*\| \leq 2\|\b{A}^*\|\].
\end{lemma}
\vspace{-2pt}
Therefore, the recovery of factor \[\b{A}^*\], and the sparse-structured matrix \[\b{X}^*\] suceeds with probability \[\delta_{\text{NOODL}}^{(t)} =  \delta_{\HT}^{(t)}  + \delta_{\beta}^{(t)} + \delta_{\rm HW} +\delta_{\gradvec}^{(t)} + \delta_{\gradmat}^{(t)}\], where \[\delta_{\HT}^{(t)} = 2m~{\exp}({-{C^2}/{\mathcal{O}^*(\epsilon_t^2)}})\], \[\delta_\beta^{(t)} = 2s~{\exp}(-{1}/{\mathcal{O}(\epsilon_t)})\], \[\delta_{\rm HW}^{(t)} = \exp(-{1}/{\mathcal{O}(\epsilon_t)})\], \[\delta_{\gradvec}^{(t)} =  \exp(-\Omega(s))\], \[\delta_{\gradmat}^{(t)} = (n+m)\exp(-\Omega(m\sqrt{\log(n)})\]. 

Further, from Lemma~\ref{our sparsity}, we have that the columns of \[\b{S}^{*(t)}\] are \[s = \c{O}(\alpha \beta m)\] sparse with probability \[(1 - \delta_s^{(t)})\], where \[\delta_s^{(t)} = \min(J,K)\exp( {-{\epsilon^2 \alpha \beta m}/{2(1 + \epsilon/3)}})\] for any \[\epsilon > 0\], and that with probability at least \[(1 - \delta_p)\], the number of data samples \[p =  L(1 - (1 - \gamma)^m)\], where \[\delta_{p}^{(t)} = \exp ({-\frac{\epsilon^2}{2}L(1 - (1 - \gamma)^m)})\] using Lemma~\ref{our sparsity}. Furthermore, from Lemma~\ref{recover krp}, we know that Alg.~\ref{algo:KRP_algo} (which only relies on recovery of \[\b{X}^{*(t)}\]) succeeds in recovering \[\b{B}^{*(t)}\] and \[\b{C}^{*(t)}\] (upto permutation and scaling) with probability \[(1 - \delta_{\b{B}_i}^{(t)})\], where \[\delta_{\b{B}_i}^{(t)} = \exp ({-\frac{\epsilon^2 J\alpha}{2(1 + \epsilon/3)}}) \] for any \[\epsilon>0\]. Combining all these results we have that, Alg.~\ref{alg:main_alg_tens} succeeds with probability \[(1 - \delta_{alg})\], where \[\delta_{alg} = \delta_{\rm s} + \delta_{p}^{(t)} + \delta_{\b{B}_i}^{(t)} + \delta_{\rm NOODL} \]. Also, the total run time of the algorithm is \[\c{O}(mnp\log(1/\delta_R) \max(\log(1/\epsilon_T),\log(\sqrt(s)/\delta_T))\] for \[p = \Omega(ms^2)\]. Hence, our main result. 

\noindent\textbf{A note on independent sample requirement:} Since the IHT-based coefficient operates independently on each column of \[\b{Y}^{(t)}\] (the non-zero columns of \[\b{Z}_1^{(t)}\top\]), the dependence structure of \[\b{S}^{*(t)}\] does not affect this stage. For the dictionary update (in theory) we only use the independent columns of \[\b{Y}^{(t)}\], these can be inferred using \[J\] and \[K\], and corresponding induced transposed Khatri-Rao structure. In practice, we don't need to throw away any samples, this is purely to ensure that the independence assumption holds for our finite sample analysis of the algorithm.
\end{proof}

%\vspace{-10pt}

\section{Proof of Intermediate Results}\label{app:tensor}
   \paragraph{Lemma~\ref{our sparsity}}
   \textit{If \[m = \Omega(\log({\min(J,K)})/{\alpha \beta})\] then with probability at least \[(1-\delta_s^{(t)})\] the number of non-zeros, \[s\], in a column of \[\b{S}^{*(t)}\] are upper-bounded as \[s = \c{O}(\alpha \beta m)\], where \[\delta_s^{(t)} = \min(J,K)\exp( -{\epsilon^2 \alpha \beta m}/2(1+ \epsilon/3))\] for any \[\epsilon > 0\].}
\vspace{10pt}

\begin{proof}\textbf{of Lemma~\ref{our sparsity}}
Consider a column of the transposed Khatri-Rao structured matrix \[\b{S}^{*(t)}\] defined as \[\b{S}^{*(t)} =(\b{C}^{*(t)}\odot \b{B}^{*(t)})^\top\]. Here, since the entries of factors \[\b{B}^{*(t)}\] and \[\b{C^{*(t)}}\] are independently non-zero with probability \[\alpha\] and \[\beta\], respectively, each entry of a column of \[\b{S}^{*(t)}\] is independently non-zero with probability \[\gamma = \alpha \beta\], i.e., \[\mathbbm{1}_{|\b{S}^{*(t)}_{ij}|>0}\sim\rm{Bernoulli}(\gamma)\]. As a result, the number of non-zero elements in a column of \[\b{S}^{*(t)}\] are \[\rm{Binomial}(m, \gamma)\].

\noindent Now, let \[\b{s}_{ij}\] be the indicator for the \[(i,j)\] element of \[\b{S}^{*(t)}\] being non-zero, defined as
	\begin{align*}
	\b{s}_{ij} = \mathbbm{1}_{|\b{S}^{*(t)}_{ij}|>0}.
	\end{align*}
Then, the expected number of non-zeros (sparsity) in the \[j\]-th column of \[\b{S}^{*(t)}\] are given by
	\begin{align*}
	\mathbf{E}[ {\textstyle\sum_{ i = 1}^{m}} \b{s}_{ij} ]  = \gamma m.
	\end{align*}
	Since, \[\gamma\] can be small, we use Lemma~\ref{theorem:rel_chern}(a) \citep{Mcdiarmid98} to derive an upper bound on the sparsity for each each column as 
	\begin{align*}
	\mathbf{Pr}[\textstyle \sum_{ i = 1}^{m} \b{s}_{ij} \geq (1 + \epsilon)\gamma m ] \leq \exp ({-\frac{\epsilon^2 \gamma m}{2(1 + \epsilon/3)}}).
	\end{align*}
	for any \[\epsilon>0\].
	%Therefore, with high probability the number of non-zero elements in a column are upper bounded by\[(1+\epsilon)\gamma m\]. 
	 Union bounding over \[L =\min(J,K)\] independent columns of \[\b{S}^{*(t)}\].
	\begin{align*}
	\mathbf{Pr}[~\textstyle\bigcup_{ j = 1}^{L} (\textstyle \sum_{ i = 1}^{m} \b{s}_{ij} \leq (1 + \epsilon)\gamma m) ] \geq 1 -  L\exp({-\tfrac{\epsilon^2 \gamma m}{2(1 + \epsilon/3)}}).
	\end{align*}%
	%We would like the quantity $Le^{-\frac{\epsilon^2 \gamma m}{2(1 + \epsilon/3)}}$ to be smaller than some $\delta$,
	Therefore, we conclude that if  \[m = \Omega({\log(L)}/{\gamma})\] then with probability \[(1-\delta_{s})\] the expected number of non-zeros in a column of \[\b{S}^{*(t)}\] are \[\c{O}(\gamma m)\], where \[\delta_{s} = L\exp( {-{\epsilon^2 \gamma m}/{2(1 + \epsilon/3)}})\].
	
\end{proof}
\vspace{-20pt}
\paragraph{Lemma~\ref{our samples}}
	For any \[\epsilon>0\] suppose we have 
	\begin{align*}
		L \geq \tfrac{2}{(1 - (1 - \gamma)^m)\epsilon^2}\log(\tfrac{1}{\delta_{p}^{(t)}}),
	\end{align*}
	for \[L = \min(J,K)\] and \[\gamma = \alpha\beta\], then with probability at least \[(1 - \delta_p)\], 
	\begin{align*}
	p =  L(1 - (1 - \gamma)^m),
	\end{align*}
	where \[\delta_{p}^{(t)} = \exp ({-\frac{\epsilon^2}{2}L(1 - (1 - \gamma)^m)})\].
	
	\vspace{10pt}
	
\begin{proof}\textbf{of Lemma~\ref{our samples}}
	%here, we estimate the number of non-zero columns of \[\b{S}\].
	We begin by evaluating the probability that a column of \[\b{S}^{*(t)}\] has a non-zero element. Let \[\b{s}_{ij}\] be the indicator for the \[(i,j)\] element of \[\b{S}^{*(t)}\] being non-zero, defined as
		\begin{align*}
		\b{s}_{ij} = \mathbbm{1}_{|\b{S}^*_{ij}|>0}.
		\end{align*}
	Further, let \[w_j\] denote the number of non-zeros in the \[j\]-th column of \[\b{S}^{*(t)}\], defined as
	\begin{align*}
	w_j = \textstyle\sum_{ i = 1}^{m} \b{s}_{ij}.
	\end{align*}
	Since each element of a column of \[\b{S}^{*(t)}\] is non-zero with probability \[\gamma\], the probability that the \[j\]-th column of \[\b{S}^{*(t)}\] is an all zero vector is,
	\begin{align*}
	\mathbf{Pr}[w_j  = 0]  = (1 - \gamma)^m. 
	\end{align*}
	Therefore, the probability that the \[j\]-th column of \[\b{S}^{*(t)}\] has at least one non-zero element is given by
	\begin{align}\label{eq:pr_non_zero}
	\mathbf{Pr}[w_j  > 0]  = 1 - (1 - \gamma)^m. 
	\end{align}
	Now, we are interested in the number of columns with at least one non-zero element among the \[L = \min(J, K)\] independent columns of \[\b{S}^{*(t)}\], which we denote by \[p\]. Specifically, we analyze the following sum %The quantity of our interest is the number of samples, $p$, which can be written as
	\begin{align*}
	p = \textstyle \sum_{ j = 1}^{L} \mathbbm{1}_{w_j>0}.
	\end{align*}
	Next, using \eqref{eq:pr_non_zero} \[\mathbf{E}[p] =  L(1 - (1 - \gamma)^m)\].
	Applying the result stated Lemma~\ref{theorem:rel_chern} (b),
	\begin{align*} 
	\mathbf{Pr}\left[\textstyle \sum\limits_{ j = 1}^{L} \mathbbm{1}_{w_j} \leq  (1 - \epsilon) \mathbf{E}[p]\right] \leq \exp ({-\tfrac{\epsilon^2 E[p]}{2}}) := \delta_{p}^{(t)}.
	\end{align*}
	Therefore, if for any \[\epsilon>0\] we have 
	\vspace{-0pt}
	\begin{align*} 
		L \geq \tfrac{2}{(1 - (1 - \gamma)^m)\epsilon^2}\log\left(\tfrac{1}{\delta_{p}^{(t)}}\right)
	\end{align*}
		\vspace{-0pt}
	then with probability at least \[(1 - \delta_p)\], \[p =  L(1 - (1 - \gamma)^m)\], where \[\delta_{p}^{(t)} = \exp ({-\frac{\epsilon^2}{2}L(1 - (1 - \gamma)^m)})\].
\end{proof}
\vspace{-10pt}
\paragraph{Lemma~\ref{recover krp}}
	Suppose the input \[\hat{\b{S}}^{(t)}\] to Alg.~\ref{algo:KRP_algo} is entry-wise \[\zeta\] close to \[\b{S}^{*(t)}\], i.e., \[|\hat{\b{S}}^{(t)}_{ij}-\b{S}^{*(t)}_{ij}|\leq \zeta\] and has the correct signed-support as  \[\b{S}^{*(t)}\]. Then with probability atleast \[(1-\delta_{\rm IHT}^{(t)}  - \delta_{\b{B}_i}^{(t)})\], both \[\hat{\b{B}}_i^{(t)} \] and \[\hat{\b{C}}_i^{(t)} \] have the correct support, and  \[	\left\|\tfrac{\b{B}_i^{*(t)}}{\|\b{B}_i^{*(t)}\|} - \pi_i \tfrac{\hat{\b{B}}_i^{(t)}}{\|\hat{\b{B}}_i^{(t)}\|} \right\| = \c{O}(\zeta^2)\] and \[	\left\|\tfrac{\b{C}_i^{*(t)}}{\|\b{C}_i^{*(t)}\|} - \pi_i \tfrac{\hat{\b{C}}_i^{(t)}}{\|\hat{\b{C}}_i^{(t)}\|} \right\| = \c{O}(\zeta^2)\], where \[\delta_{\rm IHT}^{(t)}  =  2m~{\exp}({-\tfrac{C^2}{\mathcal{O}^*(\epsilon_t^2)}}) + 2s~{\exp}(-\tfrac{1}{\mathcal{O}(\epsilon_t)})\] for \[\|\b{A}_i^{(t)} - \b{A}_i^*\|\leq \epsilon_t\], and \[\delta_{\b{B}_i}^{(t)} = \exp ({-\frac{\epsilon^2 J\alpha}{2(1 + \epsilon/3)}}) \] for any \[\epsilon>0\].
	
	\vspace{10pt}
\begin{proof}\textbf{of Lemma~\ref{recover krp}}
The Iterative Hard Thresholding (IHT) results in an estimate of \[\b{X}^{*(t)}\] which has the correct signed support \cite{Rambhatla2019NOODL}. As a result, putting back the columns of \[\hat{\b{X}}^{(t)}\] at the respective non-zero column locations of \[\b{Z}_1^{(t)\top}\], we arrive at the estimate \[\hat{\b{S}}^{(t)}\] of \[\b{S}^{*(t)}\], which has the correct signed-support, we denote this estimate by \[\hat{\b{S}}^{(t)}\]. 
To recover the estimates \[\hat{\b{B}}^{(t)}\] and \[\hat{\b{C}}^{(t)}\], we use a SVD-based procedure. Specifically, we note that, 
\begin{align*}
\b{S}_{i,:}^{*(t)\top} = \b{C}_i^{*(t)}\otimes\b{B}_i^{*(t)}  = vec({\b{B}_i^{*(t)}}{\b{C}_i}^{*(t)\top})
\end{align*}
As a result, the left and right singular vectors of the rank-1 matrix \[\b{B}_i^{*(t)}\b{C}_i^{*(t)\top}\] are the columns \[\b{B}_i^{*(t)}\] and \[\b{C}_i^{*(t)}\], respectively (up to scaling). 

Let \[\b{M}^{(i)}\] denote the \[J \times K\] matrix formed by reshaping the vector \[\hat{\b{S}}_{i,:}^{(t)\top}\]. We choose the appropriately scaled left and right singular vectors corresponding to the largest singular value of \[\b{M}^{(i)}\] as our estimates \[\hat{\b{B}}_i^{(t)}\] and \[\hat{\b{C}}_i^{(t)}\], respectively.

First, notice that since \[\hat{\b{S}}_{i,:}^{(t)\top}\] has the correct sign and support (due to Lemma~\ref{our:signed_supp}), the support of matrix \[\b{M}^{(i)}\] is the same as \[\b{B}_i^{*(t)}\b{C}_i^{*(t)\top}\]. As a result, the estimates \[\hat{\b{B}}_i^{(t)}\] and \[\hat{\b{C}}_i^{(t)}\] have the correct support, and the error is only due to the scaling ambiguity on the support. This is due to the fact that the principal singular vectors (\[\b{u}\] and \[\b{v}\]) align with the sparsity structure of \[\b{M}^{(i)}\] as they solve the following  maximization problem also known as variational characterization of svd,
\begin{align*}
\sigma_1^2 = \underset{\|\b{u}\|=1}{\max} \b{u}^\top \b{M}^{(i)}{\b{M}^{(i)}}^\top \b{u}  = \underset{\|\b{v}\|=1}{\max} \b{v}^\top {\b{M}^{(i)}}^\top\b{M}^{(i)} \b{v},
\end{align*}
where \[\sigma_1\] denotes the principal singular value. Therefore, since \[\b{M}^{(i)}\] has the correct sparsity structure as \[{\b{B}_i^{*(t)}}{\b{C}_i}^{*(t)\top}\] the resulting \[\b{u}\] and \[\b{v}\] have the correct supports as well.  Here, \[\b{u}\] and \[\b{v}\] can be viewed as the normalized versions of \[\hat{\b{B}}_i^{(t)}\] and \[\hat{\b{C}}_i^{(t)}\], respectively, i.e., \[\b{u} = \hat{\b{B}}_i^{(t)}/\|\hat{\b{B}}_i^{(t)}\|\]  and \[\b{v} = \hat{\b{C}}_i^{(t)}/\|\hat{\b{C}}_i^{(t)}\|\].

Let \[\b{E} = \b{M}^{(i)}- \b{B}_i^{*(t)}\b{C}_i^{*(t)\top}\], now since \[|\hat{\b{S}}_{ij}^{(t)}-\b{S}_{ij}^{*(t)}|\leq \zeta\] and, from Lemma~\ref{our:signed_supp})  \[\hat{\b{S}}_{i,:}^{(t)}\] has the correct signed-support with probability \[(1 - \delta_{\rm IHT}^{(t)} )\], where \[\delta_{\rm IHT}^{(t)}  =  2m~{\exp}({-\tfrac{C^2}{\mathcal{O}^*(\epsilon_t^2)}}) + 2s~{\exp}(-\tfrac{1}{\mathcal{O}(\epsilon_t)})\], and further using Claim~\ref{lem:nonzero_in_a_row}, we have that the expected number of non-zeros in  \[\hat{\b{S}}_{i,:}^{(t)}\] are \[JK\alpha\beta\], with probability atleast \[(1 - \delta_{\b{B}_i}^{(t)})\], where \[ \delta_{\b{B}_i}^{(t)} = \exp ({-\frac{\epsilon^2 J\alpha}{2(1 + \epsilon/3)}}) \] for some \[\epsilon>0\], we have
	\begin{align*}
	\|\b{E} \|\leq \|\b{E} \|_{\rm F} \leq \sqrt{JK\alpha\beta}\zeta,
	\end{align*}
Then, using the result in \cite{Yu14}, and noting that \[\sigma_1(\b{B}_i^{(t)}\b{C}_i^{(t)\top}) = \|\b{B}_i^{(t)}\|\|\b{C}_i^{(t)}\|\] and letting \[\pi_i \in \{-1, 1\}\] (to resolve the sign ambiguity), we have that 
	\begin{align*}
	\left\| \tfrac{\b{B}_i^{*(t)}}{\|\b{B}_i^{*(t)}\|}- \pi_i \b{u}\right\| = \left\| \tfrac{\b{B}_i^{*(t)}}{\|\b{B}_i^{*(t)}\|} - \pi_i \tfrac{\hat{\b{B}}_i^{(t)}}{\|\hat{\b{B}}_i^{(t)}\|}\right\|\leq  \tfrac{2^{3/2}(2 \|\b{B}_i^{(t)}\|\|\b{C}_i^{(t)}\|+\sqrt{JK\alpha\beta}\zeta)\sqrt{JK\alpha\beta}\zeta}{ \|\b{B}_i^{(t)}\|^2\|\b{C}_i^{(t)}\|^2}.
	\end{align*}
	Next, since \[\mathbf{E}[(\b{B}_{ij}^{(t)})^2 |(i, j) \in \supp(\b{B}^{(t)}) ] = 1\] as per our distributional assumptions \textbf{Def.\ref{def:dist_bc}}, we have 
	\begin{align*}
	\mathbf{E}[\|\b{B}_{ji}^{*(t)}\|^2] = \mathbf{E}[(\b{B}_{ji}^{*(t)})^{2}|(j,i) \in \supp(\b{B}^{*(t)})]\mathbf{Pr}[(j, i) \in \supp(\b{B}^{*(t)})] + 0.\mathbf{Pr}[( j, i) \notin \supp(\b{B}^{*(t)})] = \alpha
	\end{align*}
	Similarly, \[\mathbf{E}[\|\b{C}_{ji}^{*(t)}\|^2] = \beta\]. Substituting,
	\begin{align*}
		\left\|\tfrac{\b{B}_i^{*(t)}}{\|\b{B}_i^{*(t)}\|} - \pi_i \tfrac{\hat{\b{B}}_i^{(t)}}{\|\hat{\b{B}}_i^{(t)}\|} \right\|\leq  \tfrac{2^{3/2}(2 \sqrt{JK\alpha \beta }+\sqrt{JK\alpha\beta}\zeta)\sqrt{JK\alpha\beta}\zeta}{JK\alpha\beta}  = \c{O}(\zeta^2).
	\end{align*}
%	Now, since \[{\rm sin}~\Theta\] 
\end{proof}
\vspace{-10pt}
\begin{claim}\label{lem:nonzero_in_a_row}
Suppose \[ J = \Omega(\tfrac{1}{\alpha}))\], then with probability at least \[(1-\delta_{\b{B}_i}^{(t)})\],
\begin{align*}
\textstyle\sum_{j=1}^{JK}\supp(\b{S}^*(i,j))  = JK\alpha\beta,
\end{align*}
where \[\delta_{\b{B}_i}^{(t)} = \exp ({-\frac{\epsilon^2 J\alpha}{2(1 + \epsilon/3)}}) \] for any \[\epsilon>0\].
\end{claim}
\begin{proof}\textbf{of Claim~\ref{lem:nonzero_in_a_row}}
In this lemma we establish an upper-bound on the number of non-zeros in a row of \[\b{S}^{*(t)}\]. The \[i\]-th row of \[\b{S}^{*(t)}\] can be written as \[\rm{vec}(\b{B}_i^{*(t)}\b{C}_i^{*(t)\top})\]. 

Since each element of matrix \[\b{B}^{*(t)}\] and \[\b{C}^{*(t)}\] are independently non-zero with probabilities \[\alpha\] and \[\beta\], the number of non-zeros in a column \[\b{B}_i^{*(t)}\] of \[\b{B}^{*(t)}\] are binomially distributed.
Let \[\b{s}_{j}\] be the indicator for the \[j\]-th element of \[\b{B}^{*(t)}_i\] being non-zero, defined as
\begin{align*}
\b{s}_{i} = \mathbbm{1}_{|\b{B}^{*(t)}(j,i)|>0}.
\end{align*}
Then, the expected number of non-zeros (sparsity) in the \[i\]-th column of \[\b{B}^{*(t)}\] are given by
\begin{align*}
\b{E}[\textstyle\sum\supp(\b{B}^{*(t)}_i)] = \mathbf{E}[ {\textstyle\sum_{ j = 1}^{J}} \b{s}_{j} ]  = J\alpha.
\end{align*}	
Since, \[\alpha\] can be small, we use Lemma~\ref{theorem:rel_chern}(a) \citep{Mcdiarmid98} to derive an upper bound on the sparsity for each each column as 
\begin{align}\label{eq:num_nonz_B}
\mathbf{Pr}[\textstyle \sum_{ j = 1}^{J} \b{s}_{j} \geq (1 + \epsilon) J\alpha ] \leq \exp ({-\frac{\epsilon^2 J\alpha}{2(1 + \epsilon/3)}}) := \delta_{\b{B}_i}^{(t)}.
\end{align}
for any \[\epsilon>0\].

Now we turn to the number of non-zeros in \[\b{S}^{*(t)}_i = \rm{vec}(\b{B}^{*(t)}_i\b{C}_i^{*(t)\top})\].  We first note that the \[j\]-th column of \[\b{B}_i^{*(t)}\b{C}_i^{*(t)\top}\]  is given by \[\b{C}(j,i)^{*(t)}\b{B}_i^{*(t)}\]. This implies that the \[j\]-th column can be all-zeros if \[\b{C}(j,i)^{*(t)} = 0\]. As a result, the expected number of non-zeros in the \[j\]-th column of \[\b{B}^{*(t)}_i\b{C}_i^{*(t)\top}\] can be written as,
\begin{align*}
&\b{E}[\textstyle\sum\supp(\b{C}^{*(t)}_{ji}\b{B}^{*(t)}_i)] \\
&=\b{E}[\textstyle\sum\supp(\b{C}^{*(t)}_{ji}\b{B}^{*(t)}_i)| \b{C}^{*(t)}_{ji}\neq 0] \b{Pr}[\b{C}^{*(t)}_{ji} \neq 0] + \b{E}[\textstyle\sum\supp(\b{C}^{*(t)}_{ji} \b{B}^{*(t)}_i)| \b{C}^{*(t)}_{ji} = 0]\b{Pr}[\b{C}_{ji}^{*(t)} = 0]\\
&=\b{E}[\textstyle\sum\supp(\b{C}^{*(t)}_{ji}\b{B}^{*(t)}_i)| \b{C}^{*(t)}_{ji} \neq 0] \b{Pr}[\b{C}^{*(t)}_{ji}\neq 0] = \b{E}[\textstyle\sum\supp(\b{B}^{*(t)}_i)]\b{Pr}[\b{C}^{*(t)}_{ji} \neq 0].
\end{align*}
Now, from \eqref{eq:num_nonz_B}, we have that if we choose \[ J = \Omega(\tfrac{1}{\alpha}))\] with probability atleast \[(1-\delta_{\b{B}_i}^{(t)})\], there are \[ J\alpha\] non-zeros in a column of \[\b{B}^{*(t)}\]. Further since, \[\b{Pr}[\b{C}_{ji}^{*(t)} \neq 0] = \beta\], we have that with probability atleast \[(1-\delta_{\b{B}_i}^{(t)})\],
\begin{align*}
\b{E}[\textstyle\sum\supp(\b{C}^{*(t)}_{ji}\b{B}_i^{*(t)})]  = J\alpha\beta.
\end{align*}
Furthermore, since there are \[K\] columns in \[\b{B}^{*(t)}_i\b{C}_i^{*(t)\top}\], with probability atleast \[(1-\delta_{\b{B}_i}^{(t)})\],
\begin{align*}
\b{E}[\textstyle\sum\supp({\rm{vec}}(\b{B}^{*(t)}_i\b{C}_i^{*(t)\top})]  = \b{E}[\textstyle\sum_{j=1}^{JK}\supp(\b{S}^{*(t)}(i,j))] = JK\alpha\beta.
\end{align*}
\end{proof}
\vspace{-15pt}
\section{Additional Theoretical Results}\label{app:add}
	\begin{lemma}\textbf{Relative Chernoff} \cite{Mcdiarmid98} \label{theorem:rel_chern} 
		Let random variables \[w_1, \dots, w_\ell\] be independent, with \[0 \leq w_i \leq 1\] for each \[i\]. Let \[S_w = \textstyle \sum_{i=1}^{\ell} w_i\], let \[\nu = \b{E}(S_w)\] and let \[p = \nu/\ell\], then for any \[\epsilon>0\],
		\begin{align*}
		(a)& ~~~~ \mathbf{Pr}[S_w - \nu \geq \epsilon \nu] \leq \exp( {-{\epsilon^2 \nu}/{2(1 + \varepsilon/3)}}),\\
		(b)& ~~~~ \mathbf{Pr}[ S_w - \nu \leq \epsilon \nu]\leq \exp({-{\epsilon^2 \nu}/{2}}).
		\end{align*}
	\end{lemma}
	\begin{lemma}[From Theorem~4 in \cite{Yu14} for singular vectors]
			Given \[\b{M}\], \[\tilde{\b{M}} \in \mathbb{R}^{m \times n}\], where \[\tilde{\b{M}} = \b{M} + \b{E}\] and the corresponding SVD of \[\b{M}= \b{U}\b{\Sigma} \b{V}^\top\] and \[\tilde{\b{M}} = \tilde{\b{U}}\tilde{\b{\Sigma}}\tilde{\b{V}}^\top\], the sine of angle between the principal left (and right) singular vectors of matrices \[\b{M}\] and \[\tilde{\b{M}}\] is given by
			\begin{align*}
			{\rm sin}~\Theta(\b{U}_1, \tilde{\b{U}}_1) \leq \tfrac{2(2 \sigma_1+\|\b{E}\|_2)(\min(\|\b{E}\|_2,\|\b{E}\|_{\rm F})}{\sigma_1^2},
			%sin\Theta(v_1, \tilde{v}_1) \leq \tfrac{2(2 \sigma_1+\|E\|_2)(min(\|E\|_2,\|E\|_F)}{\sigma_1^2}
			\end{align*}
			where \[\sigma_1\] is the principal singular value corresponding to \[\b{U}_1\]. Furthermore, there exists \[\pi \in {-1, 1}\] s.t.
			\begin{align*}
			\|\b{U}_1 - \pi\tilde{\b{U}}_1\| \leq \tfrac{2^{3/2}(2 \sigma_1+\|\b{E}\|_2)(\min(\|\b{E}\|_2,\|\b{E}\|_{\rm F})}{\sigma_1^2}.
			\end{align*}
		\end{lemma}
	\begin{theorem}[\cite{Rambhatla2019NOODL}]\label{main_result_noodl}\textit{Suppose that assumptions \ref{assumption:mu}-\ref{assumption:step coeff}  hold, and Alg.~\ref{alg:main_alg_tens} is provided with  \[ p = \tilde{\Omega}(mk^2)\] new samples generated according to model \eqref{CPD} at each iteration \[t\]. %then for \[\eta_A = \Theta(m/k)\] 
	Then for some \[0 < \omega < 1/2\], the estimate \[\b{A}^{(t)}\] at \[(t)\]-th iteration satisfies 
	\vspace{-2pt}
	\begin{align*}
	\|\b{A}_i^{(t)} - \b{A}_i^*\|^2 \leq (1 - \omega)^t\|\b{A}_i^{(0)} - \b{A}_i^*\|^2,~\text{for all}~t = 1,2,\ldots.\vspace{-2pt}
	\end{align*}
	Furthermore, given \[R = \Omega({\rm log}(n))\], with probability at least \[(1 - \delta_{\text{alg}}^{(t)})\] for some small constant \[\delta_{\text{alg}}^{(t)}\], the coefficient estimate \[\hat{\b{x}}_{i}^{(t)}\] at \[t\]-th iteration has the correct signed-support and satisfies
	\vspace{-2pt}
	\begin{align*}
	(\hat{\b{x}}_{i}^{(t)} - \b{x}_{i}^*)^2 %C_{i_1}^{(R)} 
	&= \mathcal{O}(k(1 - \omega)^{t/2}\|\b{A}_i^{(0)} - \b{A}_i^*\|), ~\text{for all}~i \in \supp({\b{x}^*}).\vspace{-2pt}
	%(c_x k + 1)(\tfrac{\epsilon_t^2}{2}|\b{x}_{\max}^*|+ t_\beta) + 2\delta_{R} k\eta_x\tfrac{\mu_t}{\sqrt{n}},\\
	%& \leq c + (c_x k + 1)\tfrac{\epsilon_t^2}{2}|\b{x}_{\max}^*|
	\end{align*} 
	}
	\end{theorem}
%\input{NOODL_dict.tex}
%\clearpage
\vspace{20pt}
\section{Experimental Evaluation}\label{app:exp}
\vspace{-10pt}
	\renewcommand{\arraystretch}{1}
	\begin{table*}[!t]
	\caption{ Tensor factorization results $\alpha,\beta=0.005$ averaged across $3$ trials. Here, \[T(\supp(\hat{\b{X}}^{(T)})?)\] field shows the number of iterations \[T\] to reach the target tolerance, while the categorical field, \[\supp(\hat{\b{X}}^{(T)})\] indicates if the support of the recovered \[\hat{\b{X}}^{(T)}\] matches that of \[\b{X}^{*(T)}\] (Y) or not (N).}
	\vspace{0pt}
\label{alpha_0_005}
\captionsetup{justification=centering}
\captionsetup{font=footnotesize}
		 \centering
		 \newcolumntype{?}{!{\vrule width 1.5pt}}
	\par\bigskip
			   	\scalebox{0.73}{
			   	\begin{tabular}{P{0.8cm}|P{1.5cm}|P{1.5cm}|P{2cm}|P{1.4cm}?P{1.5cm}|P{2cm}|P{1.3cm}?P{1.5cm}|P{2cm}|P{1.3cm}}
				\hline 
				 \multirow{3}{0.8cm}{\centering\textbf{$(J,K)$}}&\multirow{3}{1.5cm}{\centering \textbf{Method}}&\multicolumn{3}{c?}{$m=50$}&\multicolumn{3}{c?}{$m=150$}&\multicolumn{3}{c}{$m=300$}\\\cline{3-11}
				  & & \multirow{2}{2cm}{$\centering\tfrac{\|\b{A}^* - \b{A}^{(T)}\|_{\rm F}}{\|\b{A}^*\|_{\rm F}}$}& \multirow{2}{2cm}{$\centering\tfrac{\|\b{X}^{*(T)} - \b{X}^{(T)}\|_{\rm F}}{\|\b{X}^{*(T)}\|_{\rm F}}$}& \multirow{2}{1.3cm}{\centering\scriptsize$\hspace{-4pt}T(\supp(\hat{\b{X}})?)$}
				  &\multirow{2}{2cm}{$\centering\tfrac{\|\b{A}^* - \b{A}^{(T)}\|_{\rm F}}{\|\b{A}^*\|_{\rm F}}$}& \multirow{2}{2cm}{$\centering\tfrac{\|\b{X}^{*(T)} - \b{X}^{(T)}\|_{\rm F}}{\|\b{X}^{*(T)}\|_{\rm F}}$}& \multirow{2}{1.3cm}{\centering\scriptsize$\hspace{-4pt}T(\supp(\hat{\b{X}})?)$}&
				   \multirow{2}{2cm}{$\centering\tfrac{\|\b{A}^* - \b{A}^{(T)}\|_{\rm F}}{\|\b{A}^*\|_{\rm F}}$}& \multirow{2}{2cm}{$\centering\tfrac{\|\b{X}^{*(T)} - \b{X}^{(T)}\|_{\rm F}}{\|\b{X}^{*(T)}\|_{\rm F}}$}& \multirow{2}{1.3cm}{\centering\scriptsize$\hspace{-4pt}T(\supp(\hat{\b{X}})?)$}\\
		    &&&&&&&&&&\\\hline
				\multirow{4}{0.8cm}{ \centering\[\centering \b{100}\]}
				&\textbf{\texttt{NOODL}} 
				&5.38e-11&2.38e-16&245~(Y)&7.04e-11&2.24e-16&257~(Y)&5.48e-11&5.14e-13&240~(Y)\\ 
				&\textbf{\texttt{Arora(b)}}
				&1.87e-06&1.14e-05&245~(N)&2.09e-03&1.41e-03&257~(N)&2.70e-03&2.41e-03&240~(N)\\ 
				&\textbf{\texttt{Arora(u)}}
				&6.78e-08&1.14e-05&245~(N)&8.94e-05&7.38e-05&257~(N)&1.72e-04&8.76e-05&240~(N))\\ 
				&\textbf{\texttt{Mairal}}
				&4.40e-03&2.00e-03&245~(N)&4.90e-03&6.87e-03&257~(N)&6.00e-03&5.10e-03&240~(N)\\ \hline
             	\multirow{4}{0.8cm}{ \centering\[\centering \b{300}\]}
             	&\textbf{\texttt{NOODL}} 
             	&5.72e-11&1.13e-12&61~(Y)&6.74e-11&5.44e-13&89~(Y)&9.10e-11&1.27e-12&168~(Y)\\ 
				&\textbf{\texttt{Arora(b)}}
				&2.13e-03&2.86e-03&61~(N)&5.90e-04&4.50e-04&89~(N)&1.00e-03&1.10e-03&168~(N)\\  
				&\textbf{\texttt{Arora(u)}}
				&2.04e-04&2.70e-04&61~(N)&3.82e-05&4.26e-05&89~(N)&1.04e-04&1.09e-04&168~(N)\\ 	
				&\textbf{\texttt{Mairal}}	
				&2.05e-01&2.28e-01&61~(N)&1.19e-02&1.09e-02&89~(N)&1.07e-02&8.40e-03&168~(N)\\ \hline
	            \multirow{4}{0.8cm}{ \centering\[\centering \b{500}\]}
				&\textbf{\texttt{NOODL}} 
				&5.49e-11&2.34e-16&50~(Y)&8.15e-11&1.25e-12&76~(Y)&9.27e-11&1.41e-12&160~(Y)\\  
				&\textbf{\texttt{Arora(b)}}
				&1.11e-04&1.34e-04&50~(N)&5.75e-04&5.60e-04&76~(N)&6.32e-04&2.71e-03&160~(N)\\  
				&\textbf{\texttt{Arora(u)}}
				&9.75e-06&1.50e-05&50~(N)&4.30e-05&4.73e-05&76~(N)&5.55e-05&2.28e-03&160~(N)\\ 	
				&\textbf{\texttt{Mairal}}
				&1.23e-01&1.10e-01&50~(N)&1.73e-02&1.20e-02&76~(N)&1.44e-02&5.99e-02&160~(N)\\ \hline
				\end{tabular}	
				}
			\par\bigskip
			   	\scalebox{0.73}{
			   	\begin{tabular}{P{0.8cm}|P{1.5cm}|P{1.5cm}|P{2cm}|P{1.4cm}?P{1.5cm}|P{2cm}|P{1.3cm}}
				\hline 
				 \multirow{3}{0.8cm}{\centering\textbf{$(J,K)$}}&\multirow{3}{1.5cm}{\centering \textbf{Method}}&\multicolumn{3}{c?}{$m=450$}&\multicolumn{3}{c}{$m=500$}\\\cline{3-8}
				  & & \multirow{2}{2cm}{$\centering\tfrac{\|\b{A}^* - \b{A}^{(T)}\|_{\rm F}}{\|\b{A}^*\|_{\rm F}}$}& \multirow{2}{2cm}{$\centering\tfrac{\|\b{X}^{*(T)} - \b{X}^{(T)}\|_{\rm F}}{\|\b{X}^{*(T)}\|_{\rm F}}$}& \multirow{2}{1.3cm}{\centering\scriptsize$\hspace{-4pt}T(\supp(\hat{\b{X}})?)$}
				  &\multirow{2}{2cm}{$\centering\tfrac{\|\b{A}^* - \b{A}^{(T)}\|_{\rm F}}{\|\b{A}^*\|_{\rm F}}$}& \multirow{2}{2cm}{$\centering\tfrac{\|\b{X}^{*(T)} - \b{X}^{(T)}\|_{\rm F}}{\|\b{X}^{*(T)}\|_{\rm F}}$}& \multirow{2}{1.3cm}{\centering\scriptsize$\hspace{-4pt}T(\supp(\hat{\b{X}}))$}\\
		    &&&&&&&\\\hline
				\multirow{4}{0.8cm}{ \centering\[\centering \b{100}\]}
				&\textbf{\texttt{NOODL}} 
				&7.82e-11&1.79e-12&257~(Y)&8.30e-11&6.39e-13&300~(Y)\\ 
				&\textbf{\texttt{Arora(b)}}
				&3.80e-03&3.20e-03&257~(N)&2.80e-03&3.06e-03&300~(N)\\ 
				&\textbf{\texttt{Arora(u)}}
				&3.06e-04&1.82e-04&257~(N)&2.52e-04&2.76e-04&300~(N)\\ 
				&\textbf{\texttt{Mairal}}
				&7.20e-03&6.90e-03&257~(N)&8.27e-03&8.07e-03&300~(N)\\ \hline
             	\multirow{4}{0.8cm}{ \centering\[\centering \b{300}\]}
             	&\textbf{\texttt{NOODL}} 
             	&9.43e-11&1.56e-12&201~(Y)&9.50e-11&1.63e-12&265~(Y)\\ 
				&\textbf{\texttt{Arora(b)}}
				&9.77e-04&1.04e-03&201~(N)&1.03e-03&9.36e-04&265~(N)\\  
				&\textbf{\texttt{Arora(u)}}
				&1.42e-04&1.68e-04&201~(N)&1.27e-04&1.23e-04&265~(N)\\ 	
				&\textbf{\texttt{Mairal}}	
				&1.47e-02&1.39e-02&201~(N)&9.40e-03&1.05e-02&265~(N)\\ \hline
	            \multirow{4}{0.8cm}{ \centering\[\centering \b{500}\]}
				&\textbf{\texttt{NOODL}} 
				&9.77e-11&1.60e-12&196~(Y)&9.72e-11&1.84e-12&264~(Y)\\  
				&\textbf{\texttt{Arora(b)}}
				&5.99e-04&5.30e-03&196~(N)&6.04e-04&6.37e-03&264~(N)\\  
				&\textbf{\texttt{Arora(u)}}
				&5.91e-05&5.30e-03&196~(N&8.08e-05&6.37e-03&264~(N)\\ 	
				&\textbf{\texttt{Mairal}}
				&3.22e-01&2.87e-01&196~(N)&2.46e-02&1.70e-01&264~(N)\\\hline
				\end{tabular}	
				}							
	\captionsetup{font=small}
	\vspace{0pt}
	\end{table*}

We now detail the specifics of the experiments and present additional results corresponding to section~\ref{sec:sims} for synthetic data experiments and real-world data experiments, respectively. 

\paragraph{Distributed Implementations:} Since the updates of \[\b{X}^{(r)(t)}\] columns are independent of each other, \texttt{TensorNOODL} is amenable for large-scale implementation in highly distributed settings. As a result, it is especially suitable for handling the tensor decomposition applications. Furthermore, the online nature of \texttt{TensorNOODL} allows the algorithm to continue to learn for its lifetime. 

\paragraph{Note on Initialization:} For synthetic data simulations, since the ground-truth factors are known, we can initialize the dictionary factor such that the requirements of Def.~\ref{def:del_kappa} are met.  In real-world data setting, the ground-truth is unknown and our initialization requirement can be met by existing algorithms, such as \cite{Arora15}. Consequently, in real-world experiments we use \cite{Arora15} to initialize the dictionary factor \[\b{A}^{(0)}\]. Here, we run the initialization algorithm once and communicate the estimate \[\b{A}^{(0)}\]  to each worker at the beginning of the distributed operation.

\vspace{-8pt}

\subsection{Synthetic Data Simulations}
\label{app:syn_res}
	\vspace{-3pt}
\subsubsection{Experimental Set-up}
	\vspace{-3pt}
\paragraph{Overview of Experiments:} As discussed in section~\ref{sec:sims}, we analyze the performance of the algorithm across different choices of tensor dimensions \[(J,K)\]  for a fixed \[n = 300\], its rank\[(m)\] and the sparsity of factors \[\b{B}^{*(t)}\] and \[\b{C}^{*(t)}\] controlled by parameters \[(\alpha,\beta)\], for recovery of the constituent factors using three Monte-Carlo runs. For each of these runs, we analyze the recovery performance across three choices of dimensions \[J=K = \{100,~300, ~500\}\], five choices of rank \[m = \{50, 150, 300, 450, 600\}\], and three choices of the sparsity parameters \[\alpha = \beta = \{0.005, 0.01, 0.05\}\]. The results corresponding to  \[\alpha = \beta = \{0.005, 0.01, 0.05\}\] are shown in Table~\ref{alpha_0_005}, \ref{alpha_0_01}, and \ref{alpha_0_05}, respectively.

\paragraph{Data Generation:} For each experiment we draw entries of the dictionary factor matrix \[\b{A}^* \in \mathbb{R}^{n \times m}\] from \[\c{N}(0,1)\], and normalize its columns to be unit-norm. To form \[\b{A}^{(0)}\] in accordance with \ref{assumption:close}, we perturb \[\b{A}^*\] with random Gaussian noise and normalized its columns, such that it is column-wise \[2/\log(n)\] away from \[\b{A}^*\] in \[\ell_2\] norm sense. To form the sparse factors \[\b{B}^{*(t)}\] and \[\b{C}^{*(t)}\], we assign their entries to the support independently with probability \[\alpha\] and \[\beta\], respectively, and then draw the values on the support from the Rademacher distribution\footnote{The corresponding code is available at \url{https://github.com/srambhatla/TensorNOODL} for reproducibility.}. 
%\addtolength{\textfloatsep}{-16pt}
%\addtolength{\abovecaptionskip}{-5pt}
%\addtolength{\belowcaptionskip}{-5pt}

\paragraph{Parameters Setting:} We set \texttt{TensorNOODL} specific IHT parameters \[\eta_x = 0.2\] and \[\tau = 0.1\] for all experiments. As recommended by our main result, the dictionary step-size parameter \[\eta_A\] is set proportional to \[m/k\]. Since \texttt{TensorNOODL}, \texttt{Arora(b)}, and \texttt{Arora(u)} all rely on an approximate gradient descent strategy for dictionary update, we use the same step-size \[\eta_A\] for a fair comparison depending upon the choice of rank \[m\], and probabilities \[(\alpha, \beta)\] as per \ref{assumption:step dict}; Table~\ref{stepsize} lists the step-size choices. Here, \texttt{Mairal} does not employ such a parameter. 
\renewcommand{\arraystretch}{1}
	\begin{table*}[!t]
		\caption{Tensor factorization results $\alpha,\beta=0.01$ averaged across $3$ trials. Here, \[T(\supp(\hat{\b{X}}^{(T)})?)\] field shows the number of iterations \[T\] to reach the target tolerance, while the categorical field, \[\supp(\hat{\b{X}}^{(T)})\] indicates if the support of the recovered \[\hat{\b{X}}^{(T)}\] matches that of \[\b{X}^{*(T)}\] (Y) or not (N).}
	\label{alpha_0_01}
	\captionsetup{justification=centering}
	\captionsetup{font=footnotesize}
	\centering
	\newcolumntype{?}{!{\vrule width 1.5pt}}
	\par\bigskip
\scalebox{0.73}{
			   	\begin{tabular}{P{0.8cm}|P{1.5cm}|P{1.5cm}|P{2cm}|P{1.4cm}?P{1.5cm}|P{2cm}|P{1.3cm}?P{1.5cm}|P{2cm}|P{1.3cm}}
				\hline 
				 \multirow{3}{0.8cm}{\centering\textbf{$(J,K)$}}&\multirow{3}{1.5cm}{\centering \textbf{Method}}&\multicolumn{3}{c?}{$m=50$}&\multicolumn{3}{c?}{$m=150$}&\multicolumn{3}{c}{$m=300$}\\\cline{3-11}
				  & & \multirow{2}{2cm}{$\centering\tfrac{\|\b{A}^* - \b{A}^{(T)}\|_{\rm F}}{\|\b{A}^*\|_{\rm F}}$}& \multirow{2}{2cm}{$\centering\tfrac{\|\b{X}^{*(T)} - \b{X}^{(T)}\|_{\rm F}}{\|\b{X}^{*(T)}\|_{\rm F}}$}& \multirow{2}{1.3cm}{\centering\scriptsize$\hspace{-4pt}T(\supp(\hat{\b{X}})?)$}
				  &\multirow{2}{2cm}{$\centering\tfrac{\|\b{A}^* - \b{A}^{(T)}\|_{\rm F}}{\|\b{A}^*\|_{\rm F}}$}& \multirow{2}{2cm}{$\centering\tfrac{\|\b{X}^{*(T)} - \b{X}^{(T)}\|_{\rm F}}{\|\b{X}^{*(T)}\|_{\rm F}}$}& \multirow{2}{1.3cm}{\centering\scriptsize$\hspace{-4pt}T(\supp(\hat{\b{X}})?)$}&
				   \multirow{2}{2cm}{$\centering\tfrac{\|\b{A}^* - \b{A}^{(T)}\|_{\rm F}}{\|\b{A}^*\|_{\rm F}}$}& \multirow{2}{2cm}{$\centering\tfrac{\|\b{X}^{*(T)} - \b{X}^{(T)}\|_{\rm F}}{\|\b{X}^{*(T)}\|_{\rm F}}$}& \multirow{2}{1.3cm}{\centering\scriptsize$\hspace{-4pt}T(\supp(\hat{\b{X}})?)$}\\
		    &&&&&&&&&&\\\hline
					\multirow{4}{0.8cm}{ \centering\[\centering \b{100}\]}
					 &\textbf{\texttt{NOODL}} 
					&5.50e-11&5.66e-13&91~(Y)&7.59e-11&5.28e-13&112~(Y)&4.34e-11&1.62e-12&190~(Y)\\ 
					&\textbf{\texttt{Arora(b)}}
					&3.93e-03&5.80e-03&91~(N)&2.61e-03&1.58e-03&112~(N)&2.70e-03&3.00e-03&190~(N)\\  
					&\textbf{\texttt{Arora(u)}}
					&4.35e-04&6.77e-04&91~(N)&6.87e-04&1.05e-04&112~(N)&2.98e-04&3.04e-04&190~(N)\\ 	
					&\textbf{\texttt{Mairal}}
					&4.03e-02&1.26e-02&91~(N)&1.34e-02&1.25e-02&112~(N)&1.18e-02&1.25e-02&190~(N)\\	 \hline
	             	\multirow{4}{0.8cm}{ \centering\[\centering \b{300}\]}
					&\textbf{\texttt{NOODL}} 
					&6.78e-11&5.75e-13&51~(Y)&6.35e-11&1.54e-12&76~(Y)&8.64e-11&2.06e-12&158~(Y)\\ 
					&\textbf{\texttt{Arora(b)}}
					&4.08e-04&4.76e-04&51~(N)&1.03e-03&1.08e-03&76~(N)&1.04e-03&1.17e-02&158~(N)\\ 
					&\textbf{\texttt{Arora(u)}}
					&1.99e-05&1.46e-05&51~(N)&1.03e-04&9.59e-05&76~(N)&2.17e-04&1.17e-02&158~(N)\\	
					&\textbf{\texttt{Mairal}}	
					&1.64e-01&1.63e-01&51~(N)&2.61e-02&2.64e-02&76~(N)&2.81e-02&1.58e-01&158~(N)\\ \hline
		            \multirow{4}{0.8cm}{ \centering\[\centering \b{500}\]}
					&\textbf{\texttt{NOODL}} 
					 &6.92e-11&8.78e-13&46~(Y)&8.77e-11&1.77e-12&77~(Y)&9.35e-11&2.12e-12&156~(Y)\\ 
					&\textbf{\texttt{Arora(b)}}
					&3.48e-04&3.28e-04&46~(N)&5.42e-04&6.40e-03&77~(N)&5.69e-04&2.41e-03&156~(N)\\ 
					&\textbf{\texttt{Arora(u)}}
					&2.56e-05&3.70e-05&46~(N)&4.81e-05&6.40e-03&77~(N)&1.08e-04&9.30e-03&156~((N)\\	
					&\textbf{\texttt{Mairal}}	
					&1.56e-01&1.53e-01&46~(N)&5.28e-02&1.30e-01&77~(N)&2.53e-02&1.57e-01&156~(N)\\ \hline
					\end{tabular}
					}	
						\par\bigskip
		\scalebox{0.73}{
				   	\begin{tabular}{P{0.8cm}|P{1.5cm}|P{1.5cm}|P{2cm}|P{1.4cm}?P{1.5cm}|P{2cm}|P{1.3cm}}
					\hline 
					 \multirow{3}{0.8cm}{\centering\textbf{$(J,K)$}}&\multirow{3}{1.5cm}{\centering \textbf{Method}}&\multicolumn{3}{c?}{$m=450$}&\multicolumn{3}{c}{$m=500$}\\\cline{3-8}
					  & & \multirow{2}{2cm}{$\centering\tfrac{\|\b{A}^* - \b{A}^{(T)}\|_{\rm F}}{\|\b{A}^*\|_{\rm F}}$}& \multirow{2}{2cm}{$\centering\tfrac{\|\b{X}^{*(T)} - \b{X}^{(T)}\|_{\rm F}}{\|\b{X}^{*(T)}\|_{\rm F}}$}& \multirow{2}{1.3cm}{\centering\scriptsize$\hspace{-4pt}T(\supp(\hat{\b{X}})?)$}
					  &\multirow{2}{2cm}{$\centering\tfrac{\|\b{A}^* - \b{A}^{(T)}\|_{\rm F}}{\|\b{A}^*\|_{\rm F}}$}& \multirow{2}{2cm}{$\centering\tfrac{\|\b{X}^{*(T)} - \b{X}^{(T)}\|_{\rm F}}{\|\b{X}^{*(T)}\|_{\rm F}}$}& \multirow{2}{1.3cm}{\centering\scriptsize$\hspace{-4pt}T(\supp(\hat{\b{X}})?)$}\\
			    &&&&&&&\\\hline
					\multirow{4}{0.8cm}{ \centering\[\centering \b{100}\]}
					 &\textbf{\texttt{NOODL}} 
					&9.48e-11&1.78e-12&211~(Y)&7.27e-11&1.94e-12&279~(Y)\\ 
					&\textbf{\texttt{Arora(b)}}
					&3.30e-03&4.00e-03&211~(N)&3.40e-03&3.37e-03&279~(N)\\  
					&\textbf{\texttt{Arora(u)}}
					&8.55e-04&1.27e-03&211~(N)&6.83e-04&6.49e-04&279~(N)\\ 	
					&\textbf{\texttt{Mairal}}
					&8.00e-03&6.60e-03&211~(N)&8.77e-03&9.93e-03&279~(N)\\	 \hline
	             	\multirow{4}{0.8cm}{ \centering\[\centering \b{300}\]}
					&\textbf{\texttt{NOODL}} 
					&9.43e-11&2.92e-12&192~(Y)&9.33e-11&2.54e-12&252~(Y)\\ 
					&\textbf{\texttt{Arora(b)}}
					&1.00e-03&1.25e-02&192~(N)&1.13e-03&1.54e-02&252~(N)\\ 
					&\textbf{\texttt{Arora(u)}}
					&2.22e-04&1.25e-02&192~(N)&2.69e-04&1.54e-02&252~(N)\\	
					&\textbf{\texttt{Mairal}}	
					&1.39e-01&2.03e-01&192~(N)&1.92e-02&1.83e-01&252~(N)\\ \hline
		            \multirow{4}{0.8cm}{ \centering\[\centering \b{500}\]}
					&\textbf{\texttt{NOODL}} 
					&9.60e-11&2.41e-12&186~(Y)&9.82e-11&2.66e-12&249~(Y)\\ 
					&\textbf{\texttt{Arora(b)}}
					&6.49e-04&1.20e-02&186~(N)&6.55e-04&1.42e-02&249~(N)\\ 
					&\textbf{\texttt{Arora(u)}}
					&1.39e-04&1.20e-02&186~(N)&1.55e-04&1.42e-02&249~(N)\\	
					&\textbf{\texttt{Mairal}}	
					&6.38e-02&1.54e-01&186~(N)&1.74e-02&1.79e-01&249~(N)\\ \hline
					\end{tabular}
					}
		\captionsetup{font=small}
		\vspace{-0pt}
		\end{table*}

\renewcommand{\arraystretch}{1}

\begin{table*}[t]
		\caption{Choosing the step-size (\[\eta_A\]) for the dictionary update step. We use the same dictionary update step-size parameter (\[\eta_A\]) for \texttt{TensorNOODL}, \texttt{Arora(b)}, and \texttt{Arora(u)} depending upon the choice of rank \[m\], and probabilities \[(\alpha, \beta)\], as per\ref{assumption:step dict}. }
	\label{stepsize}
	\captionsetup{justification=centering}
	\captionsetup{font=footnotesize}
	\centering
	\vspace{10pt}
			   	\scalebox{0.8}{
					\begin{tabular}{P{3cm}|P{3cm}|P{6cm}}
				\hline
					\textbf{Rank} (\[m\])& \textbf{Step-size (\[\eta_A\])}& \textbf{Notes}\\ \hline
						50 & 20 & For \[(\alpha, \beta) = 0.005\], we use \[\eta_A = 5\]\\ \hline
						150 & 40 & --\\ \hline
						300 & 40 & --\\ \hline
						450 & 50 & --\\ \hline
						600 & 50 & --\\ \hline
					\end{tabular}
					}
		\captionsetup{font=small}
	%	\vspace{-12pt}
		\end{table*}
\captionsetup{justification=justified}
 
%\vspace{-10pt}
\paragraph{Evaluation Metrics:}
We run all algorithms till one of them achieves target tolerance (error in the factor \[\b{A}\], \[\epsilon_T\]) of \[10^{-10}\], and report the number of iterations \[T\] for each experiment. Note that, in all cases \texttt{TensorNOODL} achieves the tolerance first, and in some cases with the algorithms considered in the analysis. Next, since recovery of \[\b{A}^*\] and \[\b{X}^{*(t)}\] is vital for the success of the tensor factorization task, we report the relative Frobenius error for each of these matrices, i.e., for a recovered matrix \[\hat{\b{M}}\], we report \[{\|\hat{\b{M}} - \b{M}^*\|_{\Fr}}/{\|\b{M}^*\|_{\Fr}}\]. In addition, since the dictionary learning task focuses on recovering the sparse matrix \[\b{X}^{*(t)}\], it is agnostic to the transposed Khatri-Rao structure \[\b{S}^{*(t)}\]. As a result, for recovering the sparse factors \[\b{B}^{*(t)}\] and \[\b{C}^{*(t)}\] is crucial for exact support recovery of \[\b{X}^{*(t)}\]. Therefore, we report if the support has been exactly recovered or not. 

\subsubsection{Other Considerations}

\paragraph{Reproducible Results:} The code employed is made available as part of the supplementary material. We fix the random seeds (to \[42, 26,\] and \[91\]) for each Monte Carlo run to ensure reproducibility of the results shown in this work. The experiments were run on a HP Haswell Linux Cluster. The processing of data samples for the sparse coefficients (\[\hat{\b{X}}^{*(t)}\]) was split across \[20\] workers (cores), allocated a total of \[200\] GB RAM. For \textbf{\texttt{Arora(b)}}, \textbf{\texttt{Arora(u)}}, and \textbf{\texttt{Mairal}},  the coefficient recovery was switched between Fast Iterative Shrinkage-Thresholding Algorithm (FISTA) \citep{Beck09}, an accelerated proximal gradient descent algorithm, or  a stochastic-version of Iterative Shrinkage-Thresholding Algorithm (ISTA) \citep{Chambolle1998, Daubechies2004} depending upon the size of the data samples available for learning (see the discussion of the coefficient update step below); see also \cite{Beck09} for details. 

\renewcommand{\arraystretch}{1}
	\begin{table*}[t]
		\caption{Tensor factorization results $\alpha,\beta=0.05$ averaged across $3$ trials.  Here, \[T(\supp(\hat{\b{X}}^{(T)})?)\] field shows the number of iterations \[T\] to reach the target tolerance, while the categorical field, \[\supp(\hat{\b{X}}^{(T)})\] indicates if the support of the recovered \[\hat{\b{X}}^{(T)}\] matches that of \[\b{X}^{*(T)}\] (Y) or not (N).}
	\label{alpha_0_05}
	\captionsetup{justification=centering}
	\captionsetup{font=footnotesize}
	\centering
	\newcolumntype{?}{!{\vrule width 1.5pt}}
	\par\bigskip	
\scalebox{0.73}{
			   	\begin{tabular}{P{0.8cm}|P{1.5cm}|P{1.5cm}|P{2cm}|P{1.4cm}?P{1.5cm}|P{2cm}|P{1.3cm}?P{1.5cm}|P{2cm}|P{1.3cm}}
				\hline 
				 \multirow{3}{0.8cm}{\centering\textbf{$(J,K)$}}&\multirow{3}{1.5cm}{\centering \textbf{Method}}&\multicolumn{3}{c?}{$m=50$}&\multicolumn{3}{c?}{$m=150$}&\multicolumn{3}{c}{$m=300$}\\\cline{3-11}
				  & & \multirow{2}{2cm}{$\centering\tfrac{\|\b{A}^* - \b{A}^{(T)}\|_{\rm F}}{\|\b{A}^*\|_{\rm F}}$}& \multirow{2}{2cm}{$\centering\tfrac{\|\b{X}^{*(T)} - \b{X}^{(T)}\|_{\rm F}}{\|\b{X}^{*(T)}\|_{\rm F}}$}& \multirow{2}{1.3cm}{\centering\scriptsize$\hspace{-4pt}T(\supp(\hat{\b{X}})?)$}
				  &\multirow{2}{2cm}{$\centering\tfrac{\|\b{A}^* - \b{A}^{(T)}\|_{\rm F}}{\|\b{A}^*\|_{\rm F}}$}& \multirow{2}{2cm}{$\centering\tfrac{\|\b{X}^{*(T)} - \b{X}^{(T)}\|_{\rm F}}{\|\b{X}^{*(T)}\|_{\rm F}}$}& \multirow{2}{1.3cm}{\centering\scriptsize$\hspace{-4pt}T(\supp(\hat{\b{X}})?)$}&
				   \multirow{2}{2cm}{$\centering\tfrac{\|\b{A}^* - \b{A}^{(T)}\|_{\rm F}}{\|\b{A}^*\|_{\rm F}}$}& \multirow{2}{2cm}{$\centering\tfrac{\|\b{X}^{*(T)} - \b{X}^{(T)}\|_{\rm F}}{\|\b{X}^{*(T)}\|_{\rm F}}$}& \multirow{2}{1.3cm}{\centering\scriptsize$\hspace{-4pt}T(\supp(\hat{\b{X}})?)$}\\
		    &&&&&&&&&&\\\hline
					\multirow{4}{0.8cm}{ \centering\[\centering \b{100}\]}
					 &\textbf{\texttt{NOODL}} 
					&8.03e-11&3.17e-12&46~(Y)&7.71e-11&4.92e-12&63~(Y)&9.66e-11&6.01e-12&110~(Y)\\ 
					&\textbf{\texttt{Arora(b)}}
					&2.90e-03&3.00e-03&46~(N)&4.60e-03&3.39e-02&63~(N)&5.50e-03&4.89e-02&110~(N)\\  
					&\textbf{\texttt{Arora(u)}}
					&8.97e-04&8.48e-04&46~(N)&1.90e-03&3.40e-02&63~(N)&2.80e-03&4.90e-02&110~(N)\\ 	
					&\textbf{\texttt{Mairal}}
					&1.57e-01&1.67e-01&46~(N)&3.63e-02&1.54e-01&63~(N)&2.32e-02&1.99e-01&110~(N)\\	 \hline
	             	\multirow{4}{0.8cm}{ \centering\[\centering \b{300}\]}
					&\textbf{\texttt{NOODL}} 
					&6.51e-11&3.27e-12&42~(Y)&9.05e-11&5.61e-12&60~(Y)&9.10e-11&7.01e-12&107~(Y)\\ 
					&\textbf{\texttt{Arora(b)}}
					&1.40e-03&1.95e-02&42~(N)&2.50e-03&3.55e-02&60~(N)&3.20e-03&5.04e-02&107~(N)\\ 
					&\textbf{\texttt{Arora(u)}}
					&2.48e-04&1.95e-02&42~(N)&6.35e-04&3.56e-02&60~(N)&9.48e-04&5.05e-02&107~(N)\\	
					&\textbf{\texttt{Mairal}}	
					&6.24e-02&1.11e-01&42~(N)&3.05e-02&1.59e-01&60(N)&1.91e-02&2.09e-01&107~(N)\\ \hline
		            \multirow{4}{0.8cm}{ \centering\[\centering \b{500}\]}
					&\textbf{\texttt{NOODL}} 
					&7.72e-11&3.86e-12&42~(Y)&8.44e-11&5.63e-12&59~(Y)&9.64e-11&7.34e-12&106~(Y)\\ 
					&\textbf{\texttt{Arora(b)}}
					&1.30e-03&2.02e-02&42~(N)&2.10e-03&3.55e-02&59~(N)&2.80e-03&5.03e-02&106~(N)\\ 
					&\textbf{\texttt{Arora(u)}}
					&1.39e-04&2.02e-02&42~(N)&3.82e-04&3.56e-02&59~(N)&5.66e-04&5.05e-02&106~(N)\\	
					&\textbf{\texttt{Mairal}}	
					&6.12e-02&1.10e-01&42~(N)&2.93e-02&1.58e-01&59~(N)&1.80e-02&2.11e-01&106~(N)\\ \hline
					\end{tabular}
					}
			\par\bigskip			
		\scalebox{0.73}{
				   	\begin{tabular}{P{0.8cm}|P{1.5cm}|P{1.5cm}|P{2cm}|P{1.4cm}?P{1.5cm}|P{2cm}|P{1.3cm}}
					\hline 
					 \multirow{3}{0.8cm}{\centering\textbf{$(J,K)$}}&\multirow{3}{1.5cm}{\centering \textbf{Method}}&\multicolumn{3}{c?}{$m=450$}&\multicolumn{3}{c}{$m=500$}\\\cline{3-8}
					  & & \multirow{2}{2cm}{$\centering\tfrac{\|\b{A}^* - \b{A}^{(T)}\|_{\rm F}}{\|\b{A}^*\|_{\rm F}}$}& \multirow{2}{2cm}{$\centering\tfrac{\|\b{X}^{*(T)} - \b{X}^{(T)}\|_{\rm F}}{\|\b{X}^{*(T)}\|_{\rm F}}$}& \multirow{2}{1.3cm}{\centering\scriptsize$\hspace{-4pt}T(\supp(\hat{\b{X}})?)$}
					  &\multirow{2}{2cm}{$\centering\tfrac{\|\b{A}^* - \b{A}^{(T)}\|_{\rm F}}{\|\b{A}^*\|_{\rm F}}$}& \multirow{2}{2cm}{$\centering\tfrac{\|\b{X}^{*(T)} - \b{X}^{(T)}\|_{\rm F}}{\|\b{X}^{*(T)}\|_{\rm F}}$}& \multirow{2}{1.3cm}{\centering\scriptsize$\hspace{-4pt}T(\supp(\hat{\b{X}})?)$}\\
			    &&&&&&&\\\hline
					\multirow{4}{0.8cm}{ \centering\[\centering \b{100}\]}
					 &\textbf{\texttt{NOODL}} 
					&8.92e-11&7.29e-12&115~(Y)&8.71e-11&1.06e-11&131~(Y)\\ 
					&\textbf{\texttt{Arora(b)}}
					&7.50e-03&6.17e-02&115~(N)&9.16e-03&7.36e-02&131~(N)\\  
					&\textbf{\texttt{Arora(u)}}
					&4.40e-03&6.19e-02&115~(N)&5.70e-03&7.40e-02&131~(N)\\ 	
					&\textbf{\texttt{Mairal}}
					&8.79e-02&2.27e-01&115~(N)&2.81e-02&2.56e-01&131~(N)\\	 \hline
	             	\multirow{4}{0.8cm}{ \centering\[\centering \b{300}\]}
					&\textbf{\texttt{NOODL}} 
					&9.20e-11&8.41-12&110~(Y)&8.49e-11&9.03e-12&128~(Y)\\ 
					&\textbf{\texttt{Arora(b)}}
					&4.00e-03&6.16e-02&110~(N)&4.90e-03&7.39e-02&128~(N)\\ 
					&\textbf{\texttt{Arora(u)}}
					&1.40e-03&6.18e-02&110~(N)&1.83e-03&7.42e-02&128~(N)\\	
					&\textbf{\texttt{Mairal}}	
					&4.85e-02&2.19e-01&110~(N)&2.32e-02&2.63e-01&128~(N)\\ \hline
		            \multirow{4}{0.8cm}{ \centering\[\centering \b{500}\]}
					&\textbf{\texttt{NOODL}} 
					&8.95e-11&8.21e-12&109~(Y)&9.06e-11&9.29e-12&127~(Y)\\ 
					&\textbf{\texttt{Arora(b)}}
					&3.60e-03&6.21e-02&109~(N)&4.40e-03&7.40e-02&127~(N)\\ 
					&\textbf{\texttt{Arora(u)}}
					&8.54e-04&6.23e-02&109~(N)&1.10e-03&7.44e-02&127~(N)\\	
					&\textbf{\texttt{Mairal}}	
					&4.62e-02&2.20e-01&109~(N)&4.05e-02&2.56e-01&127~(N)\\ \hline
					\end{tabular}
					}						
		\captionsetup{font=small}
		\vspace{-0pt}
		\end{table*}
\captionsetup{justification=justified}
\setlength\extrarowheight{5pt}

\paragraph{Sparse Factor Recovery Considerations:}
In \cite{Arora15}, the authors present two algorithms -- a simple algorithm with a sample complexity of \[\tilde{\Omega}(ms)\] which incurs an estimation bias (\texttt{Arora(b)}), and a more involved variant for unbiased estimation of the dictionary whose sample complexity was not established \texttt{Arora(u)}. However, these algorithms do not provide guarantees on, or recover the sparse coefficients. As a result, we need to adopt an additional \[\ell_1\] minimization based coefficient recovery step. Further, the algorithm proposed by \cite{Mairal09} can be viewed as a variant of regularized alternating least squares algorithm which employs \[\ell_1\] regularization for the recovery of the transposed Khatri-Rao structured matrix.

To form the coefficient estimates for \texttt{Arora(b)}, \texttt{Arora(u)}, and {\texttt{Mairal} `\texttt{09}} we solve the Lasso \citep{Tibshirani1996} program using a stochastic-version of Iterative Shrinkage-Thresholding Algorithm (ISTA) \citep{Chambolle1998, Daubechies2004}  (or Fast Iterative Shrinkage-Thresholding Algorithm (FISTA) \citep{Beck09} if \[p\] is small) and report the best estimate (in terms of relative Frobenius error) across $10$ values of the regularization parameter. The stochastic projected gradient descent is necessary to make coefficient recovery tractable since size of \[\b{X}^{*(t)}\] grows quickly with \[(\alpha, \beta)\]. For these algorithms, coefficient estimation step the slowest step since it has to scan through different values of the regularization parameters to arrive at an estimate. In contrast, \texttt{TensorNOODL} does not require such an expensive tuning procedure, while providing recovery guarantees on the recovered coefficients. 

Note that in practice ISTA and FISTA can be parallelized as well, but tuning of the regularization parameters still involves (an expensive) grid search. Arguably even if each step of these algorithms (ISTA and FISTA) take the same amount of time as that of \texttt{TensorNOODL}, the search over, say \[10\], values of the regularization parameters will still be take \[10\] times the time. As a result, \texttt{TensorNOODL} is an attractive choice as it does not involve an expensive tuning procedure. 

\paragraph{Additional Discussion:}  Table~\ref{alpha_0_005}, \ref{alpha_0_01}, and \ref{alpha_0_05} show the results of the analysis averaged across the three Monte Carlo runs, for \[\alpha=\beta= \{0.005, 0.01, 0.05\}\], respectively. We note that for every choice of \[(J,K)\], \[m\], and \[(\alpha, \beta)\], \texttt{TensorNOODL} is orders of magnitude superior to related techniques. In addition, it also recovers the support correctly in all of the cases, ensuring that the sparse factors can be recovered correctly. Specifically,  the sparse factors \[\b{B}^{*(t)}\] and \[\b{C}^{*(t)}\] can be recovered (upto permutation and scaling) via Alg.~\ref{algo:KRP_algo}.

\vspace{-5pt}
\subsection{Real-world Data Simulations}

\subsubsection{Analysis of the Enron Dataset}\label{app:enron}
\paragraph{Enron Email Dataset:} Sparsity-regularized ALS-based tensor factorization techniques, albeit possessing limited convergence guarantees, have been a popular choice to analyze the Enron Email Dataset ($184\times184\times44$) \cite{Fu2015, Bader2006}. We now use \texttt{TensorNOODL} to analyze the email activity of $184$ Enron employees over $44$ weeks (Nov. `98 --Jan. '02) during the period before and after the financial irregularities were uncovered.
	
The Enron Email Dataset ($184\times184\times 44$) consists of email exchanges between $184$ employees over $44$ weeks (Nov. `98 --Jan. '02) which includes the period before and after the financial irregularities were uncovered. In general, every person in an organization (like Enron) communicates with only a subset of employees, as a result the tensor of email activity (Employees vs. Employees vs. Time) naturally has the model analyzed in this work. Moreover, as pointed out by  \cite{Diesner2005} ``\emph{\dots in   $2000$   Enron   had   a   segmented  culture  with  directives  being  sent  from  on-high  and  sporadic  feedback}''. Meaning that different units within the organization exhibited clustered communication structure. This motivates us to analyze the dataset for the presence charateristic ways of communications between different business units. 

We run \texttt{TensorNOODL} in batch setting here, this is to showcase that in practice \texttt{TensorNOODL} also works in batch settings, and also to overcome the limited size of the Enron Dataset. 

\paragraph{Data Preparation and Parameters:}
For \texttt{TensorNOODL} and \texttt{Mairal `09}, we use the initialization algorithm of \cite{Arora15}, which yielded $4$ dictionary elements. Following this, we use these techniques in batch setting to simultaneously identify email activity patterns and cluster employees. We also compare our results to \cite{Fu2015}, which just aims to cluster the employees by imposing sparsity constraint on one of the factors, and does not learn the patterns. As opposed to \cite{Fu2015}, \texttt{TensorNOODL}  did not require us to guess the number of dictionary elements to be used. We use Alg.~\ref{algo:KRP_algo} to identify the employees corresponding to email activity patterns from the recovered sparse factors. As in case of \cite{Fu2015}, we transform each non-zero element \[\underline{\b{Z}}(i,j,k)^{(t)}\] of the dataset as follows to compress its dynamic range,
\begin{align*}
\underline{\b{Z}}(i,j,k)^{(t)}= \log_2(\underline{\b{Z}}(i,j,k)) + 1.
\end{align*}
 We also scale all elements by the largest element magnitude and subtract the mean (over the temporal aspect) from the non-zero fibers. We initialization the dictionary using the algorithm presented in \cite{Arora15} for \texttt{TensorNOODL} and \texttt{Mairal `09}, resulting in $4$ dictionary elements. As with synthetic data experiments, we set \[\eta_x = 0.2\], \[\tau = 0.1\] and \[C=1\]. We set the dictionary update step-size \[\eta_A = 10\], and run \texttt{TensorNOODL} in batch setting for \[100\] iterations. We recover the sparse factors \[\b{B}^{*(t)}\] and \[\b{C}^{*(t)}\] using our untangling Alg.~	\ref{algo:KRP_algo}.  To compile the results, we ignore the entries with magnitude smaller than $5\%$ of the largest entry in that sparse factor column.

\paragraph{Evaluation Specifics:}
As in \cite{Fu2015}, we use cluster purity (\texttt{False Positives/Cluster Size}) as the measure of the clustering performance. To this end, we also compare our results with \cite{Fu2015}. Note that \cite{Fu2015} solves a regularized least squares-based formulation for low-rank non-negative tensor factorization, wherein one of factor is sparse (corresponds to employees) and the others have controlled Frobenius norms. Here, the non-zero entries of the sparse factor gives insights into the employees who exhibit similar behaviour. Unlike \texttt{TensorNOODL} and \texttt{Mairal `09}, this procedure however does not recover the email patterns of interest.

\paragraph{Discussion:}
The results of the decomposition are shown in Fig.~\ref{fig:enron}. The Enron organizational structure has four main units, namely, `Legal', `Traders', `Executives', and `Pipeline', which coincides with the number of dictionary elements recovered by \texttt{TensorNOODL}. Specifically, as opposed to \cite{Fu2015}, which take the number of clusters to be found as an input, \texttt{TensorNOODL} leverages the model selection performed by initialization algorithms. Furthermore, along with recovering the email activity patterns, \texttt{TensorNOODL} is also superior in terms of the clustering purity as compared to other techniques as inferred from the False Positives to Cluster-size ratio (Fig.~\ref{fig:enron}). The email activity patterns show how different group activities changed as time unfolded. In line with \citet{Diesner2005}, we observe that during the crisis the employees of different divisions indeed exhibited cliquish behavior. These results illustrate that our model (and algorithm) can be used to study organizational behavior via their communication activity. Note that here we use \texttt{TensorNOODL} in the batch setting, i.e., we reuse samples. This shows that empirically our algorithm can be used in the batch setting also, although our analysis applies to the online setting. We leave the analysis of the batch setting to future work. 

\vspace{-5pt}
\subsubsection{Analysis of the NBA Dataset}\label{app:nba}

\begin{figure}[!t]
\centering
    \resizebox{\textwidth}{!}{\begin{tabular}{p{0.01cm}ccp{0.01cm}cc}
    & {\large Shot Patterns}& {\large Clustering of Teams}&& {\large Shot Patterns}& {\large Clustering of Teams} \vspace{-3pt}\\
     &{\small (\[{\b{A}}^{(T)}\] Elements)} \hspace{-5pt}&  {\small (\[\hat{\b{B}}\] Elements)}&&{\small (\[{\b{A}}^{(T)}\] Elements)} \hspace{-5pt}&  {\small (\[\hat{\b{B}}\] Elements)}\\
    \rotatebox{90}{\hspace{10pt}\parbox{2cm}{\centering Element 1}}\hspace{-10pt}&
    \includegraphics[width=0.2\textwidth]{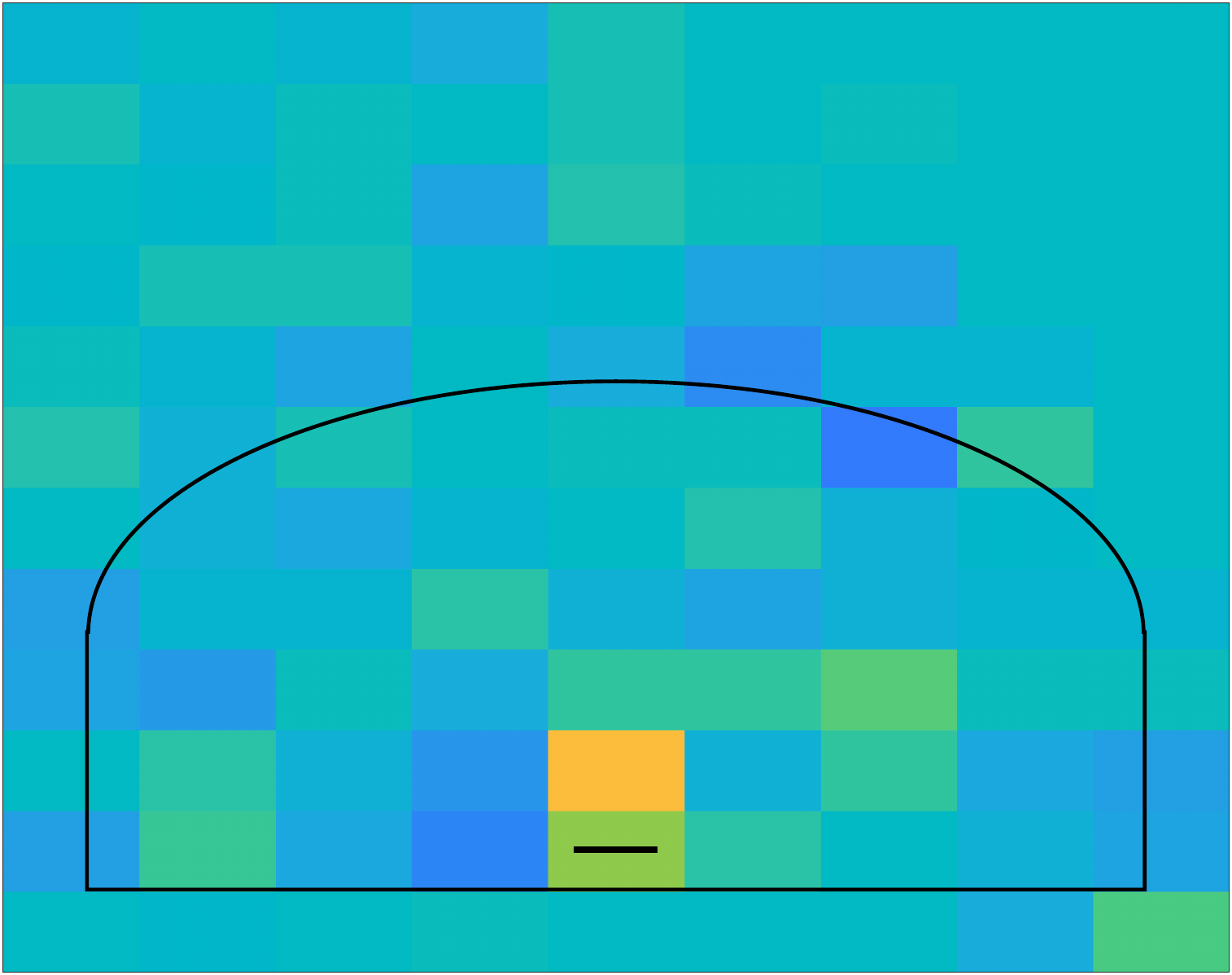}\hspace{2pt}
    \includegraphics[width=0.028\textwidth]{colormap7-crop-crop.pdf}& \includegraphics[width=0.34\textwidth]{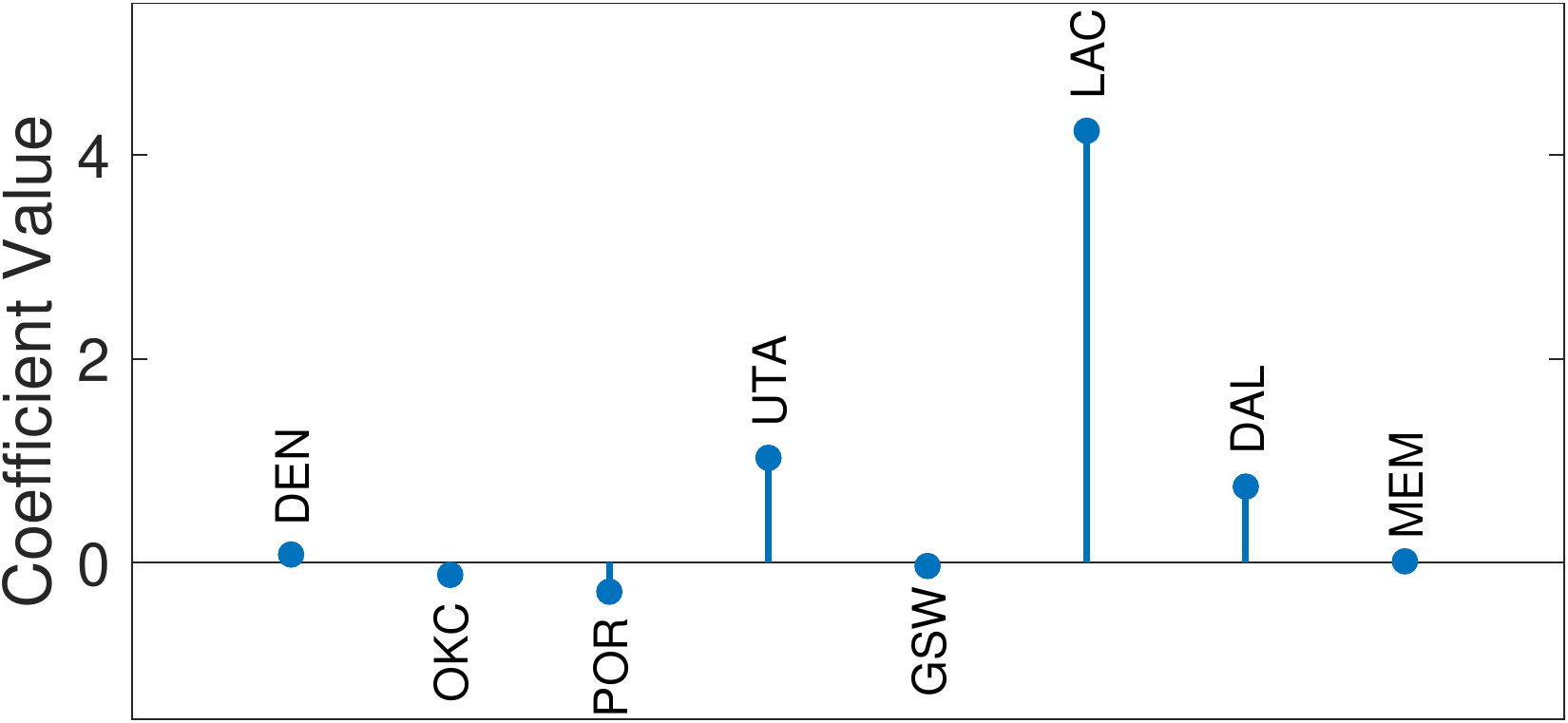} \hspace{5pt}&
       \rotatebox{90}{\hspace{10pt}\parbox{2cm}{\centering Element 2}}\hspace{-5pt}&
      \includegraphics[width=0.2\textwidth]{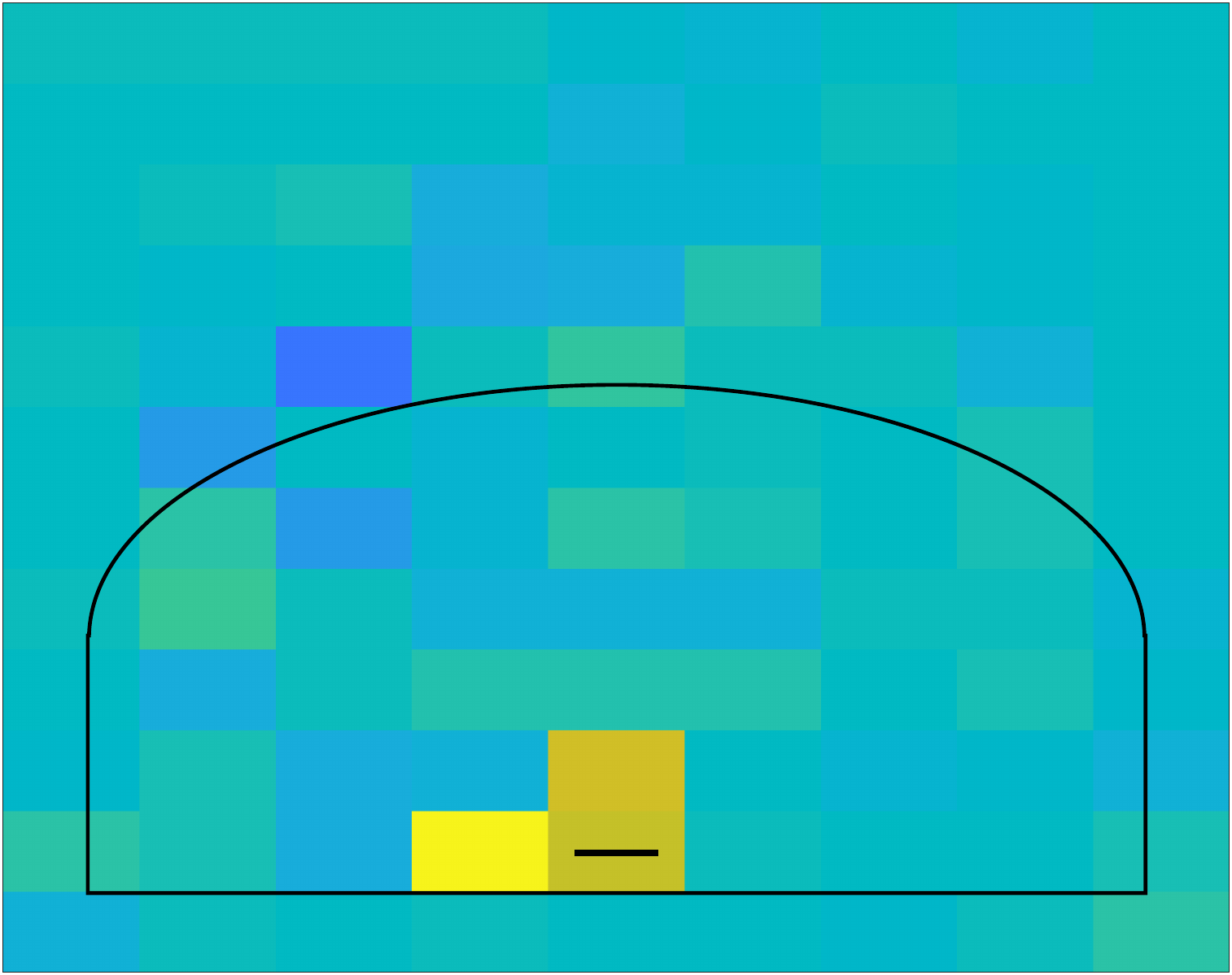}\hspace{2pt}
       \includegraphics[width=0.028\textwidth]{colormap7-crop-crop.pdf}&  \includegraphics[width=0.34\textwidth]{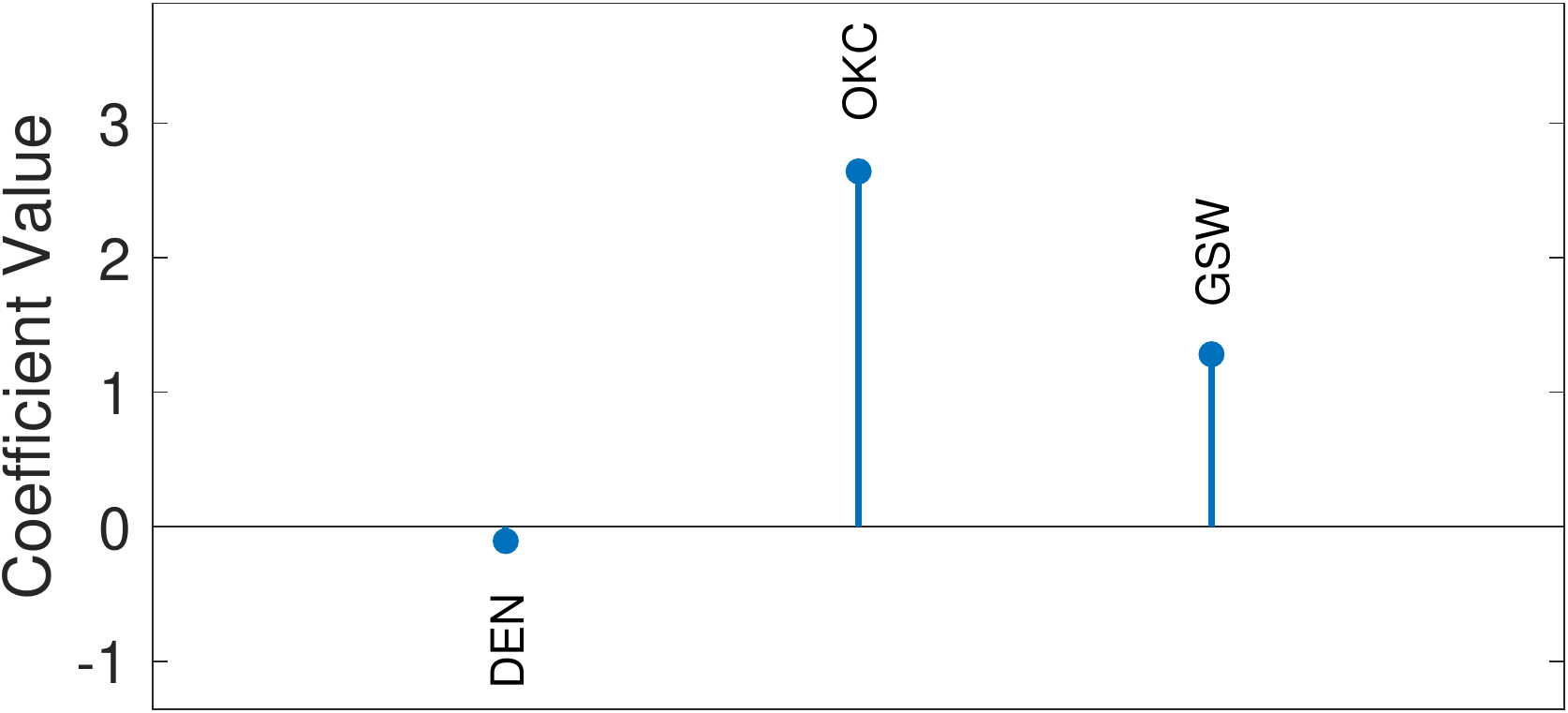}\\
          & (a-i) & (a-ii)&   & (b-i) & (b-ii)\\
       \rotatebox{90}{\hspace{10pt}\parbox{2cm}{\centering Element 3}}\hspace{-5pt}& 
      \includegraphics[width=0.2\textwidth]{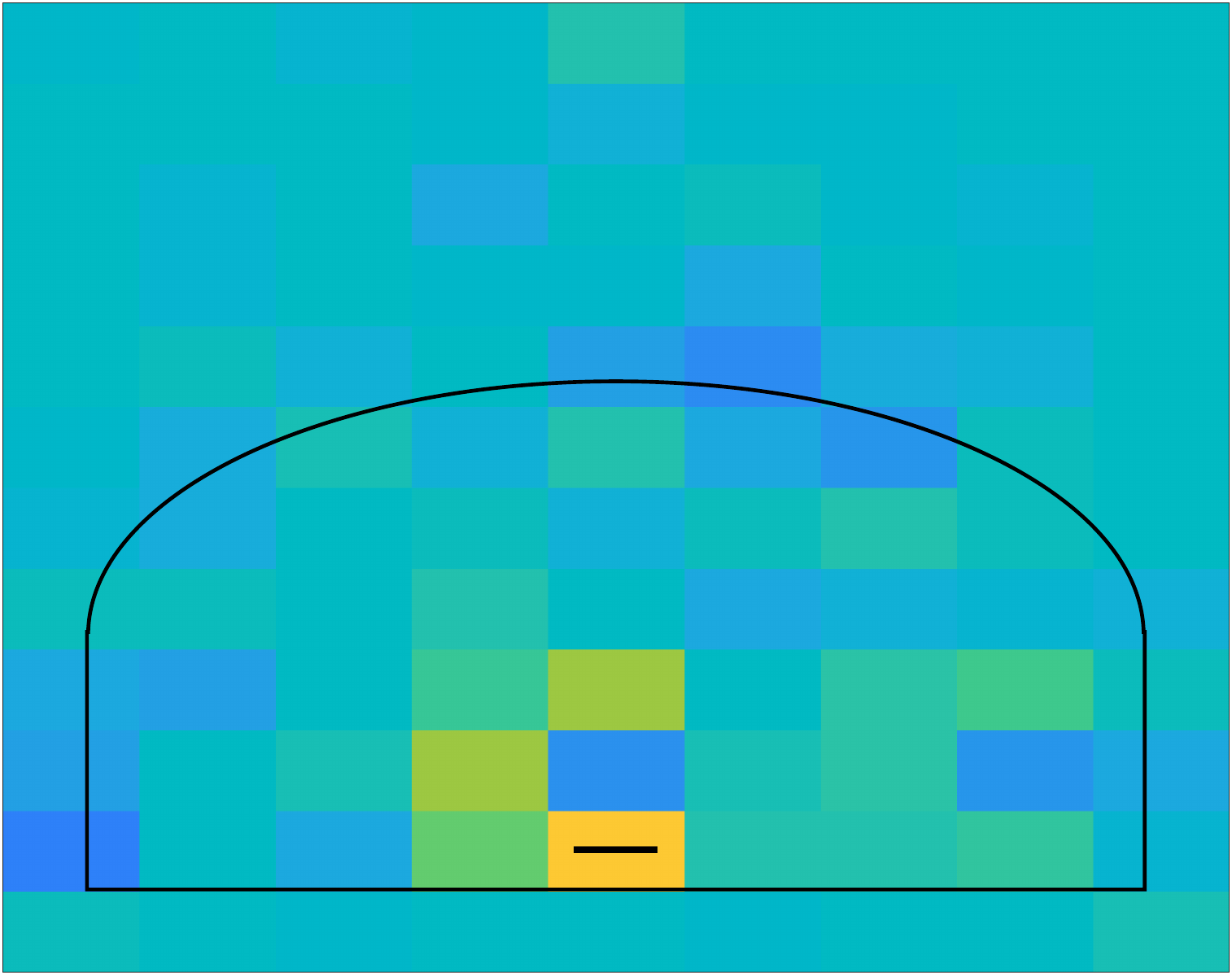}\hspace{2pt}
        \includegraphics[width=0.028\textwidth]{colormap7-crop-crop.pdf}& \includegraphics[width=0.34\textwidth]{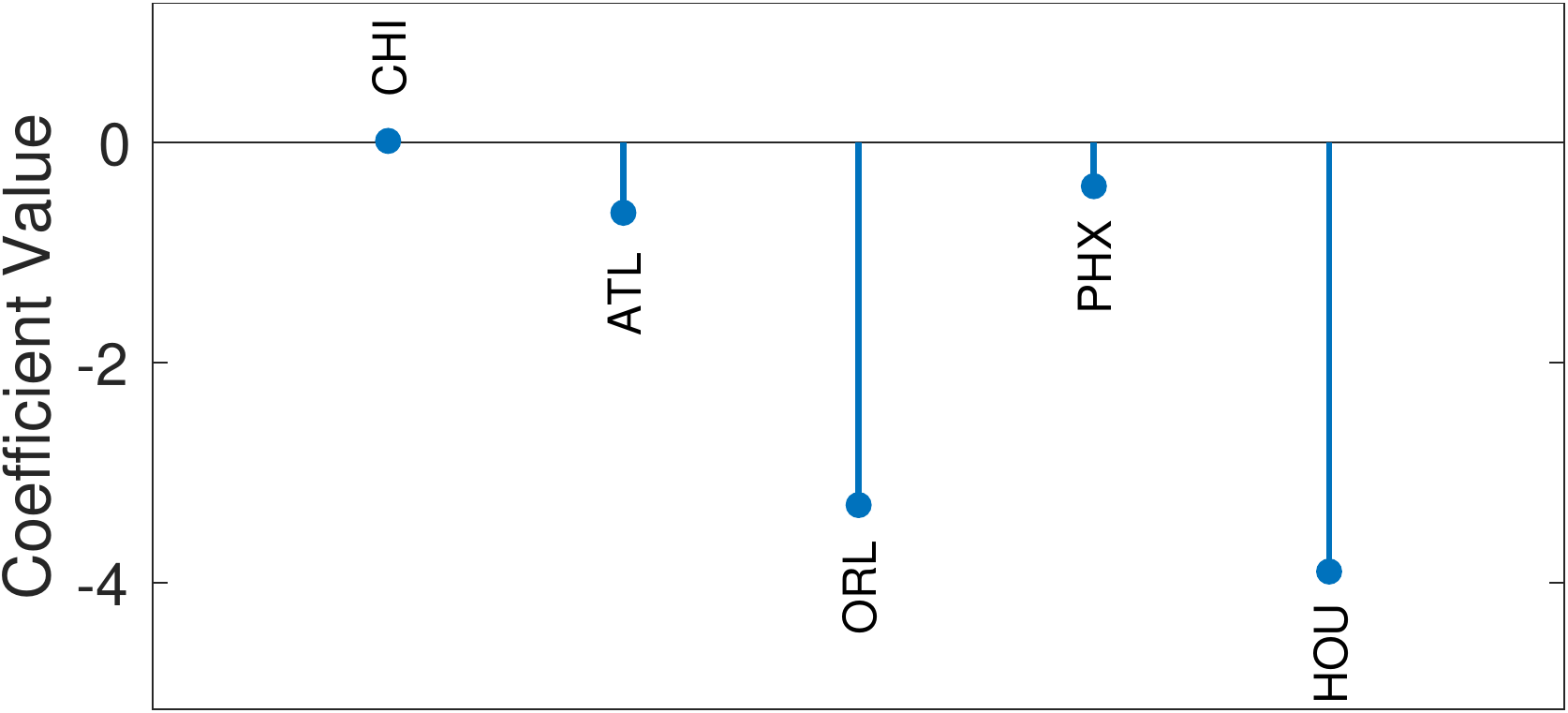} \hspace{5pt}&
    
      \rotatebox{90}{\hspace{10pt}\parbox{2cm}{\centering Element 4}}\hspace{-5pt}& 
     \includegraphics[width=0.2\textwidth]{dict_ele_4-crop.pdf}\hspace{2pt}
        \includegraphics[width=0.028\textwidth]{colormap7-crop-crop.pdf}&
        \includegraphics[width=0.34\textwidth]{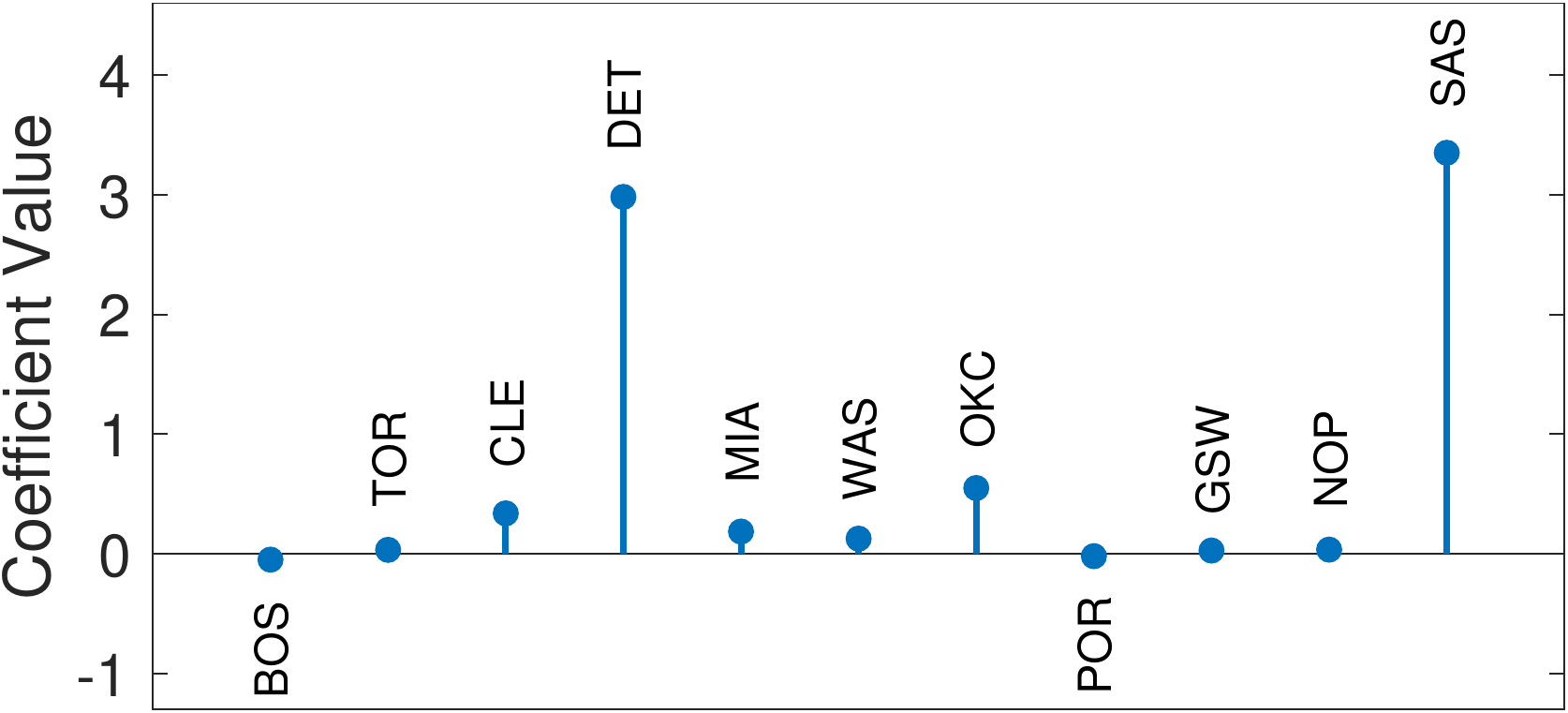}
    \\
      & (c-i) & (c-ii)&& (d-i) & (d-ii)\\
   %     {(a) Dictionary element $1$ }& {(b) Dictionary element $2$ }& 
    %    {(c) Dictionary element $3$ }&  (d) Dictionary element $4$ \\ \\ \vspace{12pt}
  \rotatebox{90}{\hspace{10pt}\parbox{2cm}{\centering Element 5}}\hspace{-5pt}&   
 \includegraphics[width=0.2\textwidth]{dict_ele_5-crop.pdf}\hspace{2pt}
    \includegraphics[width=0.028\textwidth]{colormap7-crop-crop.pdf}&
    \includegraphics[width=0.34\textwidth]{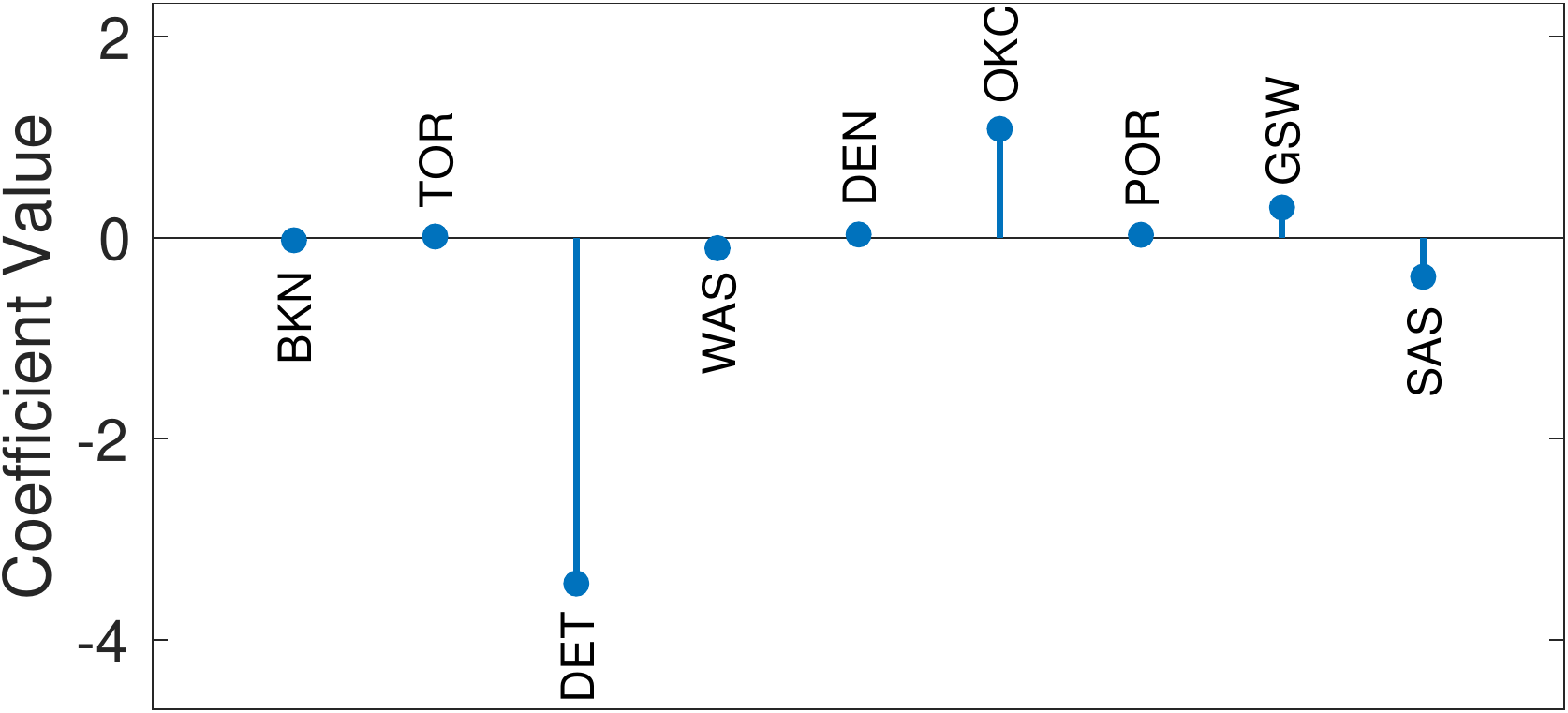} \hspace{5pt}&
     
  \rotatebox{90}{\hspace{10pt}\parbox{2cm}{\centering Element 6}}\hspace{-5pt}&  
  \includegraphics[width=0.2\textwidth]{dict_ele_6-crop.pdf}\hspace{2pt}
     \includegraphics[width=0.028\textwidth]{colormap7-crop-crop.pdf}&
     \includegraphics[width=0.34\textwidth]{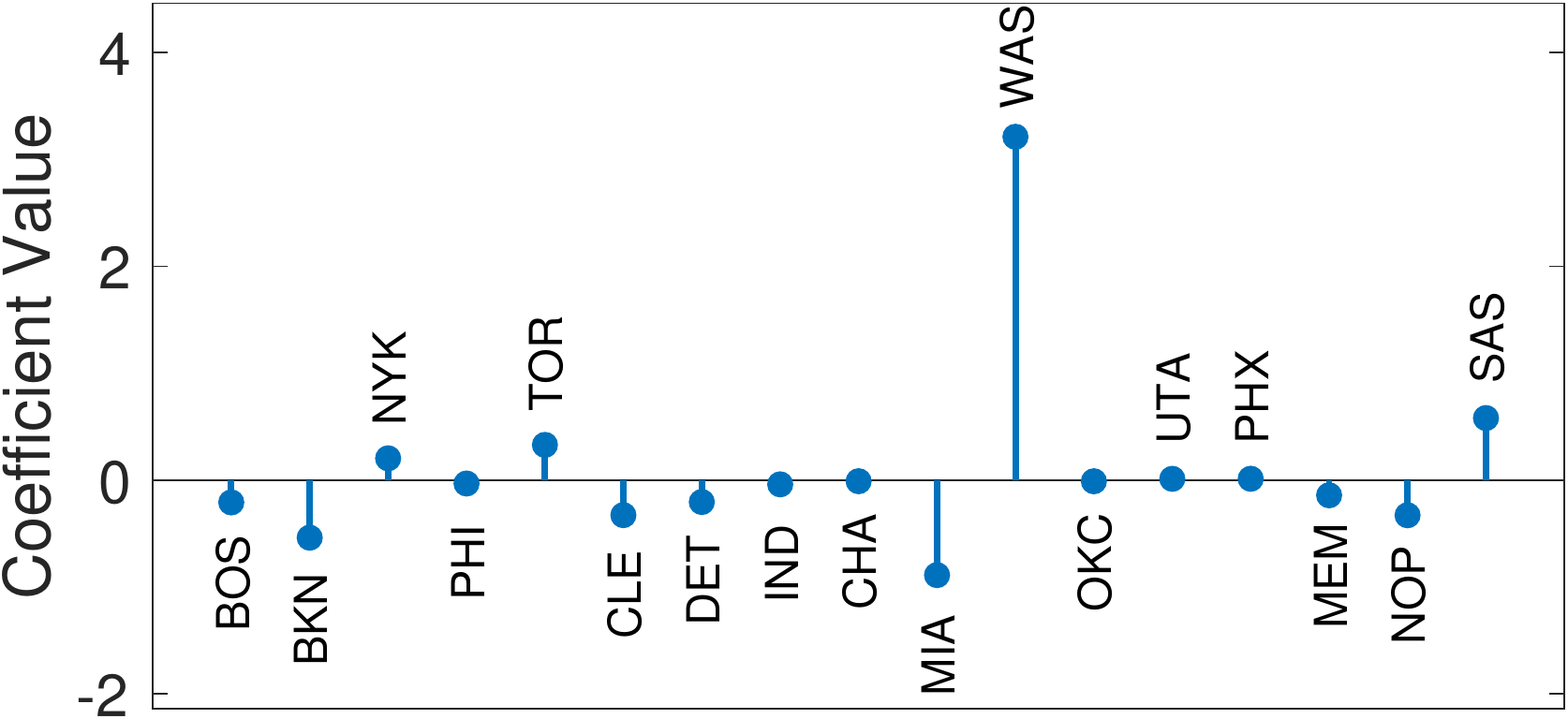}\\
     & (e-i) & (e-ii)&   & (f-i) & (f-ii)\\
      \rotatebox{90}{\hspace{10pt}\parbox{2cm}{\centering Element 7}}\hspace{-5pt}&  
      \includegraphics[width=0.2\textwidth]{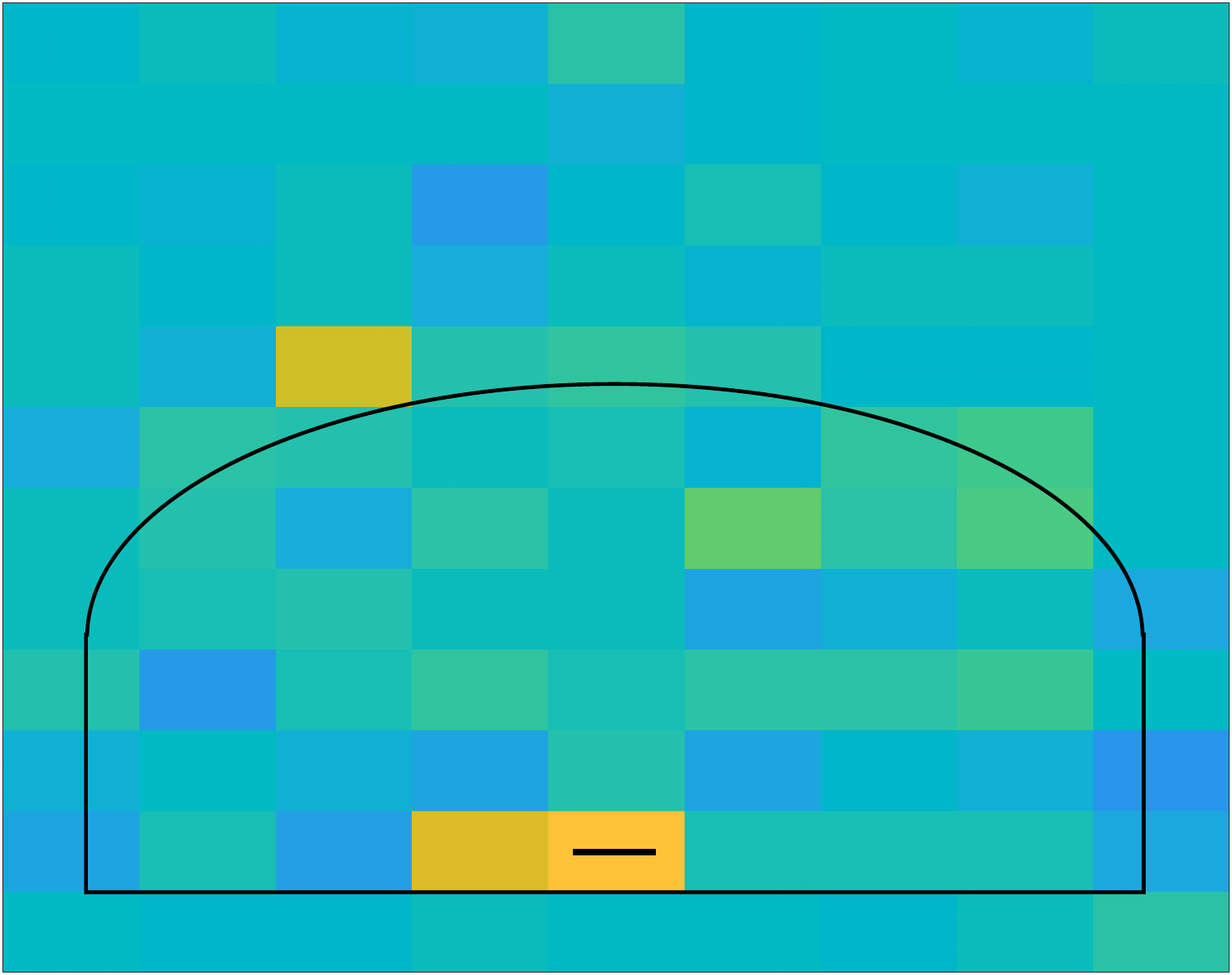}\hspace{2pt}
      \includegraphics[width=0.028\textwidth]{colormap7-crop-crop.pdf}&
      \includegraphics[width=0.34\textwidth]{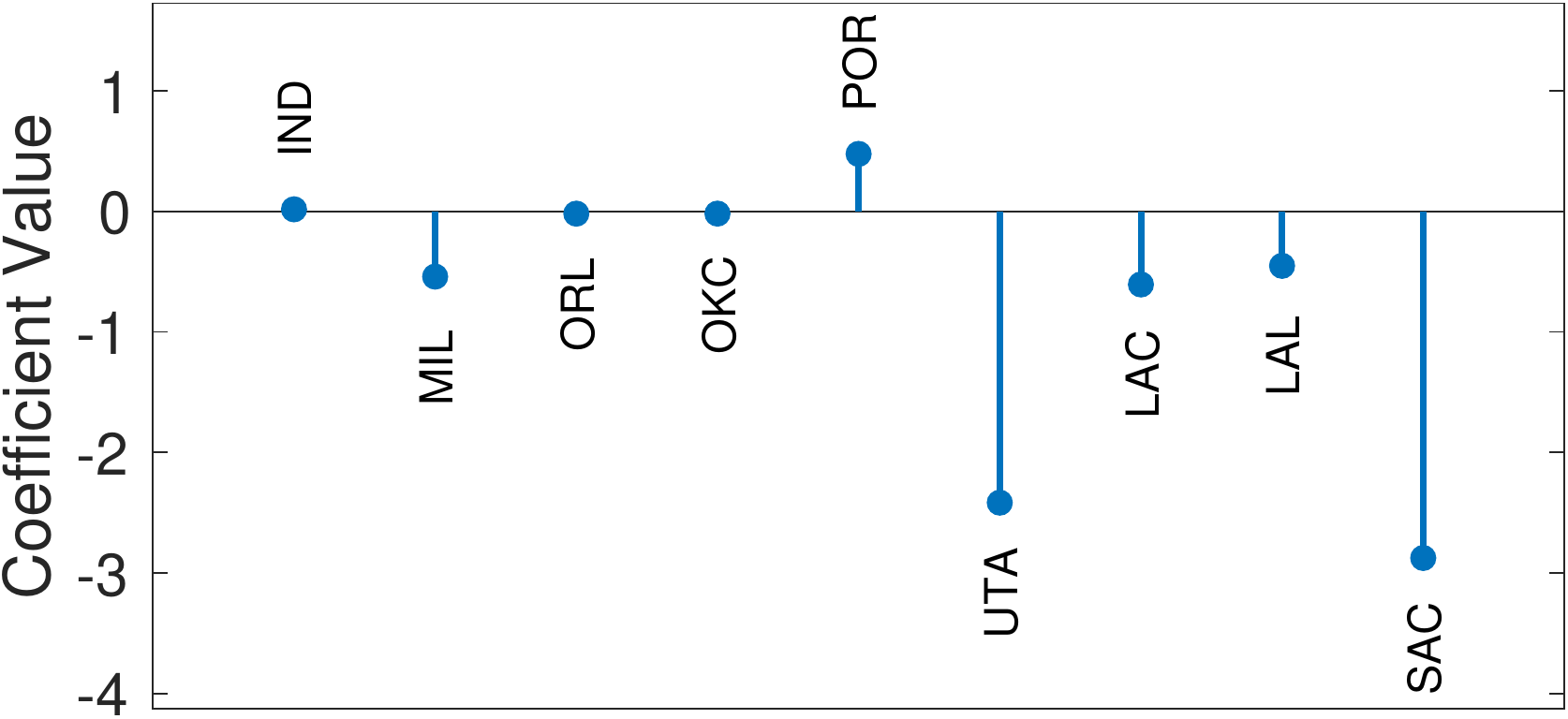} & \\ %\vspace{-8pt}\\ 
        & (g-i) & (g-ii) & & & \\
       % {(e) Dictionary element $5$ }& {(f) Dictionary element $6$}  & {(g) Dictionary element $7$}&\\
   
       \end{tabular}
     }
     \vspace{5pt}
     \caption{ Shot Patterns and Teams in the NBA dataset. Panels (a-g)-i show dictionary factor ( \[{\b{A}}^{(T)}\]) columns (elements) reshaped into a matrix to show different recovered shot patterns. Here, the $3$-point line and the rim is indicated in black. 
     Corresponding sparse factor (\[\hat{\b{B}}\]) representing Teams are shown in panels (a-g)-ii. }
    \label{fig:nba_combined}
     \vspace{-4pt}
    \end{figure}

 \begin{table}[!h]
 \centering
  \caption{Analysis of Sparse factor corresponding to Players (\[\hat{\b{C}}\]) }
  \vspace{5pt}
  \begin{minipage}{0.5\textwidth}
   \centering
   \resizebox{0.78\textwidth}{!}{
 \begin{tabular}{lcc}
 \multicolumn{3}{c}{\textbf{Players corresponding to element 1}}\\
  \textbf{Players} & \textbf{Position} & \textbf{Coefficient Value}\\\hline
  Harrison Barnes & Small forward / Power forward &	-0.2770\\
  Stephen Curry&Point guard&	-0.7620\\
   Kevin Durant & Small forward&-0.0707\\
  Nikola Jokic & Center&	0.5040\\
  CJ McCollum& Shooting guard&	-0.0771\\
   Donovan Mitchell&Shooting guard &	0.0414\\
  Jamal Murray&Point guard / Shooting guard&	-0.1677\\
   Jusuf Nurkic&Center&	0.0352\\
   Ricky Rubio&Point guard&	0.0191\\
   Klay Thompson&Shooting guard&	-0.2128\\
   Russell Westbrook&Point guard&	-0.0208\\
  Lou Williams&Shooting guard / Point guard&	-0.0198\\\vspace{25pt}
     \end{tabular}}
    \end{minipage}%
      \begin{minipage}{0.5\textwidth}
       \centering
       \resizebox{0.78\textwidth}{!}{
     \begin{tabular}{lcc}
     \multicolumn{3}{c}{\textbf{Players corresponding to element 2}}\\
      \textbf{Players} & \textbf{Position} & \textbf{Coefficient Value}\\\hline
     Harrison Barnes &  	 Small forward / Power forward &	-0.0187\\
     Danilo Gallinari & 	Power forward / Small forward &	-0.0515\\
     Tobias Harris &	Small forward / Power forward &	-0.2729\\
     Donovan Mitchell  &	Shooting guard&	0.6536\\
     Karl-Anthony Towns	&Center&	0.5449\\
     Andrew Wiggins	&Shooting guard / Small forward	&0.4454\\
     \multicolumn{3}{c}{\textbf{Players corresponding to element 3}}\\
           \textbf{Players} & \textbf{Position} & \textbf{Coefficient Value}\\\hline
         LaMarcus Aldridge	&Power forward / Center&	-0.2248\\
         Trevor Ariza 	&Small forward / Shooting guard&	0.3195\\
         DeMar DeRozan	&Small forward / Shooting guard	&-0.6716\\
         Bryn Forbes  &	Shooting guard / Point guard & 	0.1241\\
         Justin Holiday   &	Shooting guard / Small forward	&0.1074\\
         Josh Richardson  &	Shooting guard / Small forward&0.6049\\
         Justise Winslow  &	Point guard &  	-0.0580\\
         \end{tabular}}
        \end{minipage}
         \begin{minipage}{0.5\textwidth}
          \centering
           \resizebox{0.78\textwidth}{!}{
            \centering
         \begin{tabular}{lcc}
         \multicolumn{3}{c}{\textbf{Players corresponding to element 4}}\\
          \textbf{Players} & \textbf{Position} & \textbf{Coefficient Value}\\\hline
         Bojan Bogdanovic   &	Small forward &	-0.0275\\
        Devin Booker &  	Shooting guard / Point guard & 	0.0114\\
        Clint Capela &  	Center&	-0.2256\\
        Willie Cauley-Stein	&Center / Power forward	&-0.0150\\
        Evan Fournier  &	Shooting guard / Small forward	&0.2032\\
        James Harden  &	Shooting guard / Point guard  &	0.1992\\
        Buddy Hield	&Shooting guard	&-0.0198\\
        Jeremy Lamb	&Shooting guard / Small forward	&-0.1468\\
        Derrick Rose   &	Point guard  & 	0.4961\\
        Ricky Rubio&	Point guard   &	0.0198\\
        Pascal Siakam  &	Power forward 	&-0.0244\\
        Karl-Anthony Towns &	Center&	0.7711\\
        Kemba Walker & 	Point guard   &	0.0331\\
        Andrew Wiggins &	Shooting guard / Small forward	&-0.0119\\
        Thaddeus Young &	Power forward &	-0.0148\\
        Trae Young &	Point guard  &	0.0415\\
             \end{tabular}}
            \end{minipage}%
              \begin{minipage}{0.5\textwidth}
               \centering
               \resizebox{0.78\textwidth}{!}{
             \begin{tabular}{lcc}
             \multicolumn{3}{c}{\textbf{Players corresponding to element 5}}\\
              \textbf{Players} & \textbf{Position} & \textbf{Coefficient Value}\\\hline
           Devin Booker  &	 Shooting guard / Point guard  &0.0104\\
           Clint Capela  &	Center	&0.0210\\
           Luka Doncic   &	Guard / Small forward &	-0.0162\\
           Eric Gordon   &	Shooting guard / Small forward&	0.0150\\
           James Harden  &	Shooting guard / Point guard  &	0.0678\\
           Tobias Harris &	Small forward / Power forward 	&-0.0247\\
           Joe Ingles&	Small forward &	0.1005\\
           Josh Jackson  &	Small forward / Shooting guard&	-0.0100\\
           Donovan Mitchell  &	Shooting guard&	0.0984\\
           Kelly Oubre Jr.  &	Small forward / Shooting guard&	-0.0143\\
           Derrick Rose & 	Point guard   &	0.6507\\
           Ricky Rubio   &	Point guard  & 	0.0488\\
           Karl-Anthony Towns	&Center&	0.6924\\  
           Kemba Walker  &	Point guard   &	0.1670\\
           Andrew Wiggins&	Shooting guard / Small forward	&0.2000\\
           Lou Williams  &	Shooting guard / Point guard  &	0.0196\\
                 \end{tabular}}
                \end{minipage}
               \begin{minipage}{0.5\textwidth}
                \centering
                      \resizebox{0.78\textwidth}{!}{
                    \begin{tabular}{lcc}
                    \multicolumn{3}{c}{\textbf{Players corresponding to element 6}}\\
                     \textbf{Players} & \textbf{Position} & \textbf{Coefficient Value}\\\hline
                    Deandre Ayton   & 	Center / Power forward& 	0.0640 \\
                   Eric Bledsoe	& Point guard   & 	0.0527\\
                   Bojan Bogdanovic& 	Small forward & 	-0.1353\\
                   Devin Booker	& Shooting guard / Point guard  & 	0.4668\\
                   Jimmy Butler	 & Shooting guard / Small forward & 	-0.0157\\
                   Kentavious Caldwell-Pope	 & Shooting guard & 	0.0507\\
                   Clint Capela	 & Center & 	0.6348\\
                   Willie Cauley-Stein  & 	Center / Power forward & 	-0.0303\\
                   Jordan Clarkson  & 	Point guard / Shooting guard   & 	-0.0141\\
                   John Collins	 & Power forward  & 	0.0948\\
                   DeAaron Fox & 	Point guard   &  	0.0148\\
                   Aaron Gordon	 & Power forward / Small forward  & 	0.0978\\
                   Eric Gordon  & 	Shooting guard / Small forward & 	0.1861\\
                   James Harden	 & Shooting guard / Point guard   & 	0.2834\\
                   Buddy Hield 	 & Shooting guard & 	-0.0135\\
                   Justin Holiday  	 & Shooting guard / Small forward & 	0.0756\\
                    Josh Jackson	 & Small forward / Shooting guard & 	0.0339\\
                                                        LeBron James	 & Small forward / Power forward  & 	-0.1362\\
                                                        Kyle Kuzma   & 	Power forward  & 	-0.0272\\
                        \end{tabular}}
                       \end{minipage}%
                         \begin{minipage}{0.5\textwidth}
                          \centering
                          \resizebox{0.78\textwidth}{!}{
                        \begin{tabular}{lcc}
                        \multicolumn{3}{c}{\textbf{Players corresponding to element 6 continued ... }}\\
                         \textbf{Players} & \textbf{Position} & \textbf{Coefficient Value}\\\hline
                  
                                      Jeremy Lamb 	 & Shooting guard / Small forward & 	-0.0229\\
                                      Kawhi Leonard    & 	Small forward  & 	-0.0384\\
                                      Brook Lopez 	 & Center & 	0.0194\\
                                      Lauri Markkanen  & 	Power forward / Center & 	0.0186\\
                                      CJ McCollum  & 	Shooting guard & 	0.0148\\
                                      Khris Middleton  & Shooting guard / Small forward & 	0.0617\\
                                      Jusuf Nurkic	 & Center & 	0.0121\\
                                      Cedi Osman  	 & Small forward / Shooting guard & 	-0.0260\\
                                      Kelly Oubre Jr. 	 & Small forward / Shooting guard & 	-0.1673\\
                                      JJ Redick  &   	Shooting guard	 & -0.0474\\
                                      Terrence Ross   &  	Small forward / Shooting guard & 	0.0216\\
                                      Pascal Siakam    & 	Power forward 	 & -0.0512\\
                                      Ben Simmons 	 & Point guard / Forward  & 	-0.0166\\
                                      Myles Turner	 & Center	 & -0.3469\\
                                      Nikola Vucevic   & 	Center	 & 0.0827\\
                                      Thaddeus Young  &  	Power forward  & 	-0.0494\\
                                      Trae Young  	 & Point guard    & 	-0.1377\\
                            \end{tabular}}
                           \end{minipage}
             \begin{minipage}{0.5\textwidth}
              \centering
                          \resizebox{0.78\textwidth}{!}{
                        \begin{tabular}{lcc}
                        \multicolumn{3}{c}{\textbf{Players corresponding to element 7}}\\
                         \textbf{Players} & \textbf{Position} & \textbf{Coefficient Value}\\\hline
        Harrison Barnes 	 & 	Small forward / Power forward 	 & 	0.0330\\
        Mike Conley 	 & 	Point guard   	 & 	0.2633\\
        Jae Crowder 	 & 	Small forward 	 & 	0.0454\\
        Stephen Curry   	 & 	Point guard   	 & 	0.0429\\
        Anthony Davis   	 & 	Power forward / Center		 & -0.3173\\
        Luka Doncic 	 & 	Guard / Small forward 	 & 	-0.0239\\
        Kevin Durant	 & 	Small forward 	 & 	-0.5214\\
        Marc Gasol  	 & 	Center	&0.0655\\
        Paul George 	 & 	Small forward 	 & 	-0.6895\\
      
                            \end{tabular}}
                           \end{minipage}%
 \begin{minipage}{0.5\textwidth}
  \centering
                          \resizebox{0.78\textwidth}{!}{
                        \begin{tabular}{lcc}
                        \multicolumn{3}{c}{\textbf{Players corresponding to element 7 continued...}}\\
                         \textbf{Players} & \textbf{Position} & \textbf{Coefficient Value}\\\hline
       Jerami Grant	 & 	Forward   	 & 	-0.0767\\
             Joe Harris  	 & 	Shooting guard / Small forward	 & 	-0.0120\\
             Jrue Holiday	 & 	Point guard / Shooting guard  	 & 	-0.2258\\
             Kyrie Irving	 & 	Point guard   	 & 	-0.0128\\
             Julius Randle   	 & 	Power forward / Center	 & 	-0.0266\\
             DAngelo Russell		 & Point guard   	 & 	-0.0365\\
             Dennis Schroder 	 & 	Point guard / Shooting guard  	 & 	0.1013\\
             Klay Thompson  	 &  	Shooting guard	 & 	0.0322\\
             Dwyane Wade 	 & 	Shooting guard		 & 0.0208\\
             Justise Winslow 	 & 	Point guard	 &   	0.0431\\
                            \end{tabular}}
                           \end{minipage}
                           \vspace{3pt}
 
  \label{fig:nba_players}
   \vspace{5pt}
  \end{table}

%\noindent\begin{minipage}{0.54\textwidth}

\noindent The online nature of \texttt{TensorNOODL} makes it suitable for learning tasks where data arrives in a streaming fashion. In this application, we analyze the National Basketball Association (NBA) weekly shot patterns of high scoring players against different teams. In this online mining application, our aim is to tease apart the relationships between shot selection of different players against different teams. Here, our model enables us to cluster the players and the teams, in addition to recovering the shot patterns shared by them.
%\end{minipage}\hspace{4pt}
%\begin{minipage}{0.45\textwidth}
%\vspace{-20pt}
%\begin{figure}[H]
%	\centering
%\resizebox{0.43\textwidth}{!}{
%              			\begin{tikzpicture} 
%              			\node[anchor=south west,inner sep=0] (image) at (0.2,0) {\includegraphics[width=0.4\textwidth]{Tensor_only.pdf}};
%              			\node[align=center]  at (-0.25,1.4) {\rotatebox[]{90}{\scriptsize$n = 120$}};
%              			\node[align=center]  at (0,1.4) {\rotatebox[]{90}{\scriptsize shot patterns}};
%              			\node[align=center]  at (0.2,0.2) {\scriptsize\rotatebox[]{-45}{$K = 30$}};
%              			\node[align=center]  at (-0.0,-0.0) {\rotatebox[]{-45}{\scriptsize Teams}};
%              			\node[align=center]  at (1.7,-0.2) {\scriptsize$J= 100$};
%              			\node[align=center]  at (1.7,-0.45) {\scriptsize Players};
%              			\node[align=center] at (1.8,1) {\large$\underline{\b{Z}}^{(t)}$};
%              			\end{tikzpicture}}
%              			
%              			\caption{\footnotesize Structured tensor of interest \protect{\[\underline{\b{Z}}^{(t)} \hspace{-1pt}\in \hspace{-1pt}\mathbb{R}^{n\times J \times K}\]} for the shot pattern analysis of NBA data. There are $27$ such tensors arriving every week of the season. }\label{fig:Tensor image nba}
%\end{figure}
%\end{minipage}

We form the NBA shot pattern dataset by collecting weekly shot patterns of players for each week ($27$ weeks) of the $2018-19$ regular season of the NBA league. Each of these tensors consists of the locations of all shots attempted by players (above 80$^{\text{th}}$ percentile of the $497$ active players, which gives us $100$ high-scorers) against ($30$) opponent teams in a week of the $2018-19$ regular season of the NBA league. To form the tensor we divide the half court into $10\times12$ blocks, and sum all the shots from a block to compile the shot pattern. We then vectorize this 2-D shot pattern, which constitutes a fiber of the tensor. Since players don't play every other team in a week, the resulting weekly shot pattern tensor \[\underline{\b{Z}}^{(t)} \in \mathbb{R}^{100\times30\times120}\] has only a few non-zero fibers, and fits the model of interest shown in Fig.~\ref{fig:Tensor image}. In case a player plays against a team more than once a week, we average the shot patterns to form the weekly shot pattern tensor.

\paragraph{Data Preparation and Parameters:}
To prepare the data, we element-wise transform each non-zero element of the weekly shot pattern tensor (\[\underline{\b{Z}}^{(t)}(i,j,k)\]) as \[\underline{\b{Z}}^{(t)}(i,j,k) = \log_2(\underline{\b{Z}}^{(t)}(i,j,k)) + 1\] to reduce its dynamic range. We then substract the mean along the shot pattern axis to reduce the effect of any dominant shot locations.  We form the initial estimate of the incoherent dictionary factor (\[\b{A}^*\]) from the $2017-18$ regular season data of the top $80^{\text{th}}$ percentile players using the initialization algorithm presented in \cite{Arora15}. We use \[\eta_x = 0.1\], \[\tau = 0.2\], \[C=1\] and \[\eta_A = 10\] as the \texttt{TensorNOODL} parameters to analyze the data. 

\vspace{-5pt}
  
\paragraph{Evaluation Specifics:}
We focus on the games in the week $10$ of the $2018-19$ regular season to illustrate the application of \texttt{TensorNOODL} for this sports analytics task. Our analysis yields the shared shot selection structure of different players and teams. 
\vspace{-5pt}

\paragraph{Discussion:}
In the main paper, we analyze the similarity between two players -- James Harden and Devin Booker -- who incidentally at that time were seen as having similar styles \cite{Rafferty18,Uggetti18}. In this case, our results corroborate that the shot selection patterns of these two players is indeed similar. This is indicated by sparse factor corresponding to the players. In  Fig.~\ref{fig:nba_combined}, and Table.~\ref{fig:nba_players} we show the recovered dictionary elements($\b{A}^{(T)}$) or the shot patterns and the corresponding clustering of teams ($\hat{\b{B}}^{(T)}$), and the players ($\hat{\b{C}}^{(T)}$), respectively, for week $10$. For both $\hat{\b{B}}^{(T)}$ and $\hat{\b{C}}^{(T)}$ we show the elements whose corresponding magnitude is greater than $10^{-2}$. These preliminary results motivate further exploration of \texttt{TensorNOODL} for sports analytics applications. The theoretical guarantees coupled with its amenability in highly distributed online processing, makes \texttt{TensorNOODL} especially suitable for such application, where we can learn and make decisions on-the-fly.

\end{document}